\documentclass[11pt,margin=3pt]{article}

\usepackage[hypertexnames=false]{hyperref}
\usepackage{color}
\usepackage{latexsym}
\usepackage{dsfont}
\usepackage{appendix}
\usepackage{amssymb}
\usepackage{subcaption}
\usepackage[margin=.9in]{geometry}
\usepackage{float}
\usepackage{algorithm,algorithmic}
\usepackage[numbers]{natbib}
\usepackage{amsmath, amsfonts,amssymb,euscript,array,enumerate,amsfonts,mathrsfs}
\usepackage[dvipsnames]{xcolor}

\usepackage{amsthm}
\newtheorem{Theorem}{Theorem}[section]
\newtheorem{Definition}[Theorem]{Definition}
\newtheorem{Proposition}[Theorem]{Proposition}
\newtheorem{Assumption}[Theorem]{Assumption}
\newtheorem{Lemma}[Theorem]{Lemma}
\newtheorem{Corollary}[Theorem]{Corollary}
\newtheorem{Remark}[Theorem]{Remark}

\usepackage[T1]{fontenc}
\usepackage[latin1]{inputenc}
\usepackage{hhline}
\usepackage{graphicx}
\usepackage{verbatim}
\usepackage{eurosym}
\usepackage{dsfont}
\DeclareMathOperator*{\dprime}{\prime \prime}
\DeclareMathOperator{\Tr}{Tr}
\DeclareMathOperator*{\argmin}{arg\,min}
\DeclareMathOperator{\sx}{\underline{\sigma}_{\pmb{X}}}
\DeclareMathOperator{\sq}{\underline{\sigma}_{\pmb{Q}}}
\DeclareMathOperator{\sr}{\underline{\sigma}_{\pmb{R}}}
\DeclareMathOperator{\HH}{\mathcal{H}}

\DeclareMathOperator{\HECK}{\mathcal{H}(1/\epsilon,C(\pmb{K}))}
\DeclareMathOperator{\HECKI}{\mathcal{H}(1/\epsilon,C(\pmb{K}^0))}

\DeclareMathOperator{\HCK}{\mathcal{H}(C(\pmb{K}))}
\DeclareMathOperator{\HCKR}{\mathcal{H}(1/\epsilon,C(\pmb{K}))}
\DeclareMathOperator{\HCKIR}{\mathcal{H}(1/\epsilon,C(\pmb{K}^0))}
\makeatletter
\@addtoreset{equation}{section}

\newcommand*{\rom}[1]{\expandafter\@slowromancap\romannumeral #1@}
\makeatother

\newcommand{\vertiii}[1]{{\left\vert\kern-0.25ex\left\vert\kern-0.25ex\left\vert #1 
    \right\vert\kern-0.25ex\right\vert\kern-0.25ex\right\vert}}

\begin{document}
\title{Policy Gradient Methods for the Noisy Linear Quadratic Regulator over a Finite Horizon}

\author{Ben Hambly
\thanks{Mathematical Institute, University of Oxford. \textbf{Email:} \{hambly,  xur,  yang\}@maths.ox.ac.uk }
\and
Renyuan Xu \footnotemark[1]
\and
Huining Yang\thanks{ 
Supported by the EPSRC Centre for Doctoral Training in Industrially Focused Mathematical Modelling (EP/L015803/1) in collaboration with BP plc.}
  \footnotemark[1]
}
\maketitle
\begin{abstract}
     We explore reinforcement learning methods for finding the optimal policy in the linear quadratic regulator (LQR) problem. In particular we consider the convergence of policy gradient methods in the setting of known and unknown parameters. We are able to produce {a global linear} convergence guarantee for this approach in the setting of finite time horizon and stochastic state dynamics under weak assumptions. The convergence of a projected policy gradient method is also established in order to handle problems with constraints. We illustrate the performance of the algorithm with two examples. The first example is the optimal liquidation of a holding in an asset. We show results for the case where we assume a model for the underlying dynamics and where we apply the method to the data directly. The empirical evidence suggests that the policy gradient method can learn the global optimal solution for a larger class of stochastic systems containing the LQR framework and that it is more robust with respect to model mis-specification when compared to a model-based approach. The second example is an LQR system in a higher dimensional setting with synthetic data.
\end{abstract}

\section{Introduction}

The Linear Quadratic Regulator (LQR) problem is one of the most fundamental in optimal control theory. Its aim is to find a control for a linear dynamical system, that is the dynamics of the state of the system is described by a linear function of the current state and input, subject to a quadratic cost. It is an important problem for a number of reasons: (1) the LQR problem is one of the few optimal control problems for which there exists a closed-form analytical representation of the optimal feedback control; (2) when the dynamics are nonlinear and hard to analyze, a LQR approximation may be obtained as a local expansion and provide an approximation that is provably close to the original problem; (3) the LQR has been used in a wide variety of applications. In particular, in the set-up of fixed time horizon and stochastic dynamics, applications include portfolio optimization \cite{abeille2016} and optimal liquidation \cite{AC2001} in finance, resource allocation in energy markets \cite{patrinos2011,wasa2017}, and biological movement systems \cite{li2004iterative}.


Until recently much of the work on the LQR problem has focused on solving for the optimal controls under the assumption that the model parameters are {\it fully known}. See the book of Anderson and Moore \cite{anderson2007optimal} for an introduction to the LQR problem with known parameters. However, assuming that the controller has access to all the model parameters is not realistic for many applications, and this has lead to the exploration of learning approaches to the problem. We consider reinforcement learning (RL), one of the three basic machine learning paradigms (alongside supervised learning and unsupervised learning). Unlike the situation with full information on the model parameters, RL is learning to make decisions via trial and error, through interactions with the (partially) unknown environment. In RL,  an agent takes an action and receives a reinforcement signal in terms of a numerical reward, which encodes the outcome of her action.  In order to maximize the accumulated reward over time, the agent  learns to select her actions  based on her past experiences (exploitation) and/or  by making new choices (exploration).
There are two popular approaches {in RL} to handle the LQR with unknown parameters: the model-based approach and the model-free approach. 

In the paradigm of the model-based approach, the controller estimates the unknown model parameters and then constructs a control policy based on the estimated parameters. The classical approach is the {\it certainty equivalence principle} \cite{aastrom2013adaptive}:  the unknown parameters are estimated using observations (or samples), and a control policy is then designed by treating the estimated parameters as the truth. In the first step, the unknown model parameters can be estimated by standard statistical methods such as least-square minimization \cite{dean2019}. The second step is to show that when the estimated parameters are accurate enough, the policy using the ``plug-in'' estimates enjoys good theoretical guarantees of being close to optimal. See \cite{dean2019} and \cite{fiechter1997pac} for the optimal gap and sample complexities along this line and see \cite{fattahi2020efficient} for the sample complexity with distributed robust learning.  Another line of work in the model-based regime focuses on {\it uncertainty} quantification. The controller updates their posterior belief or the confidence bounds on the unknown model parameters and then makes decisions in an online manner, see \cite{abbasi2011regret,abeille2017thompson,faradonbeh2020optimism,ibrahimi2012efficient,ouyang2017control}.

Another recently developed approach is the {\it model-free approach}, where the controller  learns the optimal policy {\it directly} via interacting with the system, without inferring the model parameters. As the optimal policy in the LQR problem is a linear function of the state, the aim is to determine this linear function. This is  equivalent to learning a set of parameters in matrix form, called the policy matrix. One natural way to achieve this goal is to apply the gradient descent method in the parameter space of the policy matrix, also referred to as the {\it policy gradient method}. In particular, the policy gradient method computes the gradient of the cost function with respect to the policy matrix and then updates the policy in the steepest decent direction to find the optimal policy. {The paper \cite{FGKM2018} was the first to show that policy gradients converge to the global optimal solution with polynomial (in the relevant quantities) sample complexity. However, \cite{FGKM2018} focuses on the case where the only noise in the system is in the initial state, and the rest of the state transitions are deterministic.} There are other methods that fall into the category of the model-free approach, including the Actor-Critic method \cite{yang2019} and least-squares temporal difference learning \cite{tu2018least}.

If the true system is indeed linear-quadratic, the model-based approaches (may) outperform the model-free approaches by fully utilizing the linear-quadratic structure. For example in the setting that the system transition matrices are unknown and the parameters in the cost function are known, 
\cite{ReBen2019} and \cite{tu2019} showed that model-based methods are (asymptotically) more sample-efficient than some popular model-free methods. However, we are often uncertain about whether the actual system is linear-quadratic in the learning setting; for instance there might be some small nonlinear terms in the system dynamics. Therefore, compared to the model-based approach, which strongly relies on the assumption that the stochastic system lies within the LQR framework and may, in practice, suffer from model mis-specification, the execution of the model-free algorithm does not rely on the assumptions of the model. It has been shown that the policy gradient method can learn the global optimal solution, not only for the LQR framework, but also for a more general class of deterministic systems in the setting of an infinite time horizon \cite{bhandari2019}. Thus the advantage of the model-free approach is that it is more robust against model mis-specification compared to the model-based approach. 

\paragraph{Our Contributions.} 
We now summarize our contributions. Motivated by many real-word decision-making problems with a fixed deadline and uncertainty in the underlying dynamics, such as the optimal liquidation problem that we discuss in Section~2, we extend the framework of \cite{FGKM2018} by incorporating a
finite time horizon and sub-Gaussian noise (which includes Gaussian noise as a special case). In particular, we provide a global linear convergence guarantee and a polynomial sample complexity guarantee for the policy gradient method in this setting with both known parameters (Theorem~\ref{thm:convergence_egd}) and unknown parameters (Theorem~\ref{thm:model_free}).  The analysis with known parameters paves the way for learning LQR with unknown parameters. In addition, numerically solving the Riccati equation with known parameters in high dimensions may suffer from computational inaccuracy. The policy gradient method provides a direct way of searching for the optimal solution with known parameters in this case, which may be of separate interest.
Note that the optimal policy is time-invariant for the LQR with infinite time horizon, whereas the optimal policy is time-dependent with finite time horizon and hence harder to learn in general. 
With noise in the dynamics, we need more careful choices of the hyper-parameters to retrieve compatible sample complexities with noisy observations.
In addition, when optimal polices need to satisfy certain constraints, we provide a global convergence result for the projected policy gradient method  in Theorem~\ref{thm:projected_GD}. This is required in the context of our application to the optimal liquidation problem.

We will formulate the optimal liquidation problem over a fixed horizon as a noisy LQR problem which is essentially the classical Almgren-Chriss formulation \cite{AC2001}. The performance of the algorithm on NASDAQ ITCH data is assessed. As well as using the method within this modelling approach, we also consider the performance of the policy gradient method when applied directly to the data with an appropriate cost function. This improves the performance of the LQR/Almgren-Chriss solution and shows promising results for the use of the policy gradient method for problems that are `close' to the LQR framework.


\subsection{Related Work}

\paragraph{Policy Gradient Methods for LQR Problems.}
Since the policy gradient method is the main focus of our paper, here we provide a review of the previous theoretical work on this method in various LQR settings and extensions. {The first global convergence result for the policy gradient method to learn the optimal policy for LQR problems was developed in \cite{FGKM2018} in the setting of infinite horizon and deterministic dynamics.} The work of \cite{FGKM2018} was extended in \cite{bhandari2019} to give global optimality guarantees of policy gradient methods for a larger class of control problem that includes the linear-quadratic case.  In particular, this class of control problem satisfies a closure condition under policy improvement and convexity of policy improvement steps.
The paper \cite{bu2019} considers policy gradient methods for LQR problems in terms of optimizing a real valued matrix function over the set of feedback gains.
The extension of the policy gradient method to continuous-time can be found in \cite{bu2020}. All of these methods are in the infinite horizon setting and without the addition of noise in the dynamics.

There has been some work on the case of noisy dynamics, but all in the setting of infinite horizon. In \cite{gravell2019} the problem with a multiplicative noise was discussed, using a relatively straightforward extension of the deterministic dynamics considered in the original framework.
In the case of additive noise \cite{jin2020analysis} studies the global convergence of policy gradient and other learning algorithms for the LQR over an infinite time-horizon and with Gaussian noise. In particular, the policy considered in \cite{jin2020analysis} is a randomized policy with Gaussian distribution. There is also \cite{malik2019derivative} which studies derivative-free (zeroth-order) policy optimization methods for the LQR with bounded additive noise.
Finally some other contributions can be found in \cite{bu2019global,zhang2019policy} for zero-sum LQR games and \cite{carmona2019linear,guo2020entropy} for  mean-field LQR games.

Compared to \cite{FGKM2018}, our technical difficulties are three-fold. First due to the time-dependent nature of the admissible policies over a finite-horizon  and randomness from the system noise, we need additional conditions and analysis to guarantee the well-definedness of the state process, i.e., the non-degeneracy of the controlled state-covariance matrices. This holds almost for free in the infinite horizon case with deterministic dynamics. Second, we need to take care of the additional randomness from the sub-Gaussian noise when developing the perturbation analysis and the gradient dominant condition. Third, we need more advanced concentration inequalities and tighter upper bounds to provide compatible sample complexity analysis in the unknown parameter case. See the more detailed discussion in Remark~\ref{rmk:comparison}.


\paragraph{Optimal Liquidation.}
An early mathematical framework 
for the optimal liquidation problem 
is due to Almgren and Chriss \cite{AC2001}.
In this problem a trader is required to liquidate a portfolio of shares over a fixed horizon.
The selling of a large number of shares at once has both temporary and permanent impacts on the share price causing it to decrease. The trader therefore wishes to find a trading strategy which maximizes their return from, or alternatively, minimizes the cost of, the liquidation of the portfolio subject to a given level of risk.

This problem has been considered in many papers and extended in many directions. See for instance \cite{alfonsi2011},  \cite{Almgren2003} and \cite{gatheral2011}. We will cast this as an LQR problem and show how the policy gradient method is a powerful tool for solving this problem even without assumptions on the model.



More recently techniques from reinforcement learning have been applied to the optimal liquidation problem. The first paper to do this was \cite{NFK2006} where the authors showed promising results for this approach by designing a Q-learning based algorithm to optimally select price levels and passively place limit orders. This was further developed in \cite{HW2014} which designed a Q-learning  based algorithm for liquidation within the standard Almgren-Chriss framework. For recent work incorporating deep learning see for example  \cite{BL2019}, \cite{leal2020learning}, \cite{NLJ2018}, and \cite{ZZR2020}.
See \cite{charpentier2020} for a detailed review on reinforcement learning with applications in finance and economics, and the references therein. However, all these works focus on the model-free setting without taking advantage of even weak modelling assumptions on the market dynamics. In addition, the performances of these proposed algorithms are validated only through empirical studies and no theoretical guarantee of convergence is provided.

\paragraph{Organization and Notation.} For any matrix $Z =(Z_1,\cdots,Z_d) \in \mathbb{R}^{m\times d}$ with $Z_j \in \mathbb{R}^m$ ($j=1,2,\cdots, d$), $Z^{\top} \in \mathbb{R}^{m\times d}$ denote the transpose of $Z$, $\|Z\|$ denotes the spectral norm of a matrix $Z$; $\Tr(Z)$ denotes the trace of a square matrix $Z$; {$\|Z\|_F$ denotes the  Frobenius norm of a matrix $Z$}; $\sigma_{\min}(Z)$ denotes the minimal singular value of a square matrix $Z$; and $\text{vec}({Z}) = (Z_1^{\top},\cdots,Z_d^{\top})^{\top}$ denote the vectorized version of a matrix $Z$. For a sequence of matrices $\pmb{D}=(D_0,\cdots,D_T)$, we define a new norm $\vertiii{\pmb{D}}$ as $\vertiii{\pmb{D}} = \sum_{t=0}^T \|D_t\|$, where $D_t \in \mathbb{R}^{m\times d}$. Further denote $\mathcal{N}(\mu,\Sigma)$ as the Gaussian distribution with mean $\mu\in \mathbb{R}^d$ and covariance matrix $\Sigma\in\mathbb{R}^{d\times d}$.

The rest of the paper is organized as follows. We introduce the mathematical framework and problem set-up in Section \ref{sc:single_agent_setup}.  The first step in our convergence analysis of the policy gradient method is to consider the case of known model parameters in Section \ref{sc:single_agent_model_based}. When parameters are unknown, the convergence results for the sample-based policy gradient method and projected policy gradient method are obtained in Section \ref{sc:single_agent_model_free}.  Finally, the algorithm is applied to liquidation problem. See Sections \ref{sec:optimal_liquidation_formulation} and \ref{sc:experiment} for the corresponding set-up and algorithm performance, respectively.

\section{Problem Set-up}\label{sc:single_agent_setup}
We consider the following LQR problem over a finite time horizon $T$,
\begin{equation}\label{LQR_problem}
    \min_{\{u_t\}_{t=0}^{T-1}} \mathbb{E}\left[\sum_{t=0}^{T-1}\left(x_t^{\top}Q_tx_t+u_t^{\top}R_tu_t\right)+x_T^{\top}Q_Tx_T\right],
\end{equation}
such that for $t=0,1,\cdots,T-1$, 
\begin{equation}\label{dynamics}
    x_{t+1} = Ax_t+Bu_t+w_t,\ x_0\sim\mathcal{D}.
\end{equation}
Here $x_t\in\mathbb{R}^d$ is the state of the system with the initial state $x_0$ drawn from a distribution $\mathcal{D}$, $u_t\in\mathbb{R}^k$ is the control at time $t$ and $\{w_t\}_{t=0}^{T-1}$ are zero-mean IID noises which are independent from $x_0$. {At this moment, we only assume $x_0$ and $\{w_t\}_{t=0}^{T-1}$ have finite second moments.} That is, $\mathbb{E}[x_0 x_0^\top]$  and $W:=\mathbb{E}[w_t w_t^{\top}]$ $(\forall\, t=0,1,\cdots,T-1)$ exist. The system parameters $A\in\mathbb{R}^{d\times d}$ and $B\in\mathbb{R}^{d\times k}$ are referred to as system (transition) matrices; $Q_t\in\mathbb{R}^{d\times d}\ (\forall\, t=0,1,\cdots,T$) and $R_t\in\mathbb{R}^{k\times k}\ (\forall t=0,1,\cdots,T-1$) are  matrices that parameterize the quadratic costs. Note that the expectation in \eqref{LQR_problem} is taken with respect to both $x_0\sim\mathcal{D}$ and $w_t$ ($t=0,1,\cdots,T-1$).
We further denote by $\pmb{u} := (u_0,\cdots,u_{T-1})$, $\pmb{x} := (x_0,\cdots,x_{T})$, $\pmb{w} := (w_0,\cdots,w_{T-1})$, $\pmb{Q} := (Q_0,\cdots,Q_{T})$, and $\pmb{R} := (R_0,\cdots,R_{T-1})$, the profile over the decision period $T$.

To solve the LQR problem \eqref{LQR_problem}-\eqref{dynamics}, let us start with some  conditions on the model parameters to assure the well-definedness of the problem.
\begin{Assumption}[Cost Parameter]\label{ass:parameters} Assume $Q_t\in\mathbb{R}^{d\times d}$, for $t=0,1,\cdots,T$, and $R_t\in\mathbb{R}^{k\times k}$, for $t=0,1,\cdots,T-1$, are positive definite matrices.
\end{Assumption}

Under Assumption \ref{ass:parameters}, we can properly define a sequence of matrices $\{P_t^*\}_{t=0}^T$ as the solution to the following dynamic Riccati equation  \cite{Bertsekas2015}:
\begin{equation}\label{eqn:P_t}
    P_{t}^* = Q_t + A^{\top}P_{t+1}^*A - A^{\top}P_{t+1}^*B\left(B^{\top}P_{t+1}^*B+R_t\right)^{-1}B^{\top}P_{t+1}^*A,
\end{equation}
with terminal condition $P_T^* = Q_T.$ The matrices $\{P_t^*\}_{t=0}^T$ can be found by solving the Riccati equations iteratively backwards in time. In particular with a slight modification of the initial state distribution in \cite[Chapter 4.1]{Bertsekas2015}, we have the following result.

\begin{Lemma}[Well-definedness and the Optimal Solution \cite{Bertsekas2015}] Under Assumption \ref{ass:parameters},
\begin{enumerate}
    \item The solution $P_t^*$ to the Riccati equation \eqref{eqn:P_t} is positive definite, $\forall\, t=0,1,\cdots,T$;
    \item Then the optimal control sequence $\{u_t\}_{t=0}^{T-1}$ is given by
\begin{eqnarray}
    u_t &=& -K_t^*x_t, \label{opt_form} \,\,\qquad\qquad {\rm where}\\
    K_t^* &=& \left(B^{\top}P_{t+1}^*B+R_t\right)^{-1}B^{\top}P_{t+1}^*A.\label{optimal_k}
\end{eqnarray}
\end{enumerate}
\end{Lemma}

To find the optimal solution in the linear feedback form \eqref{opt_form}, we only need to focus on the following class of linear {\it admissible policies} in feedback form
\begin{equation}\label{admissible_form}
    {u}_t=-{K}_t{x}_t, \qquad t=0,1,\cdots,T-1,
\end{equation}
which can be fully characterized by $\pmb{K}:= (K_0,\cdots,K_{T-1})$.

\subsection{Application: The Optimal Liquidation Problem}\label{sec:optimal_liquidation_formulation}

One application of the LQR framework \eqref{LQR_problem}-\eqref{dynamics} is the optimal liquidation problem. We give a slight variant of the setup of Almgren-Chriss \cite{AC2001}. Our aim is to liquidate an amount $q_0$ of an asset, with price $S_0$ at time 0, over the time period $[0,T]$ with trading decisions made at discrete time points $t=0,1,\dots,T-1$. At each time $t$ our decision is to liquidate an amount $u_t$ of the asset. Any residual holding is then liquidated at time $T$. This will have two types of price impact. There will be a temporary price impact, caused when the order `walks the book' and a permanent price impact as traders rearrange their positions in the light of the sell order.
We will assume the impacts are linear in the number of traded shares.


We write $S_t$ for the asset price at time $t$. This evolves according to a Bachelier model with a linear permanent price impact in that
\[ S_{t+1} = S_t + \sigma Z_{t+1} - \gamma u_t, \]
where, for each $t=1,\dots,T$, $Z_t$ is an independent standard normal random variable, $\sigma$ is the volatility and $\gamma$ is the permanent price impact parameter.  The inventory process $q_t $ records the current holding in the asset at time $t$. Thus we have 
\[ q_{t+1}=q_t - u_t. \]
Therefore, the two-dimensional state process is
\begin{eqnarray}\label{eq:liquidaton_states}
    \begin{pmatrix}
  S_{t+1} \\   q_{t+1} 
\end{pmatrix} 
=\begin{pmatrix}
 1 &0\\  0&1
\end{pmatrix} 
 \begin{pmatrix}
  S_{t} \\   q_{t} 
\end{pmatrix}
+
 \begin{pmatrix}
  -\gamma\\   -1
\end{pmatrix} u_t
+
 \begin{pmatrix}
 \sigma Z_{t+1} \\  0
\end{pmatrix}.
\end{eqnarray}

When selling shares we incur a temporary price impact, parameter $\beta$, in that if, at time $t$, we trade $u_t$ of our asset then we obtain $\Tilde{S}_t = S_t-\beta u_t$ per share. Therefore the total revenue is
$\sum_{t=0}^{T-1}u_t\Tilde{S}_t + q_T\Tilde{S}_T$, and $C_T$,  the total cost of execution over $[0,T]$, is the book value at time 0 minus the revenue: 
\[ C_T=q_0 S_0-\sum_{t=0}^{T-1}u_t\Tilde{S}_t-q_T\Tilde{S}_T. \]
In a similar way to \cite{AC2001}, after summation by parts, we have
\[ C_T = -\sigma \sum_{t=1}^{T} q_t Z_t - \frac{\gamma}{2} \sum_{t=0}^{T-1} u_t^2 +\frac{\gamma}{2}\left(q_0^2-q_T^2\right) +\beta \sum_{t=0}^{T-1} u_t^2 + \beta q_T^2.
\]
The mean and variance of the total cost of execution are given by
\begin{equation*}
    \mathbb{E}(C) =  \sum_{t=0}^{T-1}  \delta u_t^2 + \delta q^2_T + \frac{\gamma}{2} q_0^2,  \quad \mbox{var}(C) = \sum_{t=1}^{T}\sigma^2 q_t^2,
\end{equation*}
where $\delta=\beta-\gamma/2$ summarizes the impact and is assumed positive.

Following Almgren-Chriss \cite{AC2001}, we minimize the following cost function
\begin{equation}\label{eq:mincost_AC}
    C_{\rm AC} = \min\,\left( \mathbb{E}(C) + \phi\,\mbox{var}(C)\right),
\end{equation}
where $\phi$ is a parameter balancing risk versus return. For our  LQR framework we take the cost function to be
\begin{equation}\label{eq:mincost_LQR}
C_{\rm LQR}(\epsilon) = \min \, \left(\mathbb{E}(C) + \phi\,\mbox{var}(C) + \epsilon \sum_{t=0}^T S_t^2\right) = \min \, \left(\sum_{t=0}^{T-1}  \delta u_t^2 + \delta q^2_T + \frac{\gamma}{2} q_0^2 + \phi\sum_{t=1}^{T}\sigma^2 q_t^2 + \epsilon \sum_{t=0}^T S_t^2\right). 
\end{equation}
Note that the term ${\epsilon \sum_{t=0}^T S_t^2}$, with some small $\epsilon>0$, serves as a regularization term to guarantee Assumption \ref{ass:parameters} holds. In practice, we can show that the optimal solution with  $\epsilon$ small  is close to the Almgren-Chriss solution (when $\epsilon=0$). In addition, the algorithm will still converge with $\epsilon=0$. See more discussion in Section \ref{sc:experiment}.
Thus, in the LQR formulation we have $A = \begin{pmatrix} 1 &0\\  0&1 \end{pmatrix} $, $B = (-\gamma,-1)^{\top}$, and $w_t = (\sigma Z_{t+1}, 0)^{\top}$
and the objective function has
$Q_T = \begin{pmatrix}
  \epsilon &0\\ 
  0&\delta+\phi\sigma^2
 \end{pmatrix} $, $Q_t = \begin{pmatrix}
  \epsilon &0\\ 
  0&\phi\sigma^2
 \end{pmatrix} $ and $R_t=\delta$.
It is easy to see that $Q_t$, for $t=0,1,\cdots,T$ and $R_t$ for $t=0,1,\cdots,T-1$ are positive definite, hence Assumption \ref{ass:parameters} is satisfied. 

We will show that the problem is well-defined and can be solved using the methods of this paper with rigorous convergence guarantees.

\section{Exact Gradient Methods with Known Parameters}\label{sc:single_agent_model_based}

In this section we assume all the parameters in the model, $\{Q_t\}_{t=0}^T$, $\{R_t\}_{t=0}^{T-1}$, $A$, $B$, are known.
The analysis of exact gradient methods with known parameters paves the way for learning LQR with unknown parameters in Section \ref{sc:single_agent_model_free}. 
In addition, the policy gradient method provides an alternative way to solve the LQR problem when the parameters are fully known. In this setting the Riccati equation \eqref{eqn:P_t} is just solved backward in time. However this operation involves inverting large matrices when the problem is in high dimensions, which may lead to high computational cost and accumulation of computational errors.

Since an admissible policy can be fully characterized by $\pmb{K}$, the cost of a policy $\pmb{K}$ can be correspondingly defined as 
\begin{equation}\label{defn of CK}
    C(\pmb{K})=\mathbb{E}\left[\sum_{t=0}^{T-1}\left(x_t^\top Q_tx_t+u_t^{\top}R_tu_t\right)+x_T^{\top}Q_Tx_T\right],
\end{equation}
where $\{x_t\}_{t=1}^T$ and $
\{u_t\}_{t=0}^{T-1}$ are the dynamics and controls induced by following $\pmb{K}$, starting with $x_0\sim\mathcal{D}$. Recall that $\pmb{K}^*$ is the optimal policy for the problem, in that
\begin{equation}\label{optimal_solution}
   \pmb{K}^*= \argmin_{\pmb{K}} C(\pmb{K}),
\end{equation}
subject to the dynamics \eqref{dynamics}.

\paragraph{Well-definedness of the State Process.}  

To prove the global convergence of policy gradient methods, the essential idea is to show the {\it gradient dominance condition}, which states that $C(\pmb{K}) - C(\pmb{K}^*)$ can be bounded by $\|\nabla C(\pmb{K})\|_F$ for any admissible policy $\pmb{K}$. One of the key steps to guarantee this gradient dominance condition is the well-definedness of the state covariance matrix. That is, $\mathbb{E}[x_t x_t^{\top}]$ is positive definite for $t=0,1,\cdots,T$. This condition holds almost for free for LQR problems with infinite time horizon and deterministic dynamics. The only condition needed there is the positive definiteness of  $\mathbb{E}[x_0 x_0^{\top}]$ (See \cite{FGKM2018}). However, some effort needs to be made to ensure that the state covariance matrix is well-defined for LQR problems with finite horizon and stochastic dynamics. We show that this condition holds under moderate conditions.

\begin{Assumption}[Initial State and Noise Process]\label{ass:State} We assume that
\begin{enumerate}
    \item Initial state: $x_0\sim\mathcal{D}$ such that $\mathbb{E}[x_0x_0^\top]$ is positive definite;
    \item Noise: $\{w_t\}_{t=0}^{T-1}$ are IID  and independent from $x_0$ such that 
    $\mathbb{E}[w_t] =  0$, and $W=\mathbb{E}[w_tw_t^{\top}]$ is positive definite, $\forall t=0,1,\cdots,T-1$.
\end{enumerate}
\end{Assumption}
Define $\sx$ as the lower bound over all the minimum singular values of $\mathbb{E}[x_tx_t^\top]$:
\begin{equation}\label{Defn of barmu}
   \sx  = \min_t \sigma_{\min}(\mathbb{E}[x_tx_t^{\top}]),
\end{equation}
then we have the following result and the proof can be found in  Appendix \ref{appendix:missing_proofs_sec3_1}.
\begin{Lemma}[Well-definedness of the State Covariance Matrix]\label{lemma:covariance}
Under Assumption \ref{ass:State}, we have $\mathbb{E}[x_t x^{\top}_t]$  is positive definite  for $t=0,1,\cdots,T$ under any control policy $\pmb{K}$. Therefore, $\sx>0$.
\end{Lemma}

Lemma \ref{lemma:covariance} implies that if the initial state and the noise driving the dynamics are non-degenerate, the covariance matrices of the state dynamics are positive definite for any policy $\pmb{K}$. However, the covariance matrix may be degenerate in many applications, especially when inventory processes are involved. (See, for example, the liquidation problem \eqref{eq:liquidaton_states}.) In this case, some problem-dependent conditions are needed to guarantee that $\sx>0$ holds. See more discussion on the condition $\sx>0$  for the liquidation problem in Section \ref{sec:singleassetliq}. In the light of this we will assume $\sx>0$ in the analysis of the convergence of the algorithm in Sections \ref{sc:single_agent_model_based} and \ref{sc:single_agent_model_free}.

Similarly, we define $\sr$ and $\sq$ to be the smallest values of all the minimum singular values of $\pmb{R}$ and $\pmb{Q}$:
\begin{eqnarray}
    \sr&=&\min_t \sigma_{\min}(R_t),\label{Defn of barSigmaR}\\
   \sq&=&\min_t\sigma_{\min}(Q_t). \label{Defn of barSigmaQ}
\end{eqnarray}
Under Assumption \ref{ass:parameters}, we have $\sr>0$ and $\sq>0$. 

{We write  $\HH =\left\{ h\,|\, h \,\,\text{are polynomials in the model parameters}\right\}$ and $\HH(.)$ when there are other dependencies. The model parameters are in terms of $d$, $k$, $\frac{1}{\|A\|}, \frac{1}{\|A\|+1}, \|A\|, \frac{1}{\|B\|}, \frac{1}{\|B\|+1}, \|B\|, \frac{1}{\vertiii{\pmb{R}}}, \frac{1}{\vertiii{\pmb{R}}+1},\\
\vertiii{\pmb{R}}, \frac{1}{\|W\|}, \frac{1}{\|W\|+1},
\|W\|$, $\frac{1}{\sq}$, $\frac{1}{\sq+1}$, $\sq$, $\frac{1}{\sr}$, $\frac{1}{\sr+1}$, $\sr$, $\frac{1}{\sx}$, $\frac{1}{\sx+1}$, $\sx$, $\vertiii{\pmb{Q}}$, $\mathbb{E}[x_0x_0^\top]$, and $\frac{1}{\mathbb{E}[x_0x_0^\top]}$.}

\paragraph{Exact Gradient Descent.} We consider the following {\it exact} gradient descent updating rule to find the optimal solution \eqref{optimal_solution},
\begin{eqnarray}\label{eq:EGD}
K_t^{n+1} = K_t^n -\eta \nabla_{t}C(\pmb{K}^n),\,\, \forall\, 0 \leq t \leq T-1,
\end{eqnarray}
where  $n$ is the number of iterations, $\nabla_{t}C(\pmb{K}) = \frac{\partial C(\pmb{K})}{\partial {K_t}}$ is the gradient of $C(\pmb{K})$ with respect to $K_t$, and $\eta$ is the step size. {We further denote  $\nabla C(\pmb{K})=(\nabla_0 C(\pmb{K}),\cdots,\nabla_{T-1}C(\pmb{K}))$.}

Let us define the state covariance matrix 
\begin{equation}\label{eq:Sigma_t}
    \Sigma_{t} = \mathbb{E}\left[x_tx_t^\top\right], \ t=0,1,\cdots,T,
\end{equation}
where $\{x_t\}_{t=1}^T$ is a state trajectory generated by $\pmb{K}$.
Further define a matrix $\Sigma_{\pmb{K}}$ as the sum of $\Sigma_t$, 
\begin{equation}\label{sigma_K}
    \Sigma_{\pmb{K}} = \sum_{t=0}^T \Sigma_t= \mathbb{E}\Big[\sum_{t=0}^{T} x_tx_t^{\top}\Big].
\end{equation}
Then, the main result for this setting is the following.
\begin{Theorem}[Global Convergence of Gradient Methods]\label{thm:convergence_egd}
Assume Assumption \ref{ass:parameters} holds. Further assume {$\sx>0$} and $C(\pmb{K}^0)$ is finite. Then, for an appropriate (constant) setting of the stepsize {$\eta\in\mathcal{H}(\frac{1}{C(\pmb{K}^0)+1})$},
and for {$\epsilon>0$, if we have} 
\begin{equation*}
    N \geq \frac{\|\Sigma_{\pmb{K}^*}\|}{2\eta\sx^2\sr}\log\frac{C(\pmb{K}^0)-C(\pmb{K}^*)}{\epsilon},
\end{equation*}
the exact gradient descent method \eqref{eq:EGD} enjoys the following performance bound:
\begin{equation*}
    C(\pmb{K}^N)-C(\pmb{K}^*)\leq\epsilon.
\end{equation*}
\end{Theorem}

The proof of Theorem \ref{thm:convergence_egd} relies on the regularity of the LQR problem, some properties of the gradient descent dynamics, and the perturbation analysis of the  covariance matrix of the controlled dynamics.

\subsection{Regularity of the LQR Problem and Properties of the Gradient Descent Dynamics} \label{sec:model_based_regularity}
Let us start with the analysis of some properties of the LQR problem \eqref{LQR_problem}-\eqref{dynamics}.
To start, Proposition \ref{prop:P_positive} focuses on the well-definedness of the Ricatti system $\{P^{\pmb{K}}_t\}_{t=0}^T$ induced by a control $\pmb{K}$; Lemma \ref{lemma 1} gives a representation of the gradient term; {Lemma \ref{lemma 11} and Lemma \ref{lemma 12} provide the gradient dominance condition and a smoothness condition on the cost function $C(\pmb{K})$ with respect to policy $\pmb{K}$, respectively; and finally, Lemma \ref{lemma 13} gives two useful upper bounds on Ricatti system and state covariance matrices.}\\

In the finite time horizon setting, define $P_t^{\pmb{K}}$ as the solution to
\begin{equation}\label{Definition of P}
    P_t^{\pmb{K}} = Q_t+K_t^{\top}R_tK_t+\left(A-BK_t\right)^{\top}P_{t+1}^{\pmb{K}}\left(A-BK_t\right),\quad t = 0,1,\cdots,T-1,
\end{equation}
with terminal condition 
\begin{equation*}
    P_{T}^{\pmb{K}} = Q_T.
\end{equation*}
{Note that \eqref{Definition of P} is equivalent to the Riccati equation \eqref{eqn:P_t} with optimal $K_t=K_t^*$ as given by \eqref{optimal_k}. We have the following result on the well-definedness of $P^{\pmb{K}}_t$ and the proof can be found in Appendix \ref{appendix:missing_proofs_sec3_1}.}
\begin{Proposition}\label{prop:P_positive}
Under Assumption \ref{ass:parameters}, the matrices $ P_t^{\pmb{K}}$ for $t=0,1,\dots,T$ derived from \eqref{Definition of P} are positive definite.
\end{Proposition}
To ease the exposition, we write $P_t^{\pmb{K}}$ as $P_t$ when there is no confusion. Then the cost of $\pmb{K}$ can be rewritten as
\begin{equation*}
    C(\pmb{K}) = \mathbb{E}_{x_0\sim\mathcal{D}}\Big[x_0^{\top}P_0x_0+{L_0}\Big],
\end{equation*}
where, for $t=0,1,\cdots,T-1$,
\begin{equation}\label{qt}
    { L_t} = { L_{t+1}} + \mathbb{E}[w_t^\top P_{t+1}w_t] = { L_{t+1}} + \Tr(WP_{t+1}),
\end{equation}
with ${ L_T} = 0$. To see this,
\begin{equation*} 
\begin{split}
&\mathbb{E}[x_0^{\top}P_0x_0] + { L_0}  = \mathbb{E}\left[x_0^{\top}Q_0x_0 + x_0^{\top}K_0^{\top}R_0K_0x_0 +  x_0^{\top}\left(A-BK_0\right)^{\top}P_{1}\left(A-BK_0\right)x_0+ \sum_{t=0}^{T-1} w_t^\top P_{t+1}w_t\right]\\ 
 & = \mathbb{E}\left[x_0^{\top}Q_0x_0 + u_0^{\top}R_0u_0 +  x_1^{\top}P_{1}x_1 + \sum_{t=1}^{T-1} w_t^\top P_{t+1}w_t\right] = \mathbb{E}\Big[\sum_{t=0}^{T-1}\left(x_t^{\top}Q_tx_t+u_t^{\top}R_tu_t\right)+x_T^{\top}Q_Tx_T\Big].
\end{split}
\end{equation*}
In addition, define
\begin{equation}\label{eq:barEt}
    E_t = (R_t+B^{\top}P_{t +1}B)K_t-B^{\top}P_{t +1}A,\ t=0,1,\cdots,T-1.
\end{equation}
Then we have the following representation of the gradient term.
\begin{Lemma} \label{lemma 1}
 The policy gradient has the following representation, for $t = 0,1,\cdots,T-1$,
\begin{equation*}
\begin{split}
    \nabla_t C(\pmb{K}) &=  2\left(\left(R_t+B^{\top}P_{t +1}B\right)K_t-B^{\top}P_{t+1}A\right) \mathbb{E}\left[x_t x_t^{\top}\right]= 2E_{t}\Sigma_{t}.
\end{split}
\end{equation*}

\end{Lemma}
\begin{proof}
Since
\begin{equation*}
\begin{split}
    C(\pmb{K})  = \mathbb{E}\Big[x_0^{\top}P_0x_0+{ L_0}\Big]= \mathbb{E}\Big[x_0^{\top}(Q_0+K_0^{\top}R_0K_0)x_0 + x_0^{\top}(A-BK_0)^{\top}P_1(A-BK_0)x_0 + \sum_{t=0}^{T-1}w_t^\top P_{t+1}w_t\big], 
\end{split}
\end{equation*}
we have
\begin{equation*}
\begin{split}
    \nabla_0 C(\pmb{K})  = \frac{\partial C(\pmb{K})}{\partial K_0} = \mathbb{E}\Big[2R_0K_0x_0x_0^{\top} - 2B^{\top}P_1(A-BK_0)x_0x_0^{\top}\Big]  = 2E_0\mathbb{E}\Big[x_0 x_0^{\top}\Big] = 2E_0\Sigma_0.
\end{split}
\end{equation*}
Similarly, $\forall\,t=0,1,\cdots,T-1$, 
\begin{equation*}
\begin{split}
    \nabla_t C(\pmb{K}) =  2\left(\left(R_t+B^{\top}P_{t +1}B\right)K_t-B^{\top}P_{t+1}A\right) \mathbb{E}[x_t x_t^{\top}]= 2E_{t}\mathbb{E}\Big[x_{t} x_{t}^{\top}\Big] = 2E_{t}\Sigma_t,
\end{split}
\end{equation*}
where the expectation $\mathbb{E}$ is taken with respect to both initial distribution $x_0\sim \mathcal{D}$ and noises $\pmb{w}$.
\end{proof}

In classical optimization theory \cite{FGKM2018}, gradient domination and smoothness of the objective function are two key conditions to guarantee the global convergence of the gradient descent methods. To prove that $C(\pmb{K})$ is gradient dominated, we first prove Lemma \ref{lemma 11}, which indicates that for a policy $\pmb{K}$, the distance between $C(\pmb{K})$ and the optimal cost $C(\pmb{K}^*)$ is bounded by the sum of the magnitude of the gradient $\nabla_t C(\pmb{K})$ for $t=0,1,\cdots,T-1$.

\begin{Lemma}\label{lemma 11}
Assume Assumption \ref{ass:parameters} and {$\sx>0$}. Let $\pmb{K}^*$ be an optimal policy and $C(\pmb{K})$ be finite, then
\begin{equation*}
    \sx \sum_{t=0}^{T-1} \frac{1}{\|R_t+B^{\top}P_{t+1} B\|}\Tr(E_t^{\top}E_t) \leq C(\pmb{K})-C(\pmb{K}^*)\leq \frac{\|\Sigma_{\pmb{K}^*}\|}{4\sx^2\sr}\sum_{t=0}^{T-1}\Tr(\nabla_t C(\pmb{K})^{\top}\nabla_t C(\pmb{K})),
\end{equation*}
where $\sx$ and $\sq$ are defined in \eqref{Defn of barmu} and \eqref{Defn of barSigmaR}.
\end{Lemma}

We defer the proof of Lemma \ref{lemma 11} to Appendix \ref{appendix:missing_proofs_sec3_1}. {Lemma \ref{lemma 11} implies that when the gradient becomes small, the value of the objective function is close to $C(\pmb{K}^*)$. Now we consider the smoothness condition of the objective function. Recall that a function $f:\mathbb{R}^n\rightarrow \mathbb{R}$ is said to be smooth if
\[
|f(x)-f(y)-\nabla f(y)^\top (x-y)|\leq \frac{M}{2}\|x-y\|^2,\ \forall\, x,y\in\mathbb{R}^n,
\]
for some finite constant $M$. In general, it is difficult to characterize the smoothness of $C(\pmb{K})$, since it may blow up when $A-BK_t$ is large. Here we will prove that $C(\pmb{K})$ is ``almost'' smooth, in the sense that when $\pmb{K}^\prime$ is sufficiently close to $\pmb{K}$, $C(\pmb{K}')-C(\pmb{K})$ is bounded by the sum of the first and second order terms in $\pmb{K}-\pmb{K}^\prime$.
}
\begin{Lemma}[``Almost Smoothness'']\label{lemma 12}
Let $\{x_t'\}$ be the sequence of states for a single trajectory generated by $\pmb{K}'$ starting from $x_0'=x_0$. Then, $C(\pmb{K})$ satisfies
\begin{equation}\label{eq:smootheqn}
     C(\pmb{K}')-C(\pmb{K})=\sum_{t=0}^{T-1} \Big[2\Tr\Big(\Sigma_t'(K_t'-K_t)^{\top}E_t\Big) + \Tr\Big(\Sigma_t'(K_t'-K_t)^{\top}(R_t+B^{\top}P_{t +1}B)(K_t'-K_t)\Big)\Big],
\end{equation}
where $\Sigma_t'=\mathbb{E}\left[x_t'(x_t')^{\top}\right]$.
\end{Lemma}

We defer the proof of Lemma \ref{lemma 12} to Appendix \ref{appendix:missing_proofs_sec3_1}. {To see why Lemma \ref{lemma 12} is related to the smoothness, observe that when $\pmb{K}^\prime$ is sufficiently close to $\pmb{K}$, in the sense that
\[
\Sigma_t^\prime \approx \Sigma_t + O(\|K_t-K_t^\prime\|),\ \forall t=0,1,\cdots,T-1,
\]
the first term in \eqref{eq:smootheqn} will behave as $\Tr\left((K_t-K_t^\prime)\nabla_tC(\pmb{K})\right)$ by Lemma \ref{lemma 1}, and the second term in \eqref{eq:smootheqn} will be of second order in $K_t-K_t^\prime$. 

To utilize Lemmas $\ref{lemma 11}$ and $\ref{lemma 12}$ {in the proof of  Theorem \ref{thm:convergence_egd}}, we need to further bound $P_{t}$ and $\Sigma_{\pmb{K}}$, which is provided below in Lemma \ref{lemma 13}. The proof can be found in Appendix \ref{appendix:missing_proofs_sec3_1}.}

\begin{Lemma}\label{lemma 13}
Assume Assumption \ref{ass:parameters} holds, and {$\sx>0$}. Then we have
\begin{equation*}
    \|P_t\|\leq \frac{C(\pmb{K})}{\sx},\ \|\Sigma_{\pmb{K}}\|\leq \frac{C(\pmb{K})}{\sq},
\end{equation*}
where $\sx$ and $\sq$ are defined in \eqref{Defn of barmu} and \eqref{Defn of barSigmaQ}.
\end{Lemma}

\subsection{Perturbation Analysis of $\Sigma_K$}
\label{sec:model_based_perturbation}
First, let us define two linear operators on symmetric matrices. For {$X\in\mathbb{R}^{d\times d}$} we set
\begin{equation*}
    \mathcal{F}_{K_t}(X)=(A-BK_t)X(A-BK_t)^{\top}, \quad \text{and}\quad
    \mathcal{T}_{\pmb{K}}(X):=X+\sum_{t=0}^{T-1} \Pi_{i=0}^t (A-BK_i)\,X\,\Pi_{i=0}^t (A-BK_{t-i})^{\top}.
\end{equation*}
If we write $\mathcal{G}_t =\mathcal{F}_{K_t}\circ\mathcal{F}_{K_{t-1}}\circ\cdots\circ \mathcal{F}_{K_0}$, then
\begin{eqnarray}
\mathcal{G}_t(X) &=& \mathcal{F}_{K_t}\circ \mathcal{G}_{t-1}(X) = \Pi_{i=0}^t (A-BK_i)\,X\,\Pi_{i=0}^t (A-BK_{t-i})^{\top}, \,\, {\rm and }\label{Gt}\\
    \mathcal{T}_{\pmb{K}}(X) &=& X\,+\,\sum_{t=0}^{T-1} \mathcal{G}_t(X). \label{eq:Tau_K defn}
\end{eqnarray}
 We first show the relationship between the operator  $\mathcal{T}_{\pmb{K}}$ and the quantity $\Sigma_{\pmb{K}}$. The proof can be found in Appendix \ref{appendix:missing_proofs_sec3_1}.

\begin{Proposition}\label{prop:Sigma_Gamma_relation} For $T \ge 2$, we have that
\begin{equation}\label{eqn:Sigma_Gamma_relation}
   \Sigma_{\pmb{K}}  =\mathcal{T}_{\pmb{K}}(\Sigma_0)+\Delta(\pmb{K},W),
\end{equation}
where $ \Delta(\pmb{K},W) = \sum_{t=1}^{T-1} \sum_{s=1}^{t}\,D_{t,s} W D_{t,s}^{\top}+T\,W,$
with $D_{t,s} = \Pi_{u=s}^{t} (A-BK_{u})$ (for $s=1,2,\cdots,t$),  and $ \Sigma_0=\mathbb{E}\left[x_0x_0^{\top}\right]$. 
\end{Proposition}
Let
\begin{eqnarray}\label{rho}
\rho := \max\Big\{\max_{0\leq t \leq T-1}\|A-BK_t\|,\max_{0\leq t \leq T-1}\|A-BK^{\prime}_t\|,1+ \xi\Big\},
\end{eqnarray}
for some small constant $\xi>0$.
Then we have the following result on perturbations of $\Sigma_{\pmb{K}}$.
\begin{Lemma}[Perturbation Analysis of $\Sigma_K$]\label{lemma:perturbation}
Assume Assumption \ref{ass:parameters} holds. Then
{
\begin{equation*}
\begin{split}
     \Big\|\Sigma_{\pmb{K}}-\Sigma_{\pmb{K}^{\prime}}\Big\|&\leq\Big\|(\mathcal{T}_{\pmb{K}}-\mathcal{T}_{\pmb{K}^{\prime}})(\Sigma_0)\Big\|+\left\|\Delta(\pmb{K},W)-\Delta(\pmb{K}^\prime,W)\right\|\\
     & \leq { \frac{ \rho^{2T}-1}{\rho^2-1} \left(\frac{C(\pmb{K})}{\sq}+T\|W\|\right)\left(2\rho\,\|B\|\,\vertiii{\pmb{K}-\pmb{K}^\prime}+\|B\|^2\,\vertiii{\pmb{K}-\pmb{K}^\prime}^2\right)}.
\end{split}
\end{equation*}}
\end{Lemma}

\begin{Remark}{\rm
By the definition of $\rho$ in \eqref{rho}, we have $\rho \ge 1+\xi>1$. This regularization term $1+\xi$ is defined for ease of exposition. Alternatively, if we define $\rho := \max\Big\{\max_{0\leq t \leq T-1}\|A-BK_t\|,\max_{0\leq t \leq T-1}\|A-BK^{\prime}_t\|\Big\}$, a similar analysis can still be carried out by considering the different cases: $\rho<1$, $\rho=1$ and $\rho>1$. Note that for the  infinite horizon problem, the spectral radius of $A-BK$ needs to be smaller than 1 to guarantee the stability of the system (see \cite{FGKM2018}). In our setting with finite horizon, instability is not an issue and we do not need a condition on the boundedness of $\rho$. However, we will show later that $\rho$ does appear in the sample complexity results. The smaller the $\rho$, the smaller the sample complexity.}
\end{Remark}

\noindent The proof of Lemma \ref{lemma:perturbation} is based on 
the following Lemmas \ref{lemma 19} and \ref{Lemma 20}, which establish the 
 Lipschitz property for the operators $\mathcal{F}_{K_t}$ and  $\mathcal{G}_t$, respectively.
\begin{Lemma}\label{lemma 19}
It holds that, $\forall\,t=0,1,\cdots,T-1$,
\begin{eqnarray}\label{F_bound}
\|\mathcal{F}_{K_t}-\mathcal{F}_{K^{\prime}_t}\| \leq 2\|A-BK_t\|\|B\|\|K_t-K_t^{\prime}\| +\|B\|^2 \|K_t-K_t^{\prime}\|^2.
\end{eqnarray}
\end{Lemma}
\noindent We refer to \cite[Lemma 19]{FGKM2018} for the proof of Lemma \ref{lemma 19}.\\

Recall the definition of $\mathcal{G}_t$ in \eqref{Gt} associated with $\pmb{K}$, similarly let us define $\mathcal{G}^{\prime}_t =\mathcal{F}_{K^{\prime}_t}\circ\mathcal{F}_{K^{\prime}_{t-1}}\circ\cdots\circ \mathcal{F}_{K^{\prime}_0}$ for policy $\pmb{K}^{\prime}$. Then we have the following perturbation analysis for  $\mathcal{G}_t$.

\begin{Lemma}[Perturbation Analysis for  $\mathcal{G}_t$]\label{Lemma 20}
For any symmetric matrix $\Sigma\in\mathbb{R}^{d\times d}$, we have that 
\begin{eqnarray}\label{Tau_closeness_G}
\sum_{t=0}^{T-1}\Big\|(\mathcal{G}_{t}-\mathcal{G}_{t}^{\prime})(\Sigma)\Big\| \leq \frac{\rho^{2T}-1}{\rho^2-1} \Big( \sum_{t=0}^{T-1}\|\mathcal{F}_{K_t}-\mathcal{F}_{K^{\prime}_t}\|\Big)\|\Sigma\|.
\end{eqnarray}
\end{Lemma}
We defer the proof of Lemma \ref{Lemma 20} to Appendix \ref{appendix:missing_proofs_sec3_2}. The following perturbation analysis on $\mathcal{T}$ follows immediately from 
Lemma \ref{Lemma 20}.

\begin{Corollary}\label{Corr 20} For any symmetric matrix $\Sigma\in\mathbb{R}^{d\times d}$, we have
\begin{eqnarray}\label{Tau_closeness}
\Big\|(\mathcal{T}_{\pmb{K}}-\mathcal{T}_{\pmb{K}^{\prime}})(\Sigma)\Big\| \leq \frac{\rho^{2T}-1}{\rho^2-1} \Big( \sum_{t=0}^{T-1}\|\mathcal{F}_{K_t}-\mathcal{F}_{K^{\prime}_t}\|\Big)\|\Sigma\|,
\end{eqnarray}
where $\rho$ is defined in \eqref{rho}.
\end{Corollary}


\noindent Now we are ready for the proof of Lemma \ref{lemma:perturbation}.
\begin{proof}[Proof of Lemma \ref{lemma:perturbation}]
Using Lemma \ref{lemma 19},
\begin{equation*}
    \begin{split}
        \sum_{t=0}^{T-1} \|\mathcal{F}_{K_t}-\mathcal{F}_{K^{\prime}_t}\| &= \sum_{t=0}^{T-1} \Big( 2\|A-BK_t\|\|B\|\|K_t-K_t^{\prime}\| +\|B\|^2 \|K_t-K_t^{\prime}\|^2\Big) \\
        & \leq  2\rho\|B\|\sum_{t=0}^{T-1}\|K_t-K_t^{\prime}\|+\|B\|^2\sum_{t=0}^{T-1}  \|K_t-K_t^{\prime}\|^2.
    \end{split}
\end{equation*}

In the same way as for the proof of Lemma \ref{Lemma 20},  we have, $\forall\,t=1,\cdots,T-1$,
\begin{eqnarray}\label{eq:sum_D_bound}
 \sum_{s=1}^{t}\left\|D_{t,s}WD_{t,s}^{\top}-D^{\prime}_{t,s}W(D^{\prime}_{t,s})^{\top}\right\| \leq \frac{\rho^{2T}-1}{\rho^2-1} \left(\sum_{s=0}^{t}\|\mathcal{F}_{K_s}-\mathcal{F}_{K^{\prime}_s}\| \right) \|W\|.
\end{eqnarray}
By Proposition \ref{prop:Sigma_Gamma_relation}, Corollary \ref{Corr 20}, \eqref{eq:Tau_K defn} and \eqref{eq:sum_D_bound}, we have
\begin{equation}\label{eq:Colla_intem}
\begin{split}
  \Big\|\Sigma_{\pmb{K}}-\Sigma_{\pmb{K}^\prime}\Big\| & \leq 
  \Big\|(\mathcal{T}_{\pmb{K}}-\mathcal{T}_{\pmb{K}^{\prime}})(\Sigma_0)\Big\| + \sum_{t=1}^{T-1} \sum_{s=1}^{t}\,\Big\|D_{t,s} W D_{t,s}^{\top}-D_{t,s}^\prime W (D_{t,s}^\prime)^{\top}\Big\|\\
     & \leq \frac{\rho^{2T}-1}{\rho^2-1} \Big( \sum_{t=0}^{T-1}\|\mathcal{F}_{K_t}-\mathcal{F}_{K^{\prime}_t}\|\Big)\left(\|\Sigma_0\|+T\|W\|\right) \\
     & \leq {\frac{ \rho^{2T}-1}{\rho^2-1} \left(\frac{C(\pmb{K})}{\sq}+T\|W\|\right)\left(2\rho\,\|B\|\,\vertiii{\pmb{K}-\pmb{K}^\prime}+\|B\|^2\,\vertiii{\pmb{K}-\pmb{K}^\prime}^2\right)}.
\end{split}
\end{equation}
The last inequality holds since $\|\Sigma_0\| \leq \|\Sigma_{\pmb{K}}\|\leq \frac{C(\pmb{K})}{\sq}$ by Lemma \ref{lemma 13}.
\end{proof}

\subsection{Convergence and Complexity Analysis} 
\label{sec:model_based_convergence}

We now provide the proof of Theorem \ref{thm:convergence_egd} after two preliminary Lemmas.


\begin{Lemma}\label{lemma 24}
Assume Assumption \ref{ass:parameters} holds,  { $\sx>0$}, and that 
\begin{equation}\label{eq:gd}
    K_t^{\prime}=K_t-\eta\nabla_{t} C(\pmb{K}), \qquad {\rm where }
\end{equation}
\begin{equation}\label{stepsize condition}
    \eta \leq \min\left\{\frac{(\rho^2-1)\sq\sx}{2T(\rho^{2T}-1)(2\rho+1)(C(\pmb{K})+\sq T\|W\|)\|B\|\max_{t}\{\|\nabla_tC(\pmb{K})\|\}},\frac{1}{2C_1}\right\}, \qquad {\rm with}
\end{equation}
\begin{equation}\label{eq:C1}
    C_1 = \left(\frac{C(\pmb{K})}{\sq}+T\|W\|\right)\left({\frac{ (2\rho+1)\|B\|(\rho^{2T}-1)}{(\rho^2-1)\sx}}\sum_{t=0}^{T-1}\|\nabla_t C(\pmb{K})\|\right)
    +\frac{{2}C(\pmb{K})}{\sq}\sum_{t=0}^{T-1}\|R_t+B^{\top}P_{t +1}B \|.
\end{equation}
Then we have
\begin{equation*}
    C(\pmb{K}^{\prime})-C(\pmb{K}^*)\leq \Big(1-2\eta\sr\frac{\sx^2}{\|\Sigma_{\pmb{K}^*}\|}\Big)\Big(C(\pmb{K})-C(\pmb{K}^*)\Big).
\end{equation*}
\end{Lemma}
We defer the proof of Lemma \ref{lemma 24} to Appendix \ref{appendix:model_based_convergence}.

\begin{Lemma}\label{lemma 25}
Assume Assumption \ref{ass:parameters} holds and  {$\sx>0$}. Then we have that
\begin{equation*}
        \sum_{t=0}^{T-1} \|\nabla_t C(\pmb{K})\|^2 \leq 4\Big(\frac{C(\pmb{K})}{\sq}\Big)^2\frac{\max_t\|R_t+B^{\top}P_{t+1} B\|}{\sx}(C(\pmb{K})-C(\pmb{K}^*)),
\end{equation*}
and that:
\begin{equation*}
    \sum_{t=0}^{T-1} \|K_t\| \leq \frac{1}{\sr}\Big(\sqrt{T\cdot\frac{\max_t\|R_t+B^{\top}P_{t+1} B\|}{\sx}(C(\pmb{K})-C(\pmb{K}^*))}+\sum_{t=0}^{T-1} \|B^{\top}P_{t+1}A\|\Big).
\end{equation*}
\end{Lemma}

\begin{proof}
Using Lemma \ref{lemma 13} we have
\begin{equation*}
        \sum_{t=0}^{T-1} \|\nabla_t C(\pmb{K})\|^2 \leq 4\sum_{t=0}^{T-1}\Tr(\Sigma_tE_t^{\top}E_t\Sigma_t) \leq 4\sum_{t=0}^{T-1}\|\Sigma_t\|^2\Tr(E_t^{\top}E_t) \leq 4\Big(\frac{C(\pmb{K})}{\sq}\Big)^2\sum_{t=0}^{T-1}\Tr(E_t^{\top}E_t).
\end{equation*}
From Lemma \ref{lemma 11} we have
\begin{equation}\label{eqn:tr_bound}
C(\pmb{K})-C(\pmb{K}^*) \geq  \sx \sum_{t=0}^{T-1} \frac{1}{\|R_t+B^{\top}P_{t+1} B\|}\Tr(E_t^{\top}E_t) \geq
    \frac{\sx}{\max_t\|R_t+B^{\top}P_{t+1} B\|}\sum_{t=0}^{T-1} \Tr(E_t^{\top}E_t),
\end{equation}
and hence
\begin{equation*}
        \sum_{t=0}^{T-1} \|\nabla_t C(\pmb{K})\|^2 \leq 4\Big(\frac{C(\pmb{K})}{\sq}\Big)^2\frac{\max_t\|R_t+B^{\top}P_{t+1} B\|}{\sx}(C(\pmb{K})-C(\pmb{K}^*)).
\end{equation*}
For the second claim, using Lemma \ref{lemma 11} again,
\begin{equation*}
\begin{split}
    \sum_{t=0}^{T-1} \|K_t\| & = \sum_{t=0}^{T-1} \|(R_t+B^{\top}P_{t+1}B)^{-1}K_t(R_t+B^{\top}P_{t+1}B)\| \\
    & \leq \sum_{t=0}^{T-1} \frac{1}{\sigma_{\min}(R_t)}\|K_t(R_t+B^{\top}P_{t+1}B)\| 
     \leq \sum_{t=0}^{T-1} \frac{1}{\sigma_{\min}(R_t)}\Big(\|E_t\|+\|B^{\top}P_{t+1}A\|\Big) \\
    & \leq \sum_{t=0}^{T-1}\left( \frac{\sqrt{\Tr(E_t^{\top}E_t)}}{\sigma_{\min}(R_t)}+\frac{\|B^{\top}P_{t+1}A\|}{\sigma_{\min}(R_t)}\right) 
     \leq \frac{1}{\sr} \Big(\sqrt{T\cdot\sum_{t=0}^{T-1}\Tr(E_t^{\top}E_t)}+\sum_{t=0}^{T-1}\|B^{\top}P_{t+1}A\|\Big) \\
    & \leq \frac{1}{\sr}\Big(\sqrt{T\cdot\frac{\max_t\|R_t+B^{\top}P_{t+1} B\|}{\sx}(C(\pmb{K})-C(\pmb{K}^*))}+\sum_{t=0}^{T-1} \|B^{\top}P_{t+1}A\|\Big). \\
\end{split}
\end{equation*}
The second inequality holds by the definition of $E_t$ in \eqref{eq:barEt}, the second last step uses the Cauchy-Schwarz inequality, and the last inequality holds by \eqref{eqn:tr_bound}.
\end{proof}
\begin{proof}[Proof of Theorem \ref{thm:convergence_egd}]
In order to show the existence of a positive $\eta$ such that \eqref{stepsize condition} holds, it suffices to show there exists a positive lower bound on the RHS of \eqref{stepsize condition}. {By Lemma \ref{lemma 25} and the Cauchy-Schwarz inequality, 
\begin{equation}\label{bound for sum nablaCK}
\begin{split}
     \sum_{t=0}^{T-1} \|\nabla_t C(\pmb{K})\|
     \leq \sqrt{T\cdot\sum_{t=0}^{T-1} \|\nabla_t C(\pmb{K})\|^2}
    \leq \sqrt{4T\cdot\Big(\frac{C(\pmb{K})}{\sq}\Big)^2\frac{\max_t\|R_t+B^{\top}P_{t+1} B\|}{\sx}(C(\pmb{K})-C(\pmb{K}^*))}.
\end{split}
\end{equation} 
Note that if $d<ab+c$ for some $a>0$, $b>0$ $c>0$ and $d>0$, then $\frac{1}{d}>\frac{1}{(a+1)(b+1)(c+1)}$. Also $\frac{1}{a^n+1}>\frac{1}{(a+1)^n}$ for $a>0$ and $n\in \mathbb{N}^+$.  Therefore, based on \eqref{eq:C1} and \eqref{bound for sum nablaCK}, $\frac{1}{C_1}$ is bounded below by polynomials in $\frac{1}{\rho}$, $\frac{1}{C(\pmb{K})+1}$, $\frac{1}{\|B\|+1}$, $\frac{1}{\vertiii{\pmb{R}}+1}$, $\frac{1}{\|W\|+1}$, $\sx$, $\sq$, $\frac{1}{\sx + 1}$, and $\frac{1}{\sq + 1}$.}

Now we aim to show that $\frac{1}{\rho}$ is bounded below by some polynomials in the parameters. To see this, let us first show that $\rho$ is bounded above by polynomials in $\|A\|$, $\|B\|$, $\vertiii{\pmb{R}}$, $\frac{1}{\sx}$, $\frac{1}{\sr}$ and $C(\pmb{K})$.
Since $\|B\|\|K_t'-K_t\| \leq \frac{\sq\sx}{4C(\pmb{K})}
\leq \frac{1}{2}$ holds under the assumptions in Lemma \ref{lemma 24}, we have
\begin{equation*}
    \max_{0\leq t \leq T-1}\|A-BK^{\prime}_t\| \leq \max_{0\leq t \leq T-1}\left(\|A-BK_t\|+\|B\|\,\|K_t^{\prime}-K_t\|\right) \leq \max_{0\leq t \leq T-1}\|A-BK_t\|+\frac{1}{2}, \,\, {\rm thus}
\end{equation*}
\begin{equation}\label{rho eqn}
\begin{split}
      \rho &= \max\Big\{\max_{0\leq t \leq {T-1}}\|A-BK_t\|,\max_{0\leq t \leq {T-1}}\|A-BK^{\prime}_t\|, \ 1+\xi \Big\}\\
      &\leq \max\Big\{\max_{0\leq t \leq T-1}\|A-BK_t\|+\frac{1}{2},\ 1+\xi\Big\}
      \leq \max \Big\{ \|A\|+\|B\|\,\sum_{t=0}^{T-1}\|K_t\|+\frac{1}{2},\ 1+\xi\Big\}.
\end{split}
\end{equation}
Given the bound on $\sum_{t=0}^{T-1} \|K_t\|$ by Lemma \ref{lemma 25} and $\|P_t\|\leq \frac{C(\pmb{K})}{\sx}$ by Lemma \ref{lemma 13}, $\rho$ is bounded {above} by polynomials in $\|A\|$, $\|B\|$, $\vertiii{\pmb{R}}$, $\frac{1}{\sx}$, $\frac{1}{\sr}$ and $C(\pmb{K})$, or a constant $1+\xi$. {Therefore $\frac{1}{\rho}$ is bounded below by polynomials in $\frac{1}{\|A\|+1}$,  $\frac{1}{\|B\|+1}$, $\frac{1}{\vertiii{\pmb{R}}+1}$, ${\sx}$, ${\sr}$ and $\frac{1}{C(\pmb{K})+1}$, or a constant $\frac{1}{1+\xi}$.  Hence, by choosing $\eta\in\mathcal{H}(\frac{1}{C(\pmb{K}^0)+1})$ to be an appropriate polynomial in  $\frac{1}{C(\pmb{K}^0)}$, 
$\frac{1}{C(\pmb{K}^0)+1}$, $\frac{1}{\|A\|+1}$, 
$\frac{1}{\|B\|+1}$, $\frac{1}{\vertiii{\pmb{R}}+1}$, $\frac{1}{\|W\|+1}$, $\sx$, $\sq$, ${\sr}$, $\frac{1}{\sx + 1}$, and $\frac{1}{\sq + 1}$, \eqref{stepsize condition} is satisfied, since by performing gradient descent, $C(\pmb{K}^1)<C(\pmb{K}^0)$. Therefore,}
by Lemma \ref{lemma 24}, we have
\begin{equation*}
    C(\pmb{K}^1) - C(\pmb{K}^*) \leq \Big(1-2\eta\sr\frac{\sx^2}{\|\Sigma_{\pmb{K}^*}\|}\Big)\Big(C(\pmb{K}^0)-C(\pmb{K}^*)\Big),
\end{equation*}
which implies that the cost decreases at $t=1$. Suppose that $C(\pmb{K}^n)\leq C(\pmb{K}^0)$, then the stepsize condition in \eqref{stepsize condition} is still satisfied by Lemma \ref{lemma 25}. Thus, Lemma \ref{lemma 24} can again be applied for the update at round $n+1$ to obtain:
\begin{equation*}
     C(\pmb{K}^{n+1}) - C(\pmb{K}^*) \leq \Big(1-2\eta\sr\frac{\sx^2}{\|\Sigma_{\pmb{K}^*}\|}\Big)\Big(C(\pmb{K}^n)-C(\pmb{K}^*)\Big).
\end{equation*}
For $\epsilon>0$, provided
$N \geq \frac{\|\Sigma_{\pmb{K}^*}\|}{2\eta\sx^2\sr}\log\frac{C(\pmb{K}^0)-C(\pmb{K}^*)}{\epsilon}$,
we have
\begin{equation*}
     C(\pmb{K}^N)-C(\pmb{K}^*)\leq\epsilon.
\end{equation*}
\end{proof}

\section{Sample-based Policy Gradient Method with Unknown Parameters}\label{sc:single_agent_model_free}
In the setting with unknown parameters, the controller has only simulation access to the model; the model parameters, $A$, $B$, $\{Q_t\}_{t=0}^{T}$, $\{R_t\}_{t=0}^{T-1}$, are unknown. By using a zeroth-order optimization method to approximate the gradient, this section proves the policy gradient method with unknown parameters also leads to a global optimal policy, with both polynomial computational and sample complexities.

{Note that in this section, when bounding the Frobenius norm of a matrix, we usually treat the matrix as a stacked vector. Therefore we denote by $D=k\times d$ the dimension of the corresponding vector formed from the $\pmb{K}$ matrix for convenience in the proofs.}
Therefore in each iteration $n=1,2,\cdots,N$, we can update the policy  as, for $t=0,1,\cdots,T-1$,
\begin{eqnarray}\label{eqn:model_free_policy_update}
K^{n+1}_t = K^n_t -\eta \widehat{\nabla_t C(\pmb{K}^n)},
\end{eqnarray} {where $\widehat{\nabla_tC(\pmb{K}^n)}$ is the estimate of $\nabla_tC(\pmb{K}^n)$.} We analyze the following Algorithm \ref{alg:MFP}. 

\begin{algorithm}[H]
\caption{\textbf{Policy Gradient Estimation with Unknown Parameters}}
\label{alg:MFP}
\begin{algorithmic}[1]
    \STATE \textbf{Input}: $\pmb{K}$, number of trajectories $m$, smoothing parameter $r$, dimension $D$
        \FOR {$i\in\{1, \ldots, m\}$}
         \FOR {$t\in\{0, \ldots, T-1\}$}
           \STATE Sample the (sub)-policy at time $t$: $\widehat{K}^i_t = K_t +U^i_t$ where $U_t^i$ is drawn uniformly at random over matrices such that $\|U_t^i\|_F=r$.
           \STATE Denote  $\widehat{c_t}^i$ as the single trajectory cost with policy\\ $(\pmb{K}_{-t},\widehat{K}^i_t):=({K}_0,\cdots,{K}_{t-1},\widehat{K}^i_t,{K}_t,\cdots,K_{T-1})$ starting from $x^i_0 \sim \mathcal{D}$.
            \ENDFOR
        \ENDFOR
 \STATE Return the estimates of $\nabla_t C(\pmb{K})$ for each $t$:
 \begin{eqnarray}\label{eqn:biased_estimate_gradient}
\widehat{\nabla_t C(\pmb{K})} = \frac{1}{m}\sum_{i=1}^m \frac{D}{r^2}\,\widehat{c_t}^i\, U^i_t.
 \end{eqnarray}
\end{algorithmic}
\end{algorithm}

\begin{Remark}{\rm [Zeroth-order Optimization Approach in the Sub-routine \eqref{eqn:biased_estimate_gradient}] In the estimation of the gradient term \eqref{eqn:biased_estimate_gradient}, we adopt a zeroth-order optimization method, using only query access to a sample of the reward function $c(\cdot)$ at input points $\pmb{K}$, without querying the gradients and higher order derivatives of $c(\cdot)$. In a similar way to the observation in \cite{FGKM2018}, the objective $C(\pmb{K})$ may not be finite for every policy $\pmb{K}$ when Gaussian smoothing is applied, therefore  $\mathbb{E}_{\pmb{U} \sim \mathcal{N} (0,\sigma^2I)}[C(\pmb{K}+\pmb{U})]$ may not be well-defined. 
This is avoidable by smoothing over the surface of a ball. The step \eqref{eqn:biased_estimate_gradient} (in Algorithm \ref{alg:MFP}) provides a procedure to find an (bounded bias) estimate $\nabla\widehat{ C(\pmb{K})} $ of $\nabla  C(\pmb{K})$. 
}
\end{Remark}

{The idea in \eqref{eqn:biased_estimate_gradient} is to approximate the gradient of a function by only using the function values (see e.g. Lemma 2.1 in \cite{Flaxman2005}). Observe that by a Taylor expansion to first order $\mathbb{E}[f(x+U)U] \approx \mathbb{E}[(\nabla f(x) \cdot U)U] = \nabla f(x) r^2/D$ when $x\in\mathbb{R}^D$ and $U$ is uniform over the surface of the ball of radius $r$ in $\mathbb{R}^D$. Thus the gradient of the function $f$ at $x$ can be estimated by averaging over the samples $\frac{D}{r^2}f(x+U)U$.}


Note that in Algorithm \ref{alg:MFP}, we require $mNT^2$ samples to perform the policy gradient method $N$ times.

{To guarantee the global convergence of the sample-based algorithm (Algorithm \ref{alg:MFP}), we propose some conditions on the distribution of $x_0$ and $\{w_t\}_{t=0}^{T-1}$, in addition to the finite second moment condition specified in Section \ref{sc:single_agent_setup}.}
\begin{Definition}
A zero-mean random variable $X$
\begin{enumerate}
    \item is said to be sub-Gaussian with variance proxy $\sigma^2$ and we write $X \in SG(\sigma^2)$
if its moment generating function satisfies
$\mathbb{E}[\exp(\lambda X)] \leq \exp\left(\frac{\lambda^2\sigma^2}{2}\right)$ for all $\lambda\in\mathbb{R}$.
\item is said to be sub-exponential with parameters $(\nu^2,\alpha)$ and we write $X\in SE(\nu^2,\alpha)$, if  
$\,\mathbb{E}[\exp(\lambda X)] \leq \exp\left(\frac{\lambda^2\nu^2}{2}\right)$
for any $\lambda$ such that $|\lambda| \leq \frac{1}{\alpha}$.
\end{enumerate}
\end{Definition}
We assume the initial distribution and the noise in the state process dynamics satisfy the following assumptions.
\begin{Assumption}[Initial State and Noise Process (II)]\label{ass:State-2}
\begin{enumerate}
    \item Initial state:  $x_0= \widetilde{W}_0 z_0$ where \\
    $z_0 = (z_{0,1}\cdots,z_{0,d})\in \mathbb{R}^d$ is a random
vector with independent components $z_{0,i}$ which are sub-Gaussian, mean-zero, and have sub-Gaussian parameter $\sigma_0^2$; $\widetilde{W}_0 \in \mathbb{R}^{d \times d}$ is an unknown and deterministic matrix.
    \item Noise process: $w_t = \widetilde{W} v_{t}$  where
    $v_t := (v_{t,1}\cdots,v_{t,d})\in \mathbb{R}^d$ are  IID  and independent from $x_0$. $v_t$ has independent components $v_{t,i}$ which are sub-Gaussian, mean-zero, and have sub-Gaussian parameter $\sigma_w^2$, $\forall\,t=0,1,\cdots,T-1$. $\widetilde{W} \in \mathbb{R}^{d \times d}$ is an unknown and deterministic matrix.
\end{enumerate}
\end{Assumption}

{Note that Assumptions \ref{ass:State} and \ref{ass:State-2} serve different purposes in this paper. Assumption \ref{ass:State} provides one sufficient condition to assure $\sx>0$. Assumption \ref{ass:State-2} is used to guarantee the convergence of the sample based algorithm (Algorithm \ref{alg:MFP}).}

{In addition to the model parameters specified in Section \ref{sc:single_agent_model_based}, here we assume $\HH(\cdot)$ includes polynomials that are also functions of $\sigma_0$, $\frac{1}{\sigma_0}$, $\frac{1}{\sigma_0+1}$, $\sigma_w$, $\frac{1}{\sigma_w}$, $\frac{1}{\sigma_w+1}$ $\|\widetilde{W}\|$,$\frac{1}{\|\widetilde{W}\|}$, $\frac{1}{\|\widetilde{W}\|+1}$, $\|\widetilde{W}_0\|$, $\frac{1}{\|\widetilde{W}_0\|}$, and $\frac{1}{\|\widetilde{W}_0\|+1}$. }

\begin{Theorem}\label{thm:model_free}
Assume Assumptions \ref{ass:parameters} and \ref{ass:State-2} hold and further assume {$\sx>0$} and $C(\pmb{K}^0)$ is finite. At every step the policy is updated as in \eqref{eqn:model_free_policy_update}, that is
\begin{eqnarray*}
K^{n+1}_t = K^n_t -\eta \widehat{\nabla_t C(\pmb{K}^n)},
\end{eqnarray*}
with {$\eta\in\mathcal{H}(\frac{1}{C(\pmb{K}^0)+1})$}
and $\widehat{\nabla_t C(\pmb{K}^n)}$ is computed with hyper-parameters $(r,m)$ such that
{$r<1/\overline{h}_{radius}$ and $m>\overline{h}_{sample}$ with some fixed polynomials $\overline{h}_{radius}\in\HECKI$ and $\overline{h}_{sample}\in$ $\HCKIR$}. Then for $\epsilon>0$, if we have
\begin{equation*}
    N \geq \frac{\|\Sigma_{\pmb{K}^*}\|}{\eta\sx^2\sr}\log\frac{C(\pmb{K}^0)-C(\pmb{K}^*)}{\epsilon},
\end{equation*}
it holds that $C(\pmb{K}^N)-C(\pmb{K}^*)\leq \epsilon$ with high probability (at least $1-\exp(-D)$).
\end{Theorem}
Note that $\overline{h}_{sample}$  is quadratic in $1/\epsilon$ (when the logarithmic order is omitted) and cubic in dimension $D$. The proof of Theorem \ref{thm:model_free} is based on a perturbation analysis of $C(\pmb{K})$ and $\nabla_t C(\pmb{K})$, smoothing and the gradient descent analysis of the procedures in Algorithm \ref{alg:MFP}. We provide the perturbation analysis and the smoothing analysis in Sections \ref{sec:model_free_perturbation} and \ref{sec:smoothing}, respectively. We defer the proof of Theorem \ref{thm:model_free} to Section \ref{sec:thm2}.

\paragraph{Projected Policy Gradient Method.} {In many situations constrained optimization problems arise and 
the \emph{projected} gradient descent method is one popular approach to solve such problems.} Recall the projection of a point $\pmb{y}=(y_0,\cdots,y_{T-1})$ with $y_t\in\mathbb{R}^{k\times d}$ ($t=0,1,\cdots,T-1$) onto a set $\mathcal{S}\subset\mathbb{R}^{k\times (T \times d)}$ is defined as
\begin{equation}\label{eqn:projection}
    \Pi_{\mathcal{S}}(\pmb{y}) = \argmin_{\pmb{x}\in \mathcal{S}}\frac{1}{2}\sum_{t=0}^{T-1}\left\|x_t-y_t\right\|^2_{F}.
\end{equation}
Then the projected policy gradient (PPG) updating rule can be defined as
\begin{equation}\label{eq:projected_updateing_rule}
    \pmb{K}^{n+1} = \Pi_{\mathcal{S}}\left(\pmb{K}^n-\eta\widehat{\nabla C(\pmb{K}^n)}\right),
\end{equation}
where 
$\widehat{\nabla C(\pmb{K}^n)}=\left(\widehat{\nabla_0 C(\pmb{K}^n}),\cdots,\widehat{\nabla_{T-1} C(\pmb{K}^n})\right)$ denotes the estimate of $\nabla C(\pmb{K}^n)$.

If the projection set $\mathcal{S}$ is convex and closed,  the projection onto $\mathcal{S}$ is non-expansive, that is, \\
{$\sum_{t=0}^{T-1}\left\|\widetilde{z}_t^1-\widetilde{z}_t^2\right\|_F\leq \sum_{t=0}^{T-1}\left\|z_t^1-z_t^2\right\|_F$ with $\widetilde{\pmb{z}}^1=\Pi_{\mathcal{S}}(\pmb{z}^1)$ and $\widetilde{\pmb{z}}^2=\Pi_{\mathcal{S}}(\pmb{z}^2)$. Given any policy matrix $\pmb{K}$ and learning rate $\eta$, define the {\it gradient mapping for the projection operator}
\begin{equation}\label{eqn:defn_grad_map}
    G(\pmb{K}) := \frac{\Pi_{\mathcal{S}}{(\pmb{K}-\eta\nabla C(\pmb{K})})-\pmb{K}}{2\eta},
\end{equation}
 with $G(\pmb{K}) = (G_0(\pmb{K}),\cdots,G_{T-1}(\pmb{K}))$.
 Note that the gradient mapping has been commonly adopted in the analysis of projected gradient descent methods in constrained optimization \cite{nesterov2003introductory,zhang2019policy}.
    A policy matrix $\widetilde{\pmb{K}}\in \mathcal{S}$ is called a stationary point of $C(\cdot)$ if
\begin{eqnarray}\label{eqn:stationary_theta}
    \nabla C(\widetilde{\pmb{K}})^{\top}({\pmb{K}}-\widetilde{\pmb{K}})\leq 0,\quad\forall \pmb{K}\in\mathcal{S}.
\end{eqnarray}
It is  well-known in the  optimization literature that \eqref{eqn:stationary_theta} holds if and only if $G(\widetilde{\pmb{K}})=0$.
We have  the following sub-linear convergence result for the PPG version.
\begin{Theorem}\label{thm:projected_GD}
Assume Assumptions \ref{ass:parameters} and \ref{ass:State-2} hold, and the projection set of policies, denoted by $\mathcal{S}$, is convex and closed. Further assume $\pmb{K}^{*}\in \mathcal{S}$, $\pmb{K}^{0}\in \mathcal{S}$, $\sx>0$ and $C(\pmb{K}^0)$ is finite. At every step the policy is updated as in (4.4), that is
\[
\pmb{K}^{n+1} = \Pi_{\mathcal{S}}\left(\pmb{K}^n-\eta\widehat{\nabla C(\pmb{K}^n)}\right)
\]
with {$\eta\in\mathcal{H}(\frac{1}{C(\pmb{K}^0)+1})$} and $\widehat{\nabla_t C(\pmb{K}^n)}$ $(t=0,1,\cdots,T-1)$ is computed with hyper-parameters $(r,m)$ such that
{$r<1/\widehat{h}_{radius}$ and $m>\widehat{h}_{sample}$ with some fixed polynomials $\widehat{h}_{radius}\in\HECKI$ and $\widehat{h}_{sample}\in\HCKIR$}.
Then the projected policy gradient method has a global sublinear convergence rate, that is, $\left\{\frac{1}{N} \sum_{n=0}^{N-1}\left(\sum_{t=0}^{T-1} \|G_t({\pmb{K}^n})\|^2_F\right)\right\}_{N \ge 1}$ converges to 0 at rate $\mathcal{O}\left(\frac{1}{N}\right)$, where $G_t(\pmb{K})$ is defined in \eqref{eqn:defn_grad_map}.
\end{Theorem}}
 The proof of Theorem \ref{thm:projected_GD} can be found in Appendix \ref{appendix:missing_proofs_sec4}.
\begin{Remark}{\rm
We assume the projection step is perform accurately and the associated computational cost is of a separate interest and hence omitted here. The convergence result in Theorem \ref{thm:projected_GD} is described in terms of the sample complexity and to perform the projection step does not need extra samples.}
\end{Remark}

\subsection{Perturbation analysis of $C(\pmb{K})$ and $\nabla_t C(\pmb{K})$}\label{sec:model_free_perturbation}

{This section shows that the objective function $C(\pmb{K})$ and its gradient are stable with respect to small perturbations. The proofs of the following Lemmas can be found in Appendix \ref{appendix:missing_proofs_sec4}.}

\begin{Lemma}[$C(\pmb{K})$ Perturbation]\label{Model-free C_K perturbation}\label{lemma 27}
Assume Assumptions \ref{ass:parameters} and \ref{ass:State-2} hold,  {$\sx>0$}, and $\pmb{K}^{\prime}$  such that, $\forall\,t=0,1,\cdots,T-1$,
\begin{equation}\label{Lemma 27 assumption}
    \|K_t'-K_t\| \leq \min\left\{{\frac{(\rho^2-1)\sq\sx}{2T(\rho^{2T}-1)(2\rho+1)(C(\pmb{K})+\sq T\|W\|)\|B\|}}, \|K_t\|,\,\frac{1}{T}\right\},
\end{equation}
where $\rho$ is defined in \eqref{rho}. Then there exists a polynomial  $h_{cost}\in\HCK$ such that
\begin{equation*}
    |C(\pmb{K}^\prime)-C(\pmb{K})| \leq h_{cost}\vertiii{\pmb{K}^\prime-\pmb{K}}.
\end{equation*}


\end{Lemma}

\begin{Lemma}[$\nabla_t C(\pmb{K})$ Perturbation]\label{lemma 28}
{Under the same assumptions as in Lemma \ref{Model-free C_K perturbation},} there exists a polynomial {$h_{grad}\in\HCK$} such that
\begin{equation*}
    \|\nabla_t C(\pmb{K}^{\prime}) - \nabla_t C(\pmb{K})\| \leq h_{grad}\vertiii{\pmb{K}^\prime-\pmb{K}},\quad {\rm and } \quad
    \|\nabla_t C(\pmb{K}^{\prime}) - \nabla_t C(\pmb{K})\|_F \leq h_{grad}\vertiii{\pmb{K}^\prime-\pmb{K}}_F.
\end{equation*}
\end{Lemma}

\subsection{Smoothing and the Gradient Descent Analysis}\label{sec:smoothing}
In this section,  Lemma \ref{lemma 29} provides the formula for the perturbed gradient term, Lemma \ref{lemma:subexponential} provides the concentration inequality for finite samples, and Lemma \ref{lemma 30} provides the guarantees for the gradient approximation.

Recall that $D = k\times d$.
Let $\mathbb{S}_r$ represent the uniform distribution over the points with norm $r$ in  dimension $D$, and $\mathbb{B}_r$ represent the uniform distribution over all points with norm at most $r$ in dimension $D$. For each $K_t$ $(t=0,1,\cdots,T-1)$, the algorithm performs gradient descent on the following function:
\begin{eqnarray}\label{eqn:smooth_C_t}
C_t^r(\pmb{K}) = \mathbb{E}_{V_t\sim\mathbb{B}_r}\left[C(\pmb{K}+\pmb{V}_t)\right],
\end{eqnarray}
where $\pmb{V}_t:= (0,\cdots,V_t,\cdots,0)$ and ${V_t} \in \mathbb{R}^{k\times d}$.




\begin{Lemma}\label{lemma 29}
Assume $C(\pmb{K})$ is finite,
\begin{equation}\label{perturbation}
    \nabla_t C_t^r(\pmb{K})=\frac{D}{r^2}\mathbb{E}_{U_t\sim\mathbb{S}_r}[C(\pmb{K}+\pmb{U}_t)U_t].
\end{equation}
\end{Lemma}

The proof of Lemma \ref{lemma 29} is similar to the proof of \cite[Lemma 29]{FGKM2018} and hence omitted.



We first state two facts on sub-Gaussian and sub-exponential random variables.
Firstly, if $X$ and $Y$ are zero-mean independent random variables such that $X\in SG(\sigma_x^2)$ and $Y\in SG(\sigma_y^2)$, then $XY \in SE(\sigma_x\sigma_y,4\sigma_x\sigma_y)$.  Secondly, if  $X_1,\cdots,X_n$ are zero-mean independent random variables such that $X_i\in SE(\nu_i^2,\alpha_i)$, then
\[
\sum_{i=1}^n X_i\in SE \left(\sum_{i=1}^n\nu_i^2,\max_i\alpha_i\right).
\]
Using the above two facts, we have the following.

\begin{Lemma}\label{lemma:subexponential}
Assume Assumptions  \ref{ass:parameters} and \ref{ass:State-2} hold and {$\sx>0$},
then there exist {polynomials $\nu\in\HCK$ and $\alpha\in\HCK$}
such that
\[
\left[\sum_{t=0}^{T-1}\Big(x_t^{\top}Q_tx_t+u_t^{\top} R_t u_t\Big)+x_T^{\top} Q_Tx_T\right]
\]
is sub-exponential with parameter $\left(\nu^2,\alpha\right)$. Here $\{x_t\}_{t=0}^T$ is the dynamics under policy  $\pmb{K}$.
\end{Lemma}

\begin{proof}
We first observe that, by direct calculation,
\begin{eqnarray}
\left[\sum_{t=0}^{T-1}\Big(x_t^{\top}Q_t x_t+u_t^{\top} R_t u_t\Big)+x_T^{\top} Q_T x_T\right] = x_0^{\top} P_0 x_0 +\sum_{t=0}^{T-1} w_t^{\top} P_{t+1}w_t.\label{eq:sample_cost}
\end{eqnarray}

Note that by \eqref{Definition of P} and Proposition \ref{prop:P_positive}, $P_t$ is symmetric and positive definite. The Frobenius norm $\|\cdot\|_F$ and the spectral norm  $\|\cdot\|$ of the matrix $P_t\in \mathbb{R}^{d\times d}$ have the following property:
\begin{eqnarray}\label{norm_property_P_t}
\|P_t\| \leq \|P_t\|_F \leq \sqrt{d} \|P_t\|,\ \forall\, t=0,1,\cdots, T.
\end{eqnarray}
Let ${\widehat{\sigma}} = \max \{\sigma_0,\sigma_w\}$. Given the Hanson-Wright inequality (Theorem 2.5 in \cite{Radoslaw2015}),
\begin{eqnarray}\label{Hanson-Wright_inequality}
&&\mathbb{P} \left(\left| w_t^{\top}P_{t+1} w_t-\mathbb{E}\left[w_t^{\top} P_{t+1} w_t\right]\right|\geq t\right)=\mathbb{P} \left(\left| v_t^{\top} (\widetilde{W}^{\top}P_{t+1}\widetilde{W}) v_t-\mathbb{E}\left[v_t^{\top} (\widetilde{W}^\top P_{t+1} \widetilde{W}) v_t\right]\right|\geq t\right)\nonumber\\
&&\qquad \qquad \qquad \leq 2\exp\left(-c\min\left\{\frac{t^2}{2{\widehat{\sigma}}^4\|\widetilde{W}^{\top}P_{t+1}\widetilde{W}\|_F^2},\frac{t}{{\widehat{\sigma}}^2\|\widetilde{W}^{\top}P_{t+1}\widetilde{W}\|}\right\}\right),
\end{eqnarray}
for some universal constant $c>0$ which is independent of $P_{t+1}$ and $w_t$.

Combining \eqref{norm_property_P_t}, \eqref{Hanson-Wright_inequality} and Lemma \ref{lemma 13},
\begin{eqnarray*}
\mathbb{P} \left(\left| w_t^{\top} P_{t+1} w_t-\mathbb{E}\left[w_t^{\top} P_{t+1} w_t\right]\right|\geq t\right) &\leq& 2\exp\left(-c\min\left\{\frac{t^2}{2{ \widehat{\sigma}}^4\,d\,\|P_{t+1}\|^2\|\widetilde{W}\|^4},\frac{t}{{ \widehat{\sigma}}^2\|P_{t+1}\|\|\widetilde{W}\|^2}\right\}\right)\\
&\leq& 2\exp\left(-c\min\left\{\frac{t^2}{2{ \widehat{\sigma}}^4\|\widetilde{W}\|^4\,d\,C^2(\pmb{K})/\sx^2},\frac{t}{{ \widehat{\sigma}^2}\|\widetilde{W}\|^2C(\pmb{K})/\sx}\right\}\right).
\end{eqnarray*}
Therefore the random variable $w_t^{\top} P_{t+1} w_t$ is sub-exponential with parameters$
\left(\frac{ { \widehat{\sigma}}^4\|\widetilde{W}\|^4d C^2(\pmb{K})}{c\sx^2},\frac{{ \widehat{\sigma}}^2\|\widetilde{W}\|^2C(\pmb{K})}{2{ c}\sx}\right).$ In the same way  $x_0^{\top} P_{0} x_0$ is sub-exponential with parameters
$\left(\frac{{ \widehat{\sigma}}^4\|\widetilde{W}_0\|^4d C^2(\pmb{K})}{c\sx^2},\frac{{ \widehat{\sigma}}^2\|\widetilde{W}_0\|^2C(\pmb{K})}{2{ c}\sx}\right)$.
Let $\overline{\sigma} =$ \\ $\max\{\|\widetilde{W}_0\|,\|\widetilde{W}\|\}$. Since $\{w_t\}_{t=0}^{T-1}$ are IID and independent from $x_0$, we have 
\eqref{eq:sample_cost} is sub-exponential with parameters
\[
\left((T+1)\,\frac{ { \widehat{\sigma}}^4 \overline{\sigma}^4d C^2(\pmb{K})}{c\sx^2},\frac{{ \widehat{\sigma}}^2\overline{\sigma}^2C(\pmb{K})}{2{ c}\sx}\right).
\]
\end{proof}

{
\noindent Define
$$
     \widetilde{\nabla}_t := \frac{1}{m}\sum_{i=1}^m \left(\frac{D}{r^2} C(\pmb{K}+\pmb{U}_t^i)U_t^i\right)
$$
as the average of perturbed {cost functions} across $m$ scenarios which is an empirical approximation of \eqref{perturbation}. Similarly, define
\begin{equation}\label{eq:hat_delta}
     \widehat{\nabla}_t := \frac{1}{m}\sum_{i=1}^m \left(\frac{D}{r^2} \left[\sum_{t=0}^{T-1}\Big((x^i_t)^{\top}Q_tx^i_t+(u_t^i)^{\top}R_tu^i_t\Big)+(x_T^i)^{\top}Q_Tx^i_T\right]U_t^i\right)
\end{equation}
as the average of perturbed and single-trajectory-based {cost functions} across $m$ scenarios, which is the same as \eqref{eqn:biased_estimate_gradient} in Algorithm \ref{alg:MFP}.
Note that in order to calculate $\widetilde{\nabla}_t$, we require access to $C(\pmb{K}+\pmb{U}_t^i)$, which involves the calculation of expectations with respect to unknown initial states and state noises. This may be restrictive in some settings. On the other hand, the calculation of $\widehat{\nabla}_t$ only involves single-trajectory-based {cost functions}. 
}

\begin{Lemma}\label{lemma 30}
Assume Assumptions \ref{ass:parameters} and \ref{ass:State-2} hold, and {$\sx>0$}. Given any $\epsilon$, there are fixed polynomials $h_{radius}\in\HECK$ and $h_{sample}\in\HCKR$ such that when $r\leq 1/h_{radius}$, with $m\geq h_{sample}$ samples of $U_t^1,\cdots,U_t^m\sim\mathbb{S}_r$ for each $t=0,\cdots,T-1$, 
\[
\left\|{\widetilde{\nabla}_t}-\nabla_t C(\pmb{K})\right\|_F \leq \epsilon,
\]
holds with high probability (at least $1-\left(\frac{D}{\epsilon}\right)^{-D}$).
In addition, there is a polynomial {$h_{sample,2}\in\HCKR$
such that when $r\leq 1/h_{radius}$, with $m\geq h_{sample}\,+\,h_{sample,2}$} samples of $U_t^1,...,U_t^m\sim\mathbb{S}_r$ for each $t=0,\cdots,T-1$, 
\[
\left\|{\widehat{\nabla}_t}-\nabla_t C(\pmb{K})\right\|_F\leq \frac{3}{2}\epsilon,
\] 
holds with high probability (at least $1-{2}\left(\frac{D}{\epsilon}\right)^{-D}$).
Here, for each $i=1,2,\cdots,m$, $\{x^i_t\}_{t=0}^T$ and $\{u^i_t\}_{t=0}^{T-1}$ are the dynamics and controls for a single path sampled using policy $\pmb{K}+\pmb{U}_t^i$.

\end{Lemma}


\begin{proof}
Note that
\begin{equation*}
    {\widetilde{\nabla}_t}-\nabla_t C(\pmb{K})=  (\nabla_t C_t^r(\pmb{K})- \nabla_t C(\pmb{K}) ) + ({\widetilde{\nabla}_t} - \nabla_t C_t^r(\pmb{K})),
\end{equation*}
where $C_t^r$ is defined in \eqref{eqn:smooth_C_t}.

For the first term, choose $h_{radius}=\max\{1/{r_0},4h_{grad}/\epsilon\}$ {($r_0$ is chosen later)}, where $h_{grad}\in\HCK$ is defined in Lemma \ref{lemma 28}. By Lemma \ref{lemma 28} when $r\leq 1/{h_{radius}}\leq\epsilon/{4h_{grad}}$, {for $\pmb{V}_t:= (0,\cdots,V_t,\cdots,0)$ where ${V_t}\sim\mathbb{B}_r$}, we have
\begin{equation}\label{eqn:grad_C_U}
    \|\nabla_t C(\pmb{K}+\pmb{V}_t) - \nabla_t C(\pmb{K})\|_F \leq h_{grad}\|\pmb{V}_t\|_F \leq h_{grad}\frac{\epsilon}{4h_{grad}}=\frac{\epsilon}{4}.
\end{equation}

Since $\nabla_t C_t^r(\pmb{K}) = \mathbb{E}_{V_t \sim \mathbb{B}_r}[\nabla_t C(\pmb{K}+\pmb{V}_t)]$, we have
\[
 \|\nabla_t C(\pmb{K}+\pmb{V}_t) - \nabla_t C_t^r(\pmb{K})\|_F \leq \frac{\epsilon}{4},
\]
by \eqref{eqn:grad_C_U} and the continuity of $\nabla_t C$.
Therefore
\begin{equation}\label{first_nabla_diff}
    \|\nabla_t C_t^r(\pmb{K}) - \nabla_t C(\pmb{K})\|_F \leq  \|\nabla_t C(\pmb{K}+\pmb{V}_t) - \nabla_t C(\pmb{K})\|_F+  \|\nabla_t C(\pmb{K}+\pmb{V}_t) - \nabla_t C_t^r(\pmb{K})\|_F \leq \frac{\epsilon}{2}
\end{equation}
holds by triangle inequality. We choose $r_0$ such that for any $\pmb{U}_t\sim\mathbb{S}_r$, we have that $C(\pmb{K}+\pmb{U}_t)\leq 2C(\pmb{K})$. By Lemma \ref{lemma 27}, {we can pick $1/r_0=h_{cost}/C(\pmb{K})$, then $|C(\pmb{K}+\pmb{U}_t)-C(\pmb{K})|\leq r_0\cdot h_{cost}\leq C(\pmb{K})$.}


For the second term, by Lemma \ref{lemma 29}, $\mathbb{E}[{\widetilde{\nabla}_t}]=\nabla_tC_t^r(\pmb{K})$, and each individual sample is bounded by $2DC(\pmb{K})/r$, so by the Operator-Bernstein inequality \cite[Theorem 12]{gross2011recovering} with 
\[
m\geq h_{sample}=\Theta\left(D\left(\frac{D\cdot C(\pmb{K})}{r\epsilon}\right)^2\log(D/\epsilon)\right),
\]
we have
\begin{equation}\label{bern1}
    \mathbb{P}\left[\left\|{\widetilde{\nabla}_t} - \nabla_t C_t^r(\pmb{K})\right\|_F\leq\frac{\epsilon}{2}\right]\geq 1-\left(\frac{D}{\epsilon}\right)^{-D}.
\end{equation}
{Note that $h_{sample}\in\HCKR$ since $1/r>h_{radius}\in\mathcal{H}(1/\epsilon,C(\pmb{K}))$.}
Adding these two terms together and applying the triangle inequality gives the result. 

For the second part, note that
\begin{eqnarray}
\mathbb{E}_{x_0,\pmb{w}}[{\widehat{\nabla}_t}] = {\widetilde{\nabla}_t}.
\end{eqnarray}
By Lemma \ref{lemma:subexponential}, 
\[
\left[\sum_{t=0}^{T-1}\Big((x^i_t)^{\top}Q_tx^i_t+(u_t^i)^{\top}R_tu^i_t\Big)+(x_T^i)^{\top}Q_Tx^i_T\right]
\]
is sub-exponential with parameters $(\nu^2,\alpha)$. Therefore,
\[
Z_i:=\left(\frac{D}{r^2} \left[\sum_{t=0}^{T-1}\Big((x^i_t)^{\top}Q_tx^i_t+(u_t^i)^{\top}R_tu^i_t\Big)+(x_T^i)^{\top}Q_Tx^i_T\right]U_t^i\right)\]
is sub-exponential matrix with parameters $(\widetilde{\nu}^2,\widetilde{\alpha}) := \left(\frac{D}{r^2}\nu^2,\alpha\right)$. Then by Operator-Berinstein inequality \cite[Theorem 12]{gross2011recovering},
\[
\mathbb{P}\left[\left\|\frac{1}{m}\sum_{i=1}^m Z_i-\mathbb{E}[Z_1]\right\|_F \leq t\right] \geq 1- 2 D\exp \left( - m\frac{t^2}{2\widetilde{\nu}^2 }\right),
\]
when $t \leq \frac{\widetilde{\nu}^2}{\widetilde{\alpha}}$. That is, there exists a polynomial {$h_{sample,2}\in\HCKR$} where
\[
h_{sample,2}:=h_{sample,2} \left(D,\frac{1}{\epsilon},\frac{1}{r},\sigma_0,\sigma_w,\|\widetilde{W}_0\|,\|\widetilde{W}\|,C(\pmb{K}),\frac{1}{\sx}\right) =\Theta\left(D \left( \frac{\widetilde{\nu}}{\epsilon}\right)^2\log(D/\epsilon)\right),
\]
such that when
$m\geq h_{sample,2}$,
\begin{equation}\label{bern2}
    \mathbb{P}\left[\left\|\widehat{\nabla}_t - \widetilde{\nabla}_t\right\|_F\leq\frac{\epsilon}{2}\right]\geq 1-\left(\frac{D}{\epsilon}\right)^{-D}.
\end{equation}
Combining \eqref{bern2} with \eqref{first_nabla_diff} and \eqref{bern1}, we arrive at the desired result.
\end{proof}

\subsection{Proof of Theorem \ref{thm:model_free}} \label{sec:thm2}
With the results in Section \ref{sec:model_free_perturbation} and Section \ref{sec:smoothing}, now we are ready to prove the main theorem.
\begin{proof}[Proof of Theorem \ref{thm:model_free}]
{By Lemma \ref{lemma 24} and by choosing $\eta\in\mathcal{H}(\frac{1}{C(\pmb{K}^0)+1})$ such that the step size condition \eqref{stepsize condition} is satisfied,
}
\begin{equation*}
    C(\pmb{K}^{\prime})-C(\pmb{K}^*)\leq \Big(1-2\eta\sr\frac{\sx^2}{\|\Sigma_{\pmb{K}^*}\|}\Big)\Big(C(\pmb{K})-C(\pmb{K}^*)\Big).
\end{equation*}
Recall the definition of {$\widehat{\nabla}_t$} in \eqref{eq:hat_delta} and let $K_t^{\dprime}=K_t-\eta {\widehat{\nabla}_t}$ be the iterate that uses the approximate gradient. We will show later that given enough samples, the gradient can be estimated with enough accuracy that makes sure
\begin{equation}\label{eqn:small_cost}
     |C(\pmb{K}^{\dprime})-C(\pmb{K}^{\prime})|\leq \eta\sr\frac{\sx^2}{\|\Sigma_{\pmb{K}^*}\|}\epsilon.
\end{equation}
That means as long as $C(\pmb{K})-C(\pmb{K}^*)\geq \epsilon$, we have
\begin{equation*}
    C(\pmb{K}^{\dprime})-C(\pmb{K}^*)\leq \Big(1-\eta\sr\frac{\sx^2}{\|\Sigma_{\pmb{K}^*}\|}\Big)\Big(C(\pmb{K})-C(\pmb{K}^*)\Big).
\end{equation*}
Then the same proof as that of Theorem \ref{thm:convergence_egd} gives the convergence guarantee.

Now let us prove \eqref{eqn:small_cost}. First note that $C(\pmb{K}^{\dprime})-C(\pmb{K}^{\prime})$ is bounded. By Lemma \ref{lemma 27}, if $\|K_t^{\dprime}-K_t^{\prime}\|\leq \eta\sr\frac{\sx^2}{\|\Sigma_{\pmb{K}^*}\|}\cdot\epsilon/(T\cdot h_{cost})$, where $h_{cost}\in\HCK$ is the polynomial in Lemma \ref{lemma 27}, then \eqref{eqn:small_cost} holds.
To get this bound, {recall $K_t^\prime = K_t-\eta\nabla_t C(\pmb{K})$ in \eqref{eq:gd} and writing $\nabla_t=\nabla_t C(\pmb{K})$ for ease of exposition,} observe that $K_t^{\dprime}-K_t^{\prime}=\eta(\nabla_t-{\widehat{\nabla}_t)}$, therefore it suffices to make sure 
\begin{equation*}
    \|\nabla_t-{\widehat{\nabla}_t}\|\leq \frac{\sx^2\sr}{T\|\Sigma_{\pmb{K}^*}\|h_{cost}}\epsilon.
\end{equation*}
By Lemma \ref{lemma 30}, it is enough to pick $\overline{h}_{ radius}=h_{radius}({3}T\|\Sigma_{\pmb{K}^*}\|h_{cost}(C(\pmb{K}))/({2}\sx^2\sr\epsilon),C(\pmb{K}))$\\
$\in\HECK$, and 
\begin{eqnarray*}
\overline{h}_{sample}&=&h_{sample}\left( \frac{{ 3}h_{cost}(C(\pmb{K}))\|\Sigma_{\pmb{K}^*}\|}{{2}\sx^2\sr\epsilon},C(\pmb{K})\right) +h_{sample,2} \left(\frac{{ 3}h_{cost}(C(\pmb{K}))\|\Sigma_{\pmb{K}^*}\|}{{ 2}\sx^2\sr\epsilon}, C(\pmb{K})\right).
\end{eqnarray*}
This gives the desired upper bound on $\|\nabla_t-{\widehat{\nabla}_t}\|$ with high probability (at least {$1-2(\epsilon/D)^{D}$}).

Since the number of steps is a polynomial, we have $TN = o(\epsilon^D)$. By the union bound with probability at least 
{
\[
\left(1-2\left(\frac{\epsilon}{D}\right)^{D}\right)^{TN}\geq 1-2\,TN \left(\frac{\epsilon}{D}\right)^{D}\geq 1-\exp(-D),
\]}
{we have $\|\nabla_t-\widehat{\nabla}_t\|\leq { \frac{\sx^2\sr}{T\|\Sigma_{\pmb{K}^*}\|h_{cost}}\epsilon}$},  $\forall\,t=0,1,\cdots,T-1$. Therefore,
\begin{equation}\label{eq:mf_one_step_prog}
    C(\pmb{K}^{\dprime})-C(\pmb{K}^*)\leq \Big(1-\eta\sr\frac{\sx^2}{\|\Sigma_{\pmb{K}^*}\|}\Big)\Big(C(\pmb{K})-C(\pmb{K}^*)\Big).
\end{equation}
{This implies $C(\pmb{K}^{\dprime})<C(\pmb{K})$. To guarantee that \eqref{eq:mf_one_step_prog} holds at each iteration $n=1,2,\cdots,N$,  it suffices to pick $\overline{h}_{radius}\in\HECKI$ and $\overline{h}_{sample}\in\HECKI$. The rest of the proof is the same as that of Theorem \ref{thm:convergence_egd}.} Note again that in the smoothing, because the function value is monotonically decreasing, and by the choice of radius, all the function values encountered are bounded by $2C(\pmb{K}^0)$, so the polynomials are indeed bounded throughout the algorithm.
\end{proof}

\subsection{Discussion}
\begin{Remark}[Comparison with \cite{FGKM2018}]\label{rmk:comparison} {\rm The proofs of our main results, Theorems \ref{thm:convergence_egd} and \ref{thm:model_free}, are different to those from \cite{FGKM2018}.
Firstly, to prove the gradient dominant condition, \cite{FGKM2018} only required conditions on the distribution of the initial position. However, we need conditions to guarantee the non-degeneracy of the state covariance matrix at any time. Secondly, the extra randomness from the sub-Gaussian noise needs to be taken care of in the perturbation analysis of $\Sigma_{\pmb{K}}$. Finally, we need more advanced concentration inequalities to provide the number of samples and number of simulation trajectories that leads to the theoretical guarantee in the case with unknown parameters.}
\end{Remark}

\begin{Remark}[Non-stationary Dynamics]{\rm
Note that our framework can be generalized to non-stationary dynamics, that is, for $t=0,1,\cdots,T-1$,
\begin{equation}\label{nonstationary_dynamics}
    x_{t+1} = A_tx_t+B_tu_t+w_t,\ x_0\sim\mathcal{D}.
\end{equation}
with $\{A_t\}_{t=0}^{T-1}$ and $\{B_t\}_{t=0}^{T-1}$ time-dependent state parameters.}
\end{Remark}

\begin{Remark}[Other Policy Gradient Methods]{\rm
Our convergence and sample complexity analysis could be applied to other policy gradient methods, including the Natural policy gradient method and the Gauss-Newton method, in the framework of the LQR with stochastic dynamics and finite horizons.}
\end{Remark}

\section{Numerical Experiments}\label{sec:experiment}
\label{sc:experiment}
The performance of the PPG algorithm \eqref{eq:projected_updateing_rule} is demonstrated for the optimal liquidation problem with single asset and  the empirical analysis of the policy gradient method \eqref{eqn:model_free_policy_update} in higher dimensions is also provided with synthetic data. We will specifically focus on the following questions.
\begin{itemize}
    \item In practice, how fast do the policy gradient algorithm and the PPG algorithm with known and unknown parameters converge to the true solution?
    \item How does the deadline (the finite horizon) influence the optimal policy?
    \item When the real-word system does not exactly follow the LQR framework, does the policy-gradient method outperform mis-specified LQR models?
\end{itemize}
This section is organized as follows.  We demonstrate the performance of the PPG algorithms for the optimal liquidation problem with a single asset in the LQR framework in Section~\ref{sec:singleassetliq}. We then show that without the LQR model specification, the learned policy from the policy gradient algorithm improves the Almgren-Chriss solution in Section~\ref{sec:learning_to_liquidate}. Finally, we test the performance of the algorithm with unknown parameters in high dimensions in Section~\ref{sec:synthetic}.

Note the policy gradient method outperforms the Q-learning algorithm, a popular model-free method, in terms of both sample complexity and accuracy in our setting. An illustration in a one-dimensional example can be found in the Appendix \ref{appendix:q_learning}.

\subsection{Optimal Liquidation within the LQR Framework}\label{sec:singleassetliq}
Recall the set up of the optimal liquidation problem in \eqref{sec:optimal_liquidation_formulation}. By convention, we write the control in the feedback form as $u_t = -K_t x_t$. Writing $K_t = (k_t^1,k_t^2)$, we have $u_t = -k_t^1 S_t - k_t^2 q_t$, the state equation becomes
\[
x_{t+1} = 
\begin{pmatrix}
 1+\gamma k_t^1 &\gamma k_t^2\\ 
 k_t^1& 1 + k_t^2 
\end{pmatrix} x_t +w_t.
\]
In the liquidation problem, we assume $u_t \ge 0$ $( 0 \leq t \leq T-1)$. That is, $k_t^1 \leq 0$ and $k_t^2 \leq 0$ $( 0 \leq t \leq T-1)$.

\begin{Assumption}[Assumptions for the Optimal Liquidation Problems]\label{ass:liquidation} {\rm We assume
\begin{enumerate} 
    \item[(1)]  $\gamma k_t^1 +k_t^2>-1$ $(0 \leq t \leq T-1)$;
    \item[(2)] $\beta >\frac{\gamma}{2}$.
\end{enumerate}}
\end{Assumption}

\paragraph{Justification of the Assumption.} Assumption \ref{ass:liquidation}-(1)  is essential to ensure that the liquidation problem is well defined. First, $\gamma k_t^1>-1$ makes sure that the stock price process $\{S_t\}_{t=0}^T$ is well-behaved:
\[
\mathbb{E}[S_{t+1}] = \mathbb{E}[S_{t}] -\gamma \mathbb{E}[u_{t}] = 
(1+\gamma k_t^1) \mathbb{E}[S_{t}] + \gamma k_t^2 q_t.
\]
If $\gamma k_t^1 < -1$, then $\mathbb{E}[S_{t+1}]\le 0$ since $k_t^2 \leq 0$. Second, $k_t^2 \ge -1$ guarantees  that inventory will not be negative. Note that
\[
q_{t+1} = q_t - (-k_t^1 S_t -k_t^2 q_t) = (1+k_t^2)q_t + k_t^1 S_t.
\]
If $k_t^2\leq -1$ and $q_t >0$, then $q_{t+1}<0$.
Assumption \ref{ass:liquidation}-(2) implies that the temporary market impact is ``bigger'' than one half of the permanent market impact, which is consistent with the empirical evidence \cite{almgren2005direct} and assumptions in \cite{AC2001}.

\paragraph{Learning to Liquidate.}
In practice, traders may not know the market impact parameter $\gamma$. But one can always take some $\bar{\gamma}> \gamma$ based on some basic understandings of the market and perform a PPG algorithm to the closed convex set $\mathcal{S}$:
\begin{equation}\label{eq:big_projection_set}
\mathcal{S} := \left\{\pmb{K}=\left(K_0,\cdots,K_{T-1}\right):K_t = (k_t^1,k_t^2),\,\, \bar{\gamma}k_t^1 + k_t^2 \geq -1+\zeta,\,\, k_t^1 \leq 0,\,\, k_t^2 \leq 0,\,\,\forall t=0,\cdots,T-1\right\},
\end{equation}
with some small parameter $\zeta >0$.

In practice $\gamma$ is usually on the order of $10^{-5}\sim 10^{-6}$ (See Table~\ref{tb:single-asset-estimation} in Appendix \ref{sec:parameter_est}) and hence a universal upper bound ${\bar{\gamma}}$ in \eqref{eq:big_projection_set} is not a strong assumption for a given portfolio of stocks to liquidate.
 
\begin{Proposition}\label{prop:liquidation}Assume $\pmb{K} \in \mathcal{S}$ and Assumptions \ref{ass:parameters}, \ref{ass:State-2} and \ref{ass:liquidation} hold, we have $\sx>0$ and $ \{P_t^{\pmb{K}}\}_{t=0}^T $ derived from \eqref{Definition of P} are  positive definite for the optimal liquidation problem \eqref{eq:liquidaton_states} and \eqref{eq:mincost_LQR}.
\end{Proposition}
The proof of Proposition \ref{prop:liquidation} is deferred to Appendix \ref{appendix:missing_proofs_sec5}. It is easy to check that the projection set $\mathcal{S}$ defined in \eqref{eq:big_projection_set} is convex and closed. {Along with Proposition \ref{prop:liquidation}, the convergence result in Theorem \ref{thm:projected_GD} holds for the liquidation problem \eqref{eq:liquidaton_states} and
\eqref{eq:mincost_LQR} as long as the conditions in Proposition \ref{prop:liquidation} are satisfied.}

We test the performance of the PPG algorithm with projection set $\mathcal{S}$ on  Apple (AAPL) and Facebook (FB) stocks. The market simulator of the associated LQR framework is constructed with NASDAQ ITCH  data and the details can be found in Appendix \ref{sec:parameter_est}.

\paragraph{Performance Measure.} We use the following {\it normalized error} to quantify the performance of a given policy $\pmb{K}$,
\[
\text{Normalized error}=\frac{C(\pmb{K})-C(\pmb{K}^*)}{C(\pmb{K}^*)},
\]
where $\pmb{K}^*$ is the optimal policy defined in \eqref{optimal_k}.

\paragraph{Set-up.}  
        (1) Parameters: $\phi = 5\times 10^{-6}$ (for both AAPL and FB), $\epsilon = 10^{-8}$, $T=10$; smoothing parameter $r=0.6$, number of trajectories $m=200$; initial policy {$\pmb{K}^0 \in \mathbb{R}^{1\times 2T}$ with $\{\pmb{K}^0\}_{ij}=-0.2$ for all $i$, $j$,} for both algorithms with known and unknown parameters; step sizes are indicated in the figures; { $\bar{\gamma} = 5\times 10^{-5}$, $\zeta=10^{-12}$ for the projection set.}
        (2) Initialization:  Assume the initial inventory $q_0$ follows $\mathcal{N}(500,1)$. 
        The small variance of the initial inventory distribution is used to guarantee the initial state covariance matrix is positive definite. In practice, the algorithm converges with deterministic initial inventories.
        
\vspace{-0.15cm}
\hspace{-1.1cm}
\begin{minipage}{18cm}
\centering
\begin{figure}[H]
  \centering
  \begin{subfigure}[b]{0.4\textwidth}
    \includegraphics[width=\textwidth]{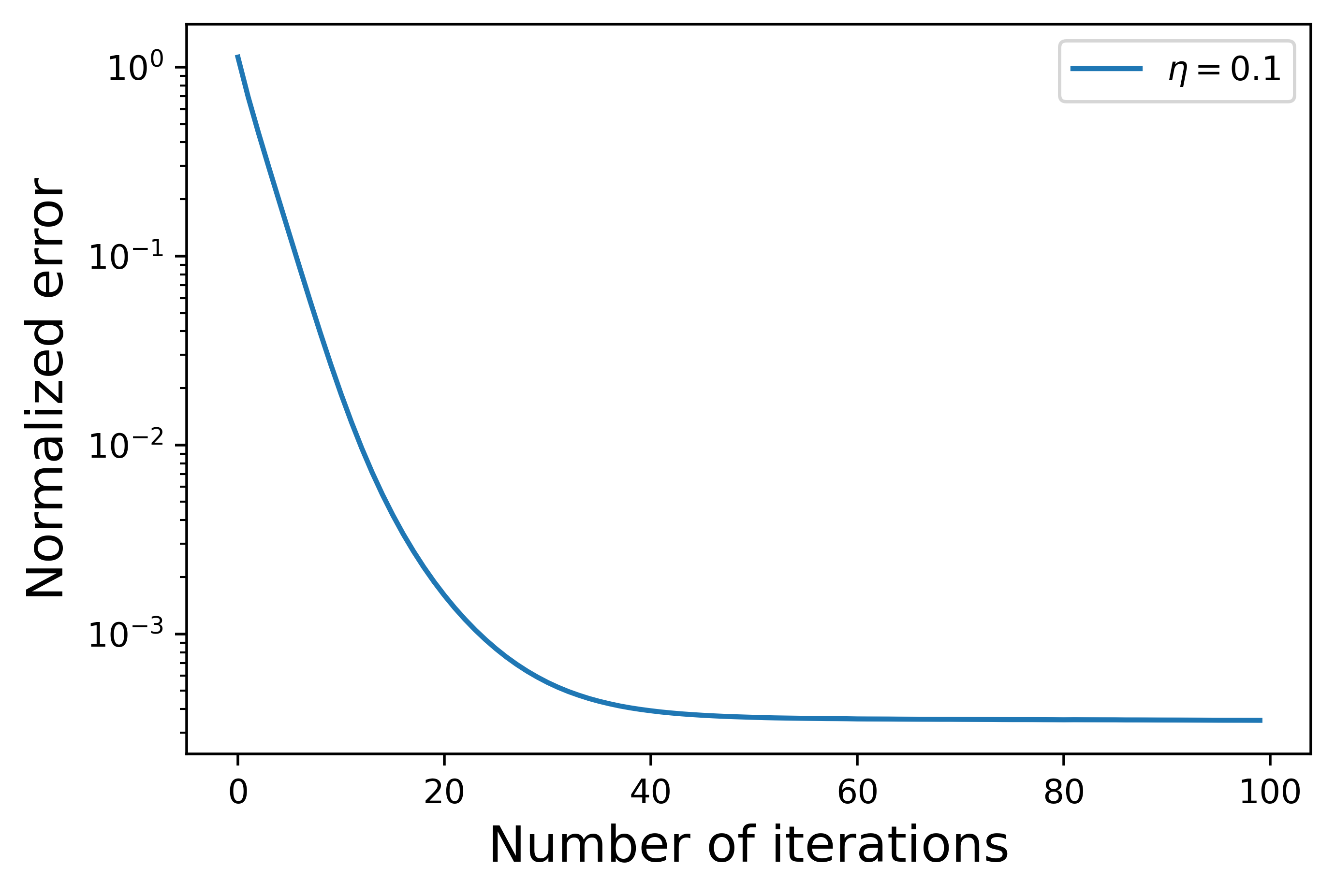}
    \caption{\label{fig:onestock_modelbased_performance}PPG with known parameters ($\eta=0.1$).}
  \end{subfigure}
  \begin{subfigure}[b]{0.4\textwidth}
    \includegraphics[width=\textwidth]{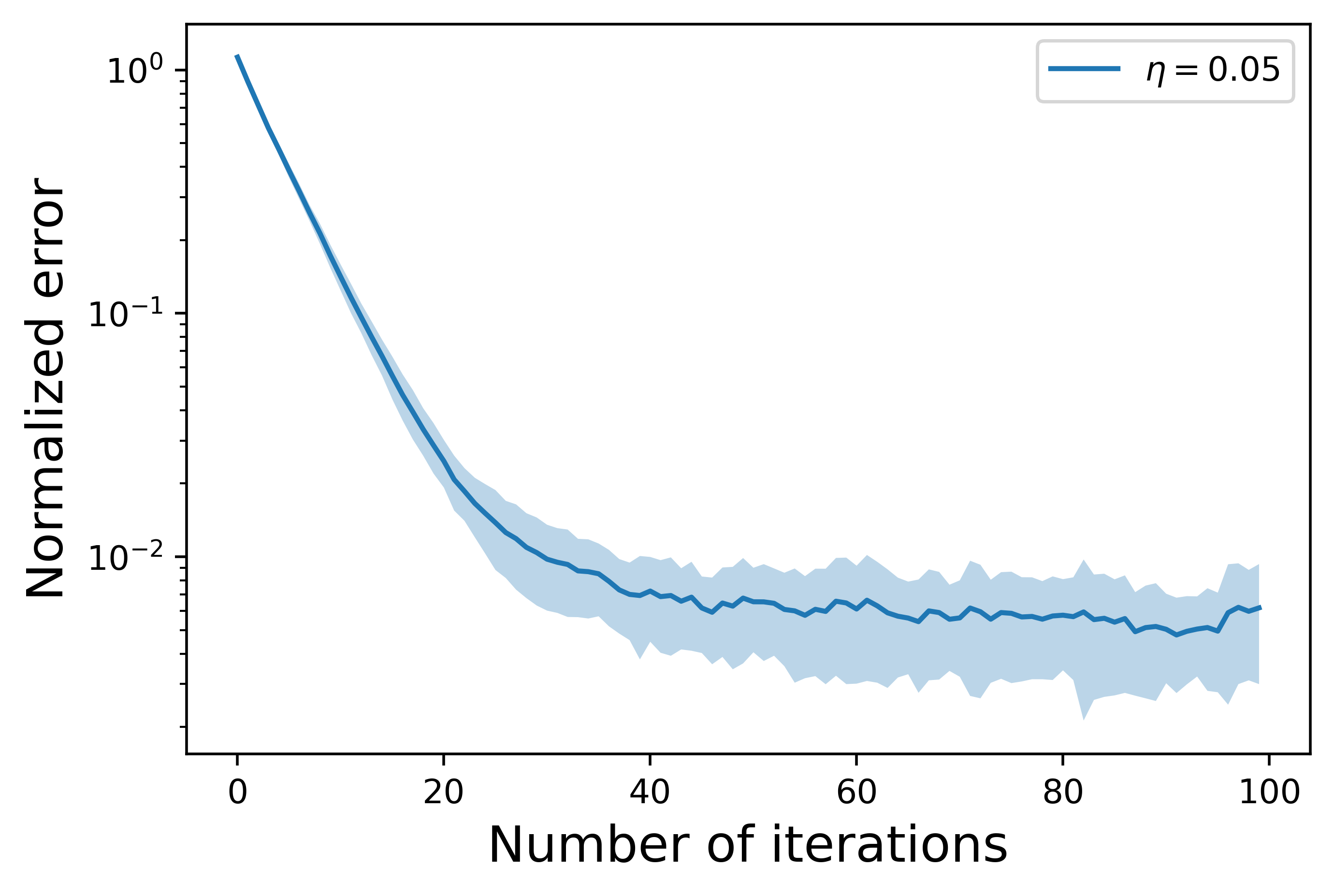}
    \caption{\label{fig:onestock_modelfree_performance}PPG  with unknown parameters ($\eta=0.05$).}
  \end{subfigure}
  \caption{\label{fig:onestock_performance}Performance of the PPG algorithms  (50 simulation scenarios).}
\end{figure}
\end{minipage}

\paragraph{Convergence.}
Both PPG algorithms with known parameters and unknown parameters show a reasonable level of accuracy within 50 iterations (that is the normalized error is less than $10^{-2}$). The PPG algorithm with known parameters has almost no fluctuations across the 50 scenarios. By  choosing $m=200$, the performance of the PPG algorithm with unknown parameters is stable with relatively small fluctuations (see the blue area in Figure \ref{fig:onestock_modelfree_performance}) across the 50 scenarios.

\paragraph{Impact of the Deadline.} The optimal policy is sensitive to the deadline in that the shapes of the optimal inventory trajectories are different with  different deadlines. See Figure \ref{fig:optimal_inventory} for both AAPL and FB with $T=30,60$ and $120$ minutes. The liquidation speed is almost linear when $T$ is small; and it is faster in the initial trading phase and slower at the end when $T$ is relatively large.
\paragraph{Impact of the Parameter $\phi$.}
Recall that in \eqref{eq:mincost_LQR} the parameter $\phi$ is used to balance the expected terminal wealth $\mathbb{E}[C]$ and the variance of the terminal wealth  $\text{var}[C]$. To show the impact of $\phi$, we set $\phi$ to be $10^{-4}$, $10^{-5}$, $10^{-6}$, and $10^{-7}$ and show the corresponding inventory trajectories in Figure \ref{fig:phi}. The optimal liquidation speed is almost linear when $\phi$ is small, while it is faster in the initial trading phase and slower at the end when $\phi$ is relatively large.

\begin{figure}[H]
  \centering
  \begin{subfigure}[b]{0.4\textwidth}
    \includegraphics[width=\textwidth]{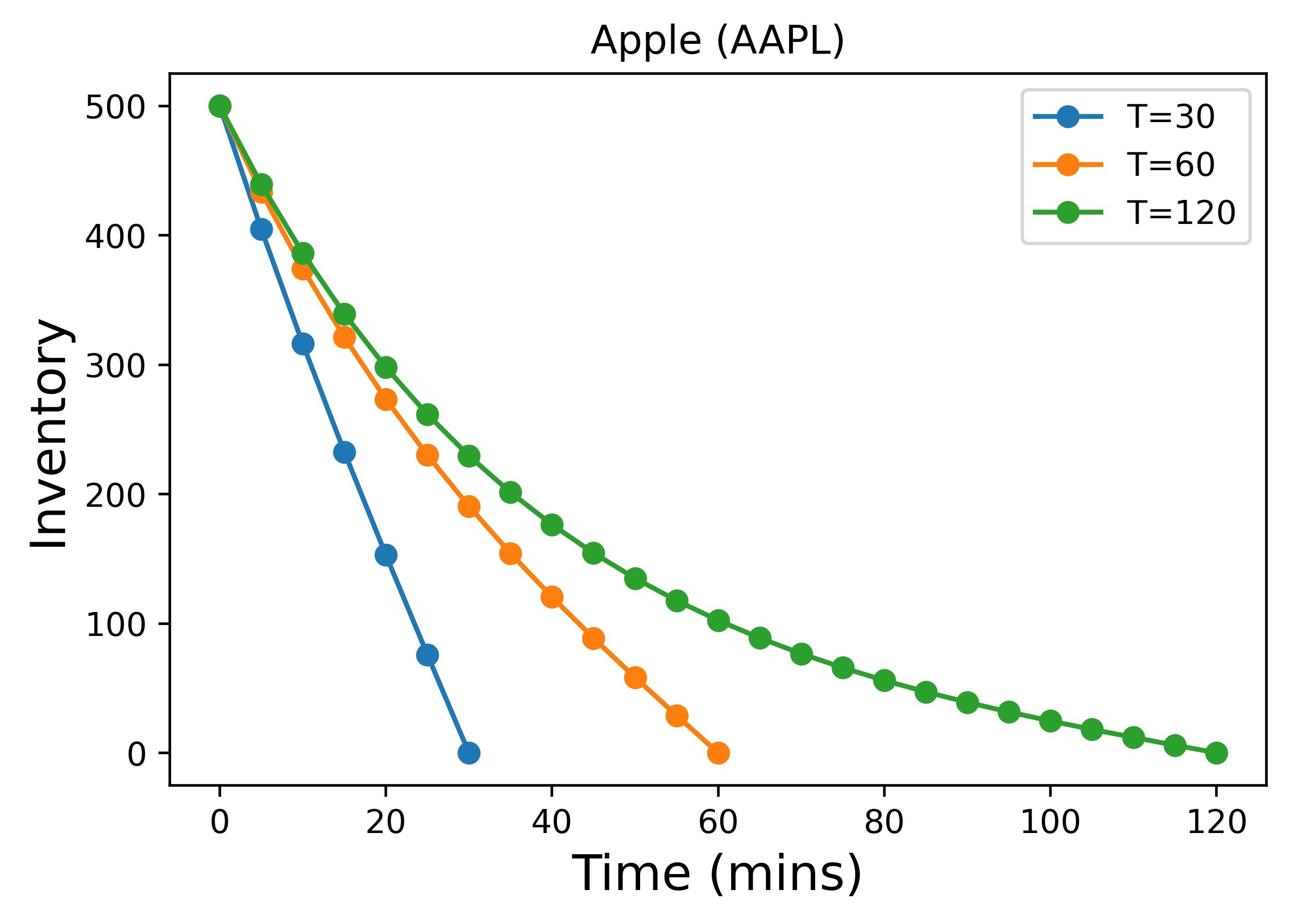}
    \caption{\label{fig:optimal_inventory_apple}AAPL.}
  \end{subfigure}
  \begin{subfigure}[b]{0.4\textwidth}
    \includegraphics[width=\textwidth]{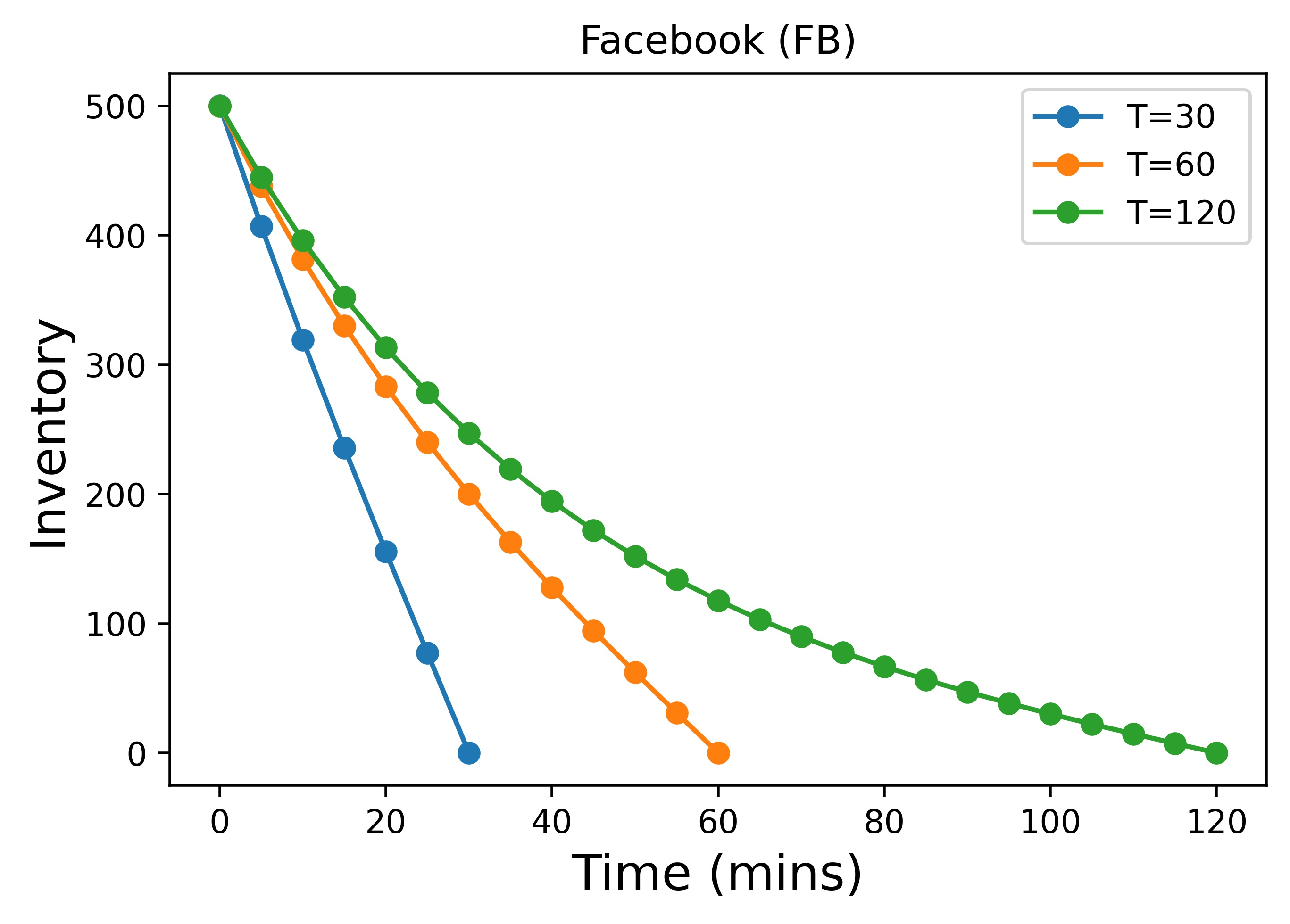}
    \caption{\label{fig:optimal_inventory_amazon}FB.}
  \end{subfigure}
  \caption{\label{fig:optimal_inventory}Optimal inventory trajectory under different deadlines (200 simulation scenarios).}
\end{figure}
\begin{figure}[H]
\begin{minipage}{5.2cm}
\centering
  \vspace{0.3cm}
\includegraphics[width=1\linewidth]{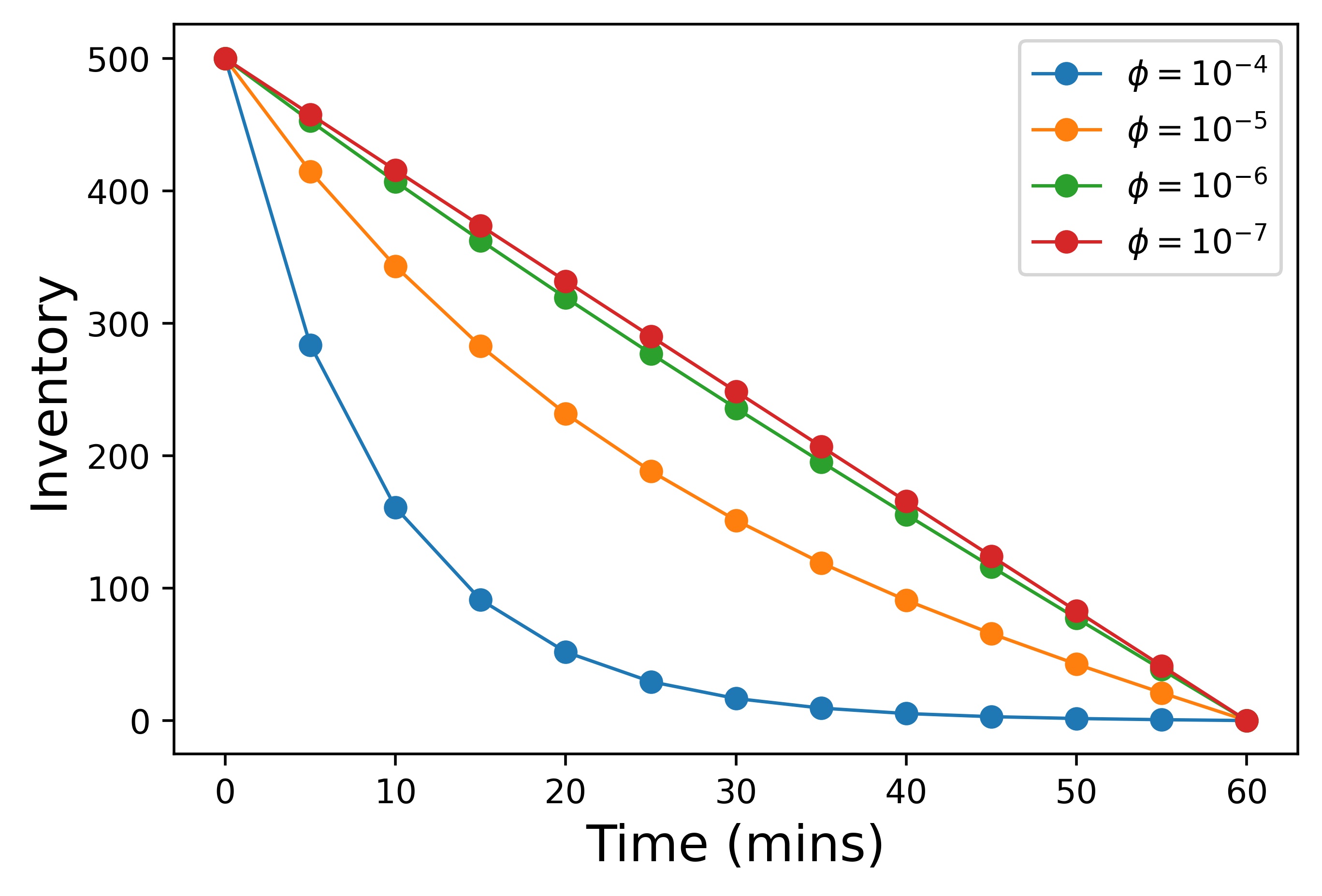}
\caption{\label{fig:phi}Inventory trajectories of AAPL under different $\phi$ (average across 200 simulation scenarios).}
\end{minipage}
\qquad
\begin{minipage}{11cm}
  \centering
  \begin{subfigure}[b]{0.48\textwidth}
    \includegraphics[width=\textwidth]{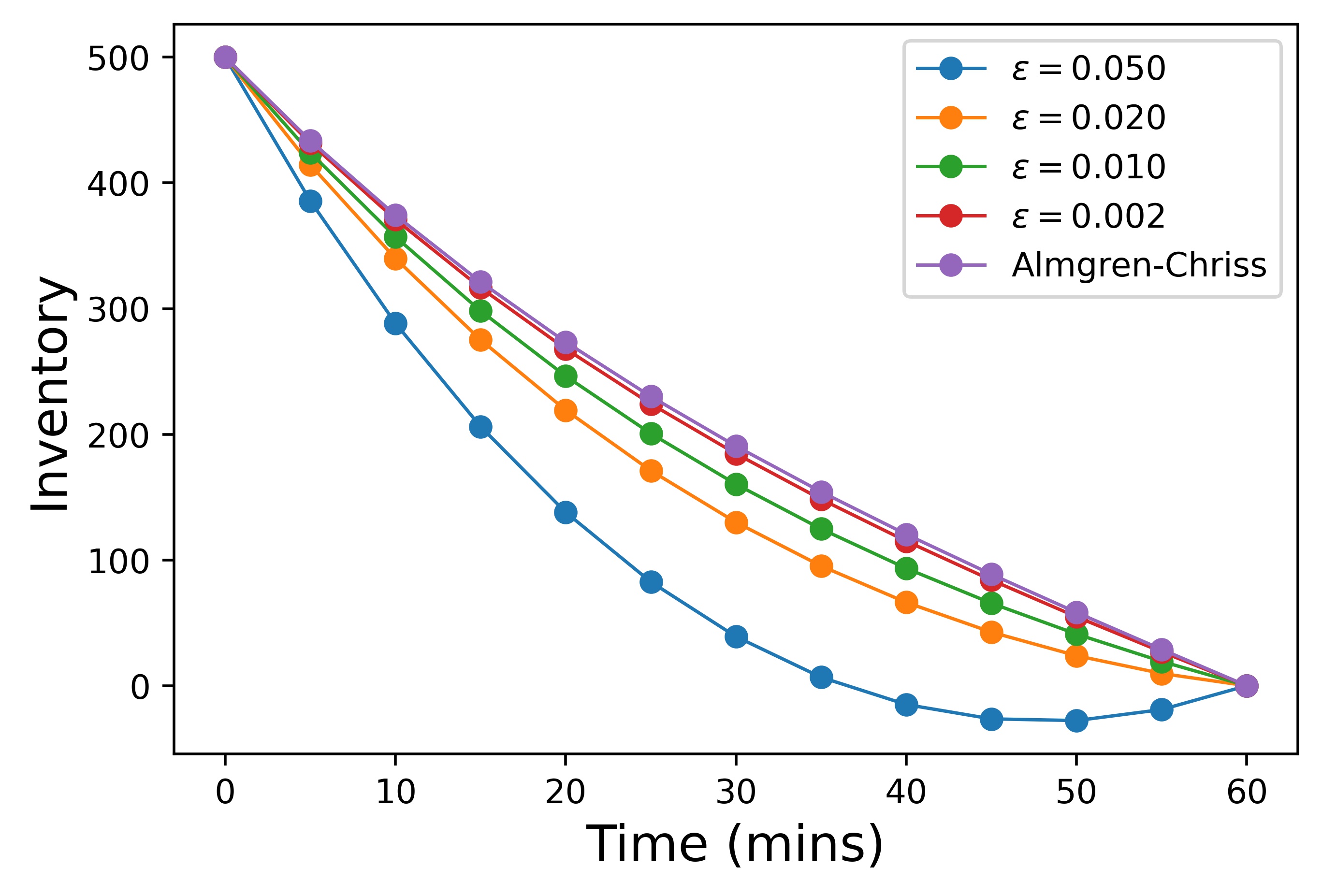}
    \caption{\label{fig:AlgChs_traj}Inventory trajectories.}
  \end{subfigure}
  \begin{subfigure}[b]{0.49\textwidth}
    \includegraphics[width=\textwidth]{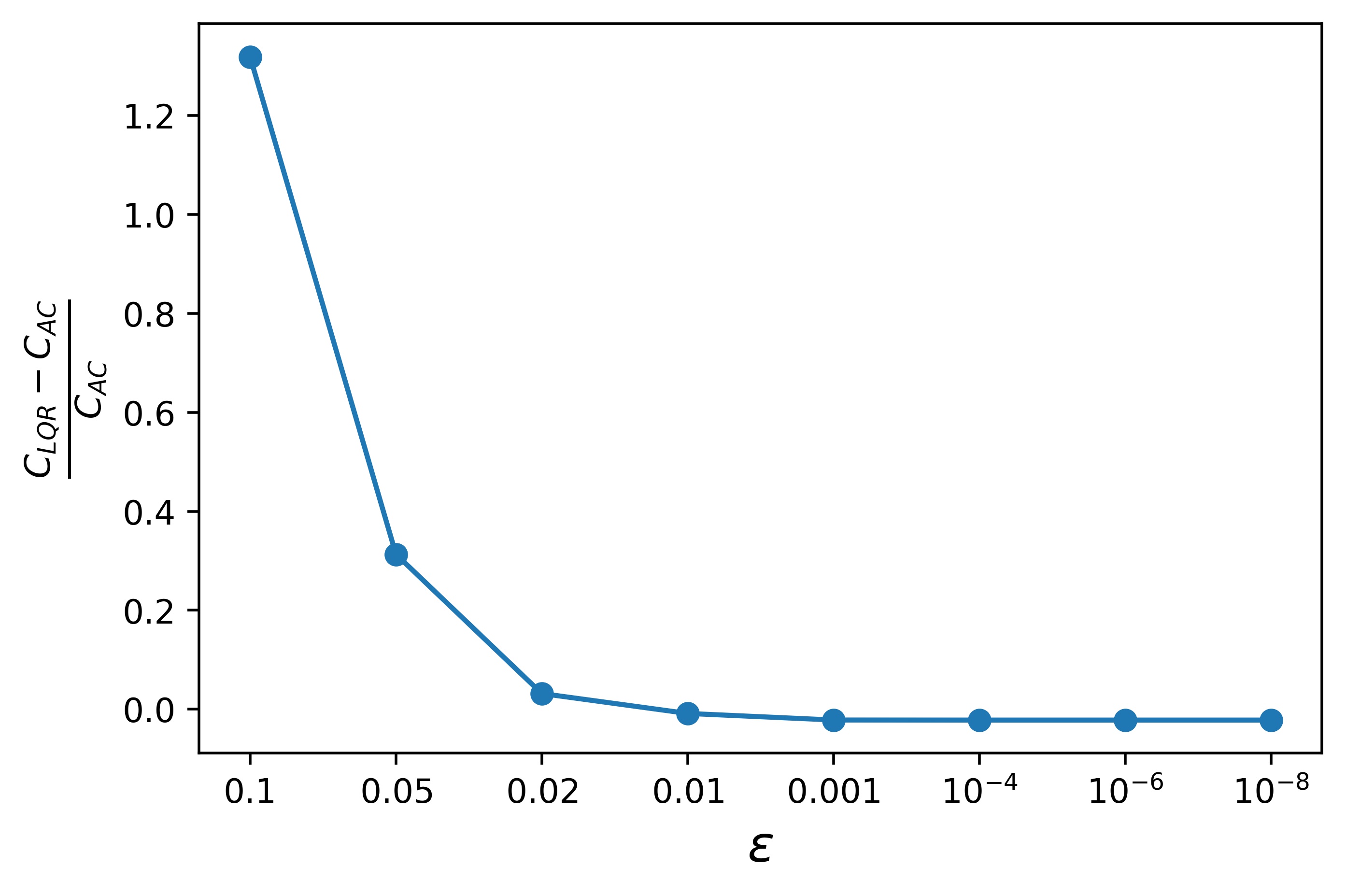}
    \caption{\label{fig:AlgChs_cost}Relative cost difference.}
  \end{subfigure}
  \caption{\label{fig:epsilon} Original Almgren-Chriss framework versus LQR formulation under different $\epsilon$ {(AAPL)}.}
\end{minipage}
\end{figure}

\paragraph{Impact of the Parameter $\epsilon$.}  Recall that our liquidation formulation \eqref{eq:mincost_LQR} differs from the Almgren-Chriss formulation \eqref{eq:mincost_AC} by an additional regularization term $\sum_{t=0}^T\epsilon S_t^2$. The role of this term is to enable the problem to be cast in the LQR framework and to guarantee the well-definedness of the Ricatti equation. From Figure~\ref{fig:AlgChs_traj},  the optimal policies and inventory trajectories are close to the Almgren-Chriss solution when $\epsilon \leq 0.01$. However, when $\epsilon=0.05$, the optimal policy is far away from the Almgren-Chriss solution. We show the difference between $C_{\rm AC}$, defined in \eqref{eq:mincost_AC}, and $C_{\rm LQR}(\epsilon)$, defined in \eqref{eq:mincost_LQR}, in Figure \ref{fig:AlgChs_cost}. We see that $C_{\rm LQR}(\epsilon)$ is close to $C_{\rm AC}$ when $\epsilon<0.02$ and is markedly different from $C_{\rm AC}$ when $\epsilon \geq 0.02$. It is worth noticing that when $\epsilon=0$, the algorithm does converge to the Almgren-Chriss solution in our setting although the convergence of the algorithm in this case is not guaranteed by our theoretical results. 


\subsection{Learning to Liquidate without Model Specification}\label{sec:learning_to_liquidate}
In practice, the dynamics of the trading system may not be exactly those assumed in the LQR framework but we might expect that the policy gradient method could still perform well when the system is ``nearly'' linear quadratic as the execution of the policy gradient method does not rely on the model specification. In this section, we consider liquidation problems in the Limit Order Book (LOB) setting. A LOB is a list of orders that a trading venue, for example the NASDAQ exchange, uses to record the interest of buyers and sellers in a particular financial instrument. There are two types of orders the buyers (sellers) can submit: a limit buy (sell) order with a preferred price for a given volume or a market buy (sell) order with a given volume which will be immediately executed with the best available limit sell (buy) orders. Here we perform the policy gradient method to learn the optimal strategies to liquidate using market orders in the LOB.

 
We denote by $S_t$ the mid-price of the asset at time $t$, that is the average of the best-bid price and best-ask price. At each time $t$, the decision is to liquidate an amount  $u_t$ of the asset. The action $u_t$ will have an impact on the market, with possibly both temporary and permanent impacts. Unlike the LQR framework or the classical Almgren-Chriss model, where dynamics are assumed to follow some stochastic model, here we run the policy gradient method directly on the LOB without any assumption on how the mid-price $S_t$ moves and what are the forms of the market impacts. 
Denote by $q_{t} = q_{t-1} - u_{t-1}$ the inventory at time $t$.  We restrict the admissible controls to be of the linear feedback form $u_t = - K_t (S_t,q_t)^{\top}$ with some $K_t \in \mathbb{R}^{1 \times 2}$.

The cost $c_t = \phi^{\prime}  (q_t-u_t)^2  -r_t(u_t)$ at time $t$ consists of two parts.  The first part $ \phi^{\prime}  (q_t-u_t)^2$ is the holding cost of the inventory weighted by a parameter $\phi^{\prime}$.   The quantity $r_t(u_t)$ is the amount we receive by liquidating $u_t$ shares at time $t$. Note that $r_t(\cdot)$ may depend on $S_t$ and other market observables. For example, if we liquidate $u_t = 1000$ shares of the asset with the market conditions given in Table \ref{tab:LOB_configuration}, then the amount received would be 
\[
 r_t(u_t) =  397\times 200.1 +  412 \times 200.0 + (1000-397-412)\times 199.9 = 200020.6.
\]
This transaction moves the best bid price two levels down. This is commonly referred to as the {\it temporary impact} of a market order.
 \begin{table}[H] 
 \centering
        \begin{tabular}{l*{5}{c}r}
Bid  level        & One & Two & Three & Four & Five   \\
\hline\hline
 Bid price  (USD)        & 200.1 & 200.0 & 199.9 & 199.8 & 199.7   \\ \hline
Volume available & 397 & 412 & 502 & 442 & 529 &   \\
\end{tabular}
       \caption{One snapshot of the LOB.}
       \label{tab:LOB_configuration}
   \end{table}

\paragraph{Performance Metric: Implementation Shortfall \cite{perold1988implementation}.}
\begin{eqnarray}\label{eq:implementation_shortfall}
{\rm IS} (\pmb{u}) =    \left( \sum_{t=0}^{T-1} c_t(u_t)+c_T\left(q_0-\sum_{t=0}^{T-1}u_t\right)\right)-c_0(q_0).
\end{eqnarray}
 The first term of \eqref{eq:implementation_shortfall} is the cost of implementing policy $\pmb{u}$ over the horizon $[0,T]$. The second term is the cost  when liquidating $q_0$ market orders at time $0$.
 If we expect $\pmb{u}$ is better than liquidating everything at time $0$, then ${\rm IS} (\pmb{u})<0$. A smaller implementation shortfall implies the strategy is more profitable.

We use the following {\it relative performance} (evaluated on a single trajectory) to compare the performance of two policies $\pmb{u}^1$ and $\pmb{u}^2$,
\[
\text{Relative performance}=\frac{{\rm IS}(\pmb{u}^2)-{\rm IS}(\pmb{u}^1)}{|{\rm IS}(\pmb{u}^2)|}.
\]

\paragraph{Experiment Set-up.} 
We consider the LOB data consisting of the best $5$ levels  and we assume the trading frequency $\Delta = 1$ minute and the trading horizon $T=10$ minutes. We perform a numerical analysis for five different stocks, Apple (AAPL), Facebook (FB), International Business Machines Corporation (IBM), American Airlines (AAL) and JP Morgan (JPM), during the period from 01/01/2019 to 12/31/2019. The data is divided into two sets, a training set with data between 10:00AM-12:00AM 01/01/2019-08/31/2019 and a test set with data between  10:00AM-12:00AM 09/01/2019-12/31/2019.

We take $ \phi^{\prime} = 5\times 10^{-6}$; $T=10$; smoothing parameter $r=0.4$; number of trajectories $m=200$; initial policy $\pmb{K}^0\in \mathbb{R}^{1\times 20}$ with $(\pmb{K}^0)_{ij}=-0.2$ for all $i,j$; and step size $\eta=10^{-6}$. We assume the initial inventory follows $q_0=2000$.  We compare the performance of the policy gradient method with the Almgren-Chriss solution with fitted parameters given in Table \ref{tb:single-asset-estimation} in the Appendix. In the Almgren-Chriss model, we set $\phi = \sigma^2 \phi^{\prime}$ to ensure a reasonable comparison.

\paragraph{Results.} 
From Table \ref{tab:relative_performance} and Figure \ref{fig:relative_performance},
the policy gradient method improves on the Almgren-Chriss solution by around $20\%$ on five different stocks from different financial sectors. Note that the goal of the policy gradient method is to learn the global minimizer of the expected cost function, hence it is expected that the Almgren-Chriss solution could perform better than the policy gradient method for some sample trajectories, as shown in Figure \ref{fig:relative_performance}.
This result is compatible with the performance of the Q-learning algorithms \cite{HW2014}. The drawback of Q-learning algorithms is that the computational complexity is highly dependent on the size of the set of (discrete) states and actions, where as the policy gradient method can handle continuous states and actions.

We conjecture that the policy gradient method may be capable of learning the global ``optimal'' solution for a larger class of models that are ``similar'' to the LQR framework with stochastic dynamics and finite time horizon. In addition, as the policy gradient method is a model-free algorithm, it is more robust with respect to model mis-specification as compared to the Almgren-Chriss framework.
\begin{figure}[H]
  \centering
  \begin{subfigure}[b]{0.32\textwidth}
    \includegraphics[width=\textwidth]{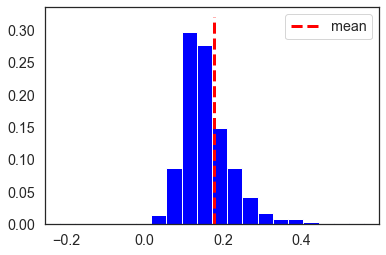}
    \caption{IBM.}
  \end{subfigure}
  \hfill
  \begin{subfigure}[b]{0.32\textwidth}
    \includegraphics[width=\textwidth]{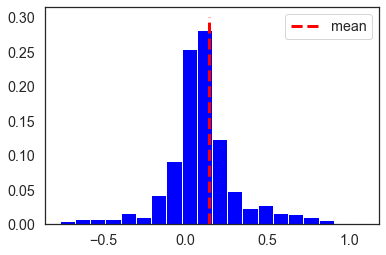}
    \caption{AAL.}
  \end{subfigure}
  \begin{subfigure}[b]{0.32\textwidth}
    \includegraphics[width=\textwidth]{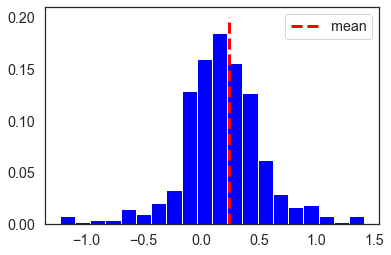}
    \caption{JPM.}
  \end{subfigure}
  \begin{subfigure}[b]{0.32\textwidth}
    \includegraphics[width=\textwidth]{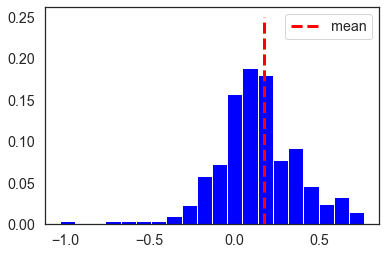}
    \caption{FB.}
  \end{subfigure}
  \begin{subfigure}[b]{0.32\textwidth}
    \includegraphics[width=\textwidth]{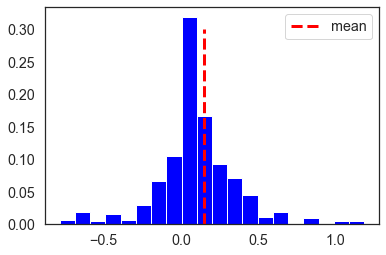}
    \caption{AAPL.}
  \end{subfigure}
  \caption{\label{fig:relative_performance}Empirical distribution of the relative performance on the test set.}
\end{figure}
\begin{table}[H]
    \centering
    \begin{tabular}{l r r r r r}
    Asset     &  IBM & AAL & JPM & FB &AAPL\\ \hline\hline
In sample       & 0.173 &0.152& 0.251 &0.181&0.165\\
(std) & (0.09) & (0.27) & (0.31) & (0.32) & (0.31)\\\hline
Out of sample   & 0.178 & 0.146 &0.245&0.175& 0.163 \\
 (std) & (0.08) & (0.29) & (0.36) & (0.24) & (0.37)\\\hline
    \end{tabular}
    \caption{Average relative performance of  the policy gradient ($\pmb{u}^1$) compared to Almgren-Chriss solution ($\pmb{u}^2$).}
    \label{tab:relative_performance}
\end{table}
\subsection{Learning LQR in Higher Dimensions}\label{sec:synthetic}
In practice we can perform the policy gradient method for the optimal liquidation problem with multiple assets. However it is difficult to capture the cross impact and permanent impact with historical LOB data. Therefore we test the performance of the policy gradient method in higher dimensions on synthetic data  consisting of a four-dimensional state variable and a two-dimensional control variable.  The parameters are randomly picked such that the conditions for our LQR framework are satisfied.

\paragraph{Set-up.} 
        (1) Parameters: 
        \[
        A =
        \begin{pmatrix}
         0.5 & 0.05 & 0.1 & 0.2\\ 
         0 & 0.2 & 0.3 & 0.1 \\
         0.06 & 0.1 & 0.2 & 0.4 \\
         0.05 & 0.2 & 0.15 & 0.1
        \end{pmatrix},
        \ B = 
        \begin{pmatrix}
         -0.05 & -0.01 \\ 
         -0.005 & -0.01 \\
         -1 & -0.01 \\
         -0.01 & -0.9 
        \end{pmatrix},
        \ Q_t=
        \begin{pmatrix}
         1 & 0.2 & -0.005 & 0.015\\ 
         0.2 & 1.1 & 0.15 & 0 \\
         -0.05 & 0.15 & 0.9 & -0.08 \\
         0.015 & 0 & -0.08 & 0.88
        \end{pmatrix},
        \]
        \[
        R_t = 
        \begin{pmatrix}
         0.4 & -0.25 \\ 
         -0.25 & 0.7
        \end{pmatrix},
        \ W=
        \begin{pmatrix}
         0.1 & 0 & 0 & 0\\ 
         0 & 0.5 & 0 & 0\\ 
         0 & 0 & 0.2 & 0\\ 
         0 & 0 & 0 & 0.3
        \end{pmatrix},
        \]
        $Q_T=Q_t$, $T=10$; smoothing parameter $r=1$, number of trajectories $m=200$; initial policy {$\pmb{K}^0\in\mathbb{R}^{2\times 40}$ with $\{\pmb{K}^0\}_{ij}= 0.05$ for all $i$, $j$}, for both known and unknown parameters;
        
        (2)  Initialization: 
        We  assume $x_0=(x_0^1,x_0^2,x_0^3,x_0^4)^\top$ and $x_0^i$ are independent. $x_0^1,x_0^2,x_0^3$, and $x_0^4$ are sampled from $\mathcal{N}(5,0.1)$, $\mathcal{N}(2,0.3)$, $\mathcal{N}(8,1)$, $\mathcal{N}(5,0.5)$.

    \paragraph{Convergence.} For the high-dimensional case, the normalized error falls below the threshold $10^{-2}$ within 80 iterations for the policy gradient algorithm with known parameters. It takes substantially more iterations for the policy gradient algorithm with unknown parameters to have an error near such a threshold, which is as expected.

\begin{figure}[h]
\centering
\begin{minipage}{10.9cm}
\begin{subfigure}[b]{0.48\textwidth}
    \includegraphics[width=\textwidth]{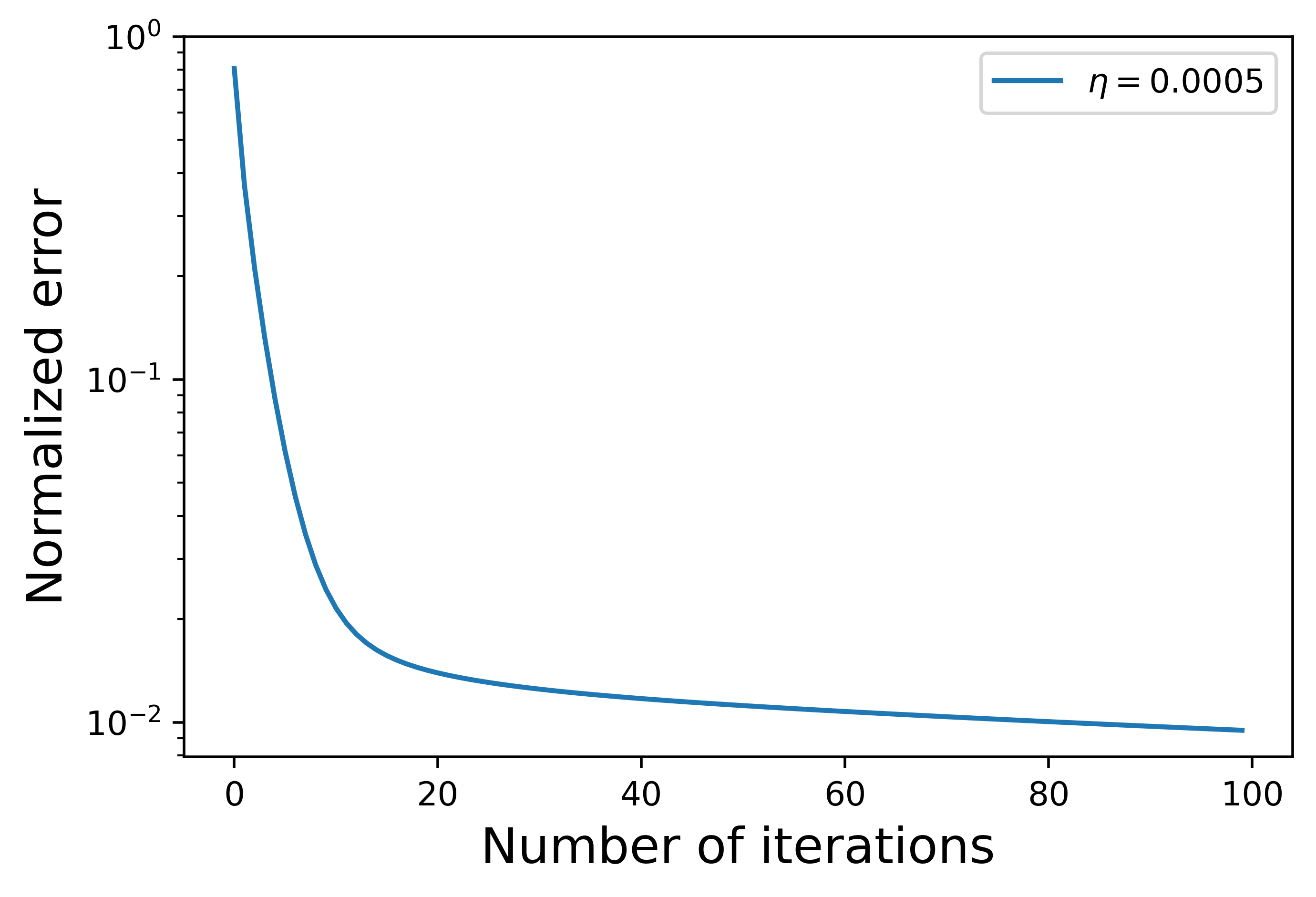}
    \caption{Known parameters \\($\eta = 0.0005$).}
  \end{subfigure}
  \begin{subfigure}[b]{0.48\textwidth}
    \includegraphics[width=\textwidth]{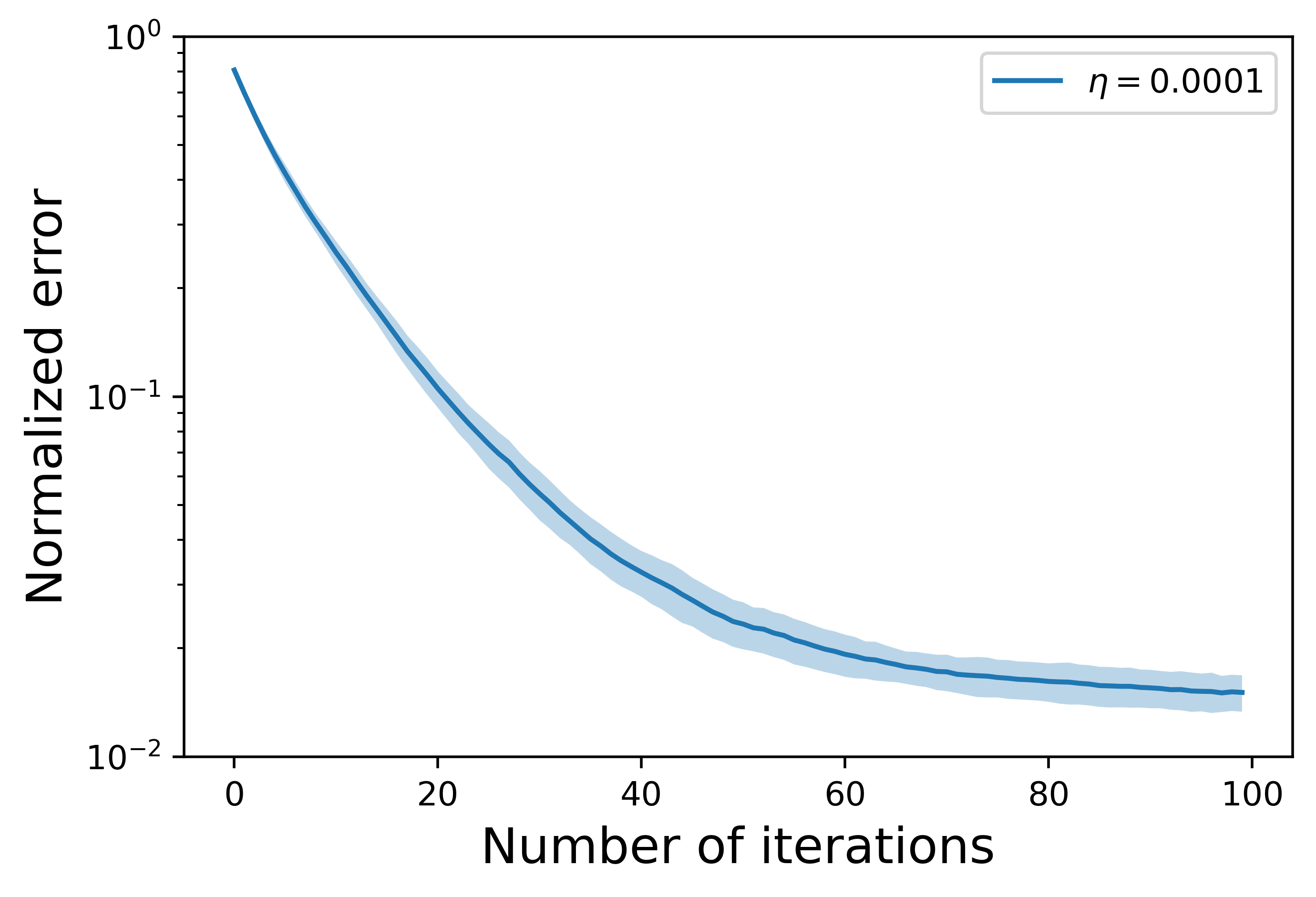}
    \caption{Unknown parameters \\($\eta = 0.0001$).}
  \end{subfigure}
  \caption{Performance of the policy gradient algorithms\\ (50 simulation scenarios)}
\end{minipage}
\qquad
\begin{minipage}{5.2cm}
\centering
\vspace{0.47cm}
\includegraphics[width=1.0\linewidth]{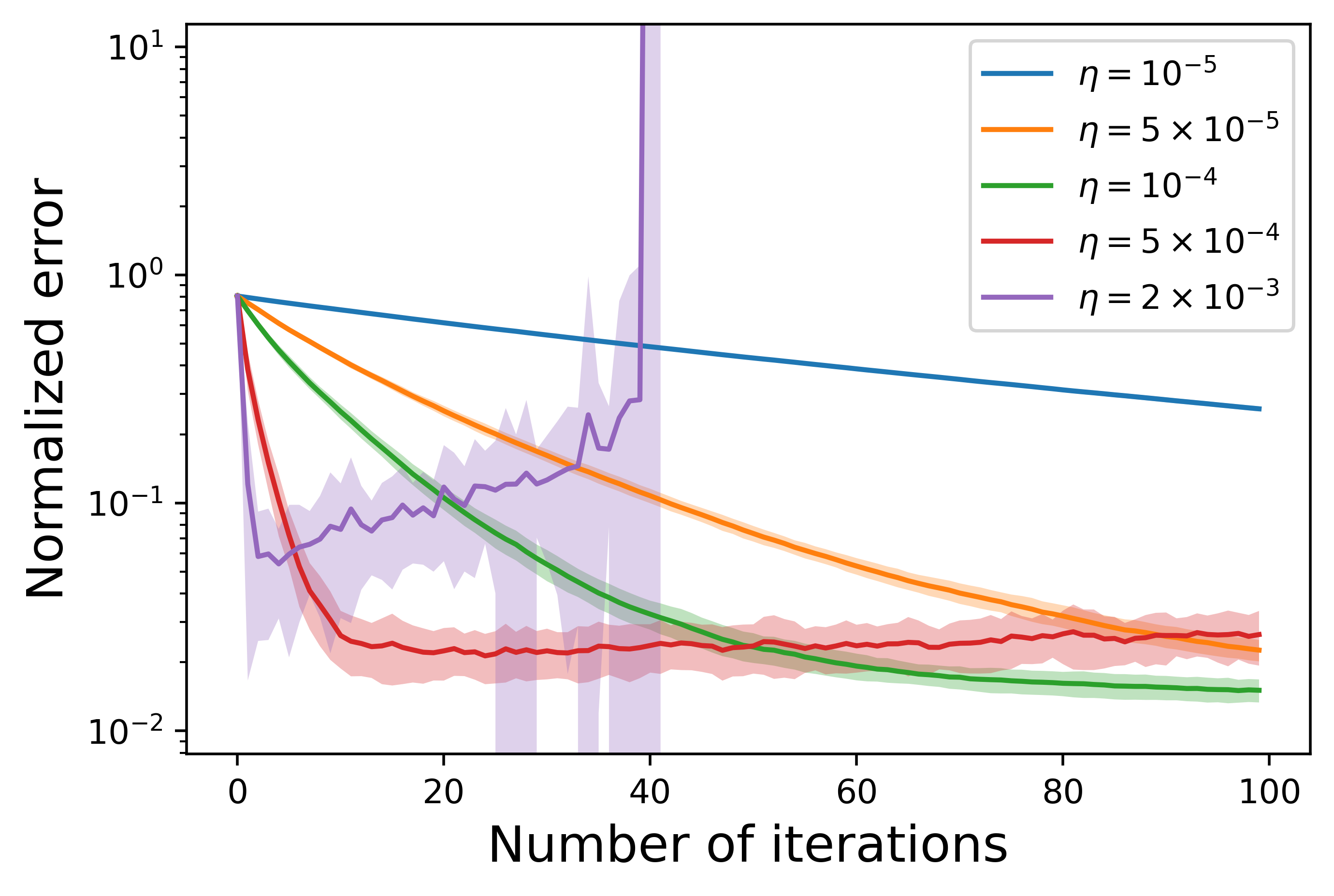}
\caption{\label{fig:eta_4D}Performance of the policy gradient algorithm with unknown parameters under different step size $\eta$ (50 simulation scenarios).}
\end{minipage}
\end{figure}


\paragraph{Outcomes from Varying the Parameter $\eta$.} The performance of the policy gradient algorithm also depends on the values of the step size $\eta$. We show how the values of the step size $\eta\in[10^{-5},2\times 10^{-3}]$ affect the convergence of the policy gradient algorithm with unknown parameters in Figure \ref{fig:eta_4D}. A tiny step size leads to slow convergence (see the blue line when $\eta=10^{-5}$) and a larger step size may cause divergence (see the purple line when $\eta=2\times 10^{-3}$).

	\newpage
\bibliographystyle{plain}
\bibliography{references}

\begin{thebibliography}{10}

\bibitem{abbasi2011regret}
Yasin Abbasi-Yadkori and Csaba Szepesv{\'a}ri.
\newblock Regret bounds for the adaptive control of linear quadratic systems.
\newblock In {\em Proceedings of the 24th Annual Conference on Learning
  Theory}, pages 1--26, 2011.

\bibitem{abeille2017thompson}
Marc Abeille and Alessandro Lazaric.
\newblock Thompson sampling for linear-quadratic control problems.
\newblock {\em AISTATS 2017 - 20th International Conference on Artificial
  Intelligence and Statistics}, 2017.

\bibitem{abeille2016}
Marc Abeille, Alessandro Lazaric, Xavier Brokmann, et~al.
\newblock {LQG} for portfolio optimization.
\newblock {\em Available at SSRN 2863925}, 2016.

\bibitem{Radoslaw2015}
Radoslaw Adamczak.
\newblock A note on the {H}anson-{W}right inequality for random vectors with
  dependencies.
\newblock {\em Electronic Communications in Probability}, 20, 2015.

\bibitem{alfonsi2011}
Aur{\'e}lien Alfonsi, Antje Fruth, and Alexander Schied.
\newblock Optimal execution strategies in limit order books with general shape
  functions.
\newblock {\em Quantitative {F}inance}, 10(2):143--157, 2010.

\bibitem{Almgren2003}
Robert Almgren.
\newblock Optimal execution with nonlinear impact functions and
  trading-enhanced risk.
\newblock {\em Applied Mathematical Finance}, 10(1):1--18, 2003.

\bibitem{AC2001}
Robert Almgren and Neil Chriss.
\newblock Optimal execution of portfolio transactions.
\newblock {\em Journal of Risk}, 3:5--40, 2001.

\bibitem{almgren2005direct}
Robert Almgren, Chee Thum, Emmanuel Hauptmann, and Hong Li.
\newblock Direct estimation of equity market impact.
\newblock {\em Risk}, 18(7):58--62, 2005.

\bibitem{anderson2007optimal}
Brian D.~O. Anderson and John~B Moore.
\newblock {\em Optimal Control: Linear Quadratic Methods}.
\newblock Courier Corporation, 2007.

\bibitem{aastrom2013adaptive}
Karl~J {\AA}str{\"o}m and Bj{\"o}rn Wittenmark.
\newblock {\em Adaptive control}.
\newblock Courier Corporation, 2013.

\bibitem{BL2019}
Wenhang Bao and Xiao-yang Liu.
\newblock Multi-agent deep reinforcement learning for liquidation strategy
  analysis.
\newblock {\em arXiv preprint arXiv:1906.11046}, 2019.

\bibitem{Bertsekas2015}
Dimitri Bertsekas.
\newblock {\em Dynamic Programming And Optimal Control}, volume~1.
\newblock Athena Scientific, 3rd edition, 2005.

\bibitem{bhandari2019}
Jalaj Bhandari and Daniel Russo.
\newblock Global optimality guarantees for policy gradient methods.
\newblock {\em arXiv preprint arXiv:1906.01786}, 2019.

\bibitem{bu2019}
Jingjing Bu, Afshin Mesbahi, Maryam Fazel, and Mehran Mesbahi.
\newblock {LQR} through the lens of first order methods: discrete-time case.
\newblock {\em arXiv preprint arXiv:1907.08921}, 2019.

\bibitem{bu2020}
Jingjing Bu, Afshin Mesbahi, and Mehran Mesbahi.
\newblock Policy gradient-based algorithms for continuous-time linear quadratic
  control.
\newblock {\em arXiv preprint arXiv:2006.09178}, 2020.

\bibitem{bu2019global}
Jingjing Bu, Lillian~J Ratliff, and Mehran Mesbahi.
\newblock Global convergence of policy gradient for sequential zero-sum linear
  quadratic dynamic games.
\newblock {\em arXiv preprint arXiv:1911.04672}, 2019.

\bibitem{carmona2019linear}
Ren{\'e} Carmona, Mathieu Lauri{\`e}re, and Zongjun Tan.
\newblock Linear-quadratic mean-field reinforcement learning: convergence of
  policy gradient methods.
\newblock {\em arXiv preprint arXiv:1910.04295}, 2019.

\bibitem{charpentier2020}
Arthur Charpentier, Romuald Elie, and Carl Remlinger.
\newblock Reinforcement learning in economics and finance.
\newblock {\em arXiv preprint arXiv:2003.10014}, 2020.

\bibitem{CKS2014}
Rama Cont, Arseniy Kukanov, and Sasha Stoikov.
\newblock The price impact of order book events.
\newblock {\em Journal of Financial Econometrics}, 12(1):47--88, 2014.

\bibitem{dean2019}
Sarah Dean, Horia Mania, Nikolai Matni, Benjamin Recht, and Stephen Tu.
\newblock On the sample complexity of the linear quadratic regulator.
\newblock {\em Foundations of Computational Mathematics}, pages 1--47, 2019.

\bibitem{faradonbeh2020optimism}
Mohamad Kazem~Shirani Faradonbeh, Ambuj Tewari, and George Michailidis.
\newblock Optimism-based adaptive regulation of linear-quadratic systems.
\newblock {\em IEEE Transactions on Automatic Control}, 2020.

\bibitem{fattahi2020efficient}
Salar Fattahi, Nikolai Matni, and Somayeh Sojoudi.
\newblock Efficient learning of distributed linear-quadratic control policies.
\newblock {\em SIAM Journal on Control and Optimization}, 58(5):2927--2951,
  2020.

\bibitem{FGKM2018}
Maryam Fazel, Rong Ge, Sham~M Kakade, and Mehran Mesbahi.
\newblock Global convergence of policy gradient methods for the linear
  quadratic regulator.
\newblock {\em Proceedings of the 35th International Conference on Machine
  Learning}, pages 1467--1476, 2018.

\bibitem{fiechter1997pac}
Claude-Nicolas Fiechter.
\newblock {PAC} adaptive control of linear systems.
\newblock In {\em Proceedings of the Tenth Annual Conference on Computational
  Learning Theory}, pages 72--80, 1997.

\bibitem{Flaxman2005}
Abraham~D. Flaxman, Adam~Tauman Kalai, and H.~Brendan McMahan.
\newblock Online convex optimization in the bandit setting: Gradient descent
  without a gradient.
\newblock In {\em Society for Industrial and Applied Mathematics}, SODA '05,
  pages 385--394, USA, 2005.

\bibitem{gatheral2011}
Jim Gatheral and Alexander Schied.
\newblock Optimal trade execution under geometric {B}rownian motion in the
  {A}lmgren and {C}hriss framework.
\newblock {\em International Journal of Theoretical and Applied Finance},
  14(03):353--368, 2011.

\bibitem{gravell2019}
Benjamin Gravell, Peyman~Mohajerin Esfahani, and Tyler Summers.
\newblock Learning robust controllers for linear quadratic systems with
  multiplicative noise via policy gradient.
\newblock {\em arXiv preprint arXiv:1905.13547}, 2019.

\bibitem{gross2011recovering}
David Gross.
\newblock Recovering low-rank matrices from few coefficients in any basis.
\newblock {\em IEEE Transactions on Information Theory}, 57(3):1548--1566,
  2011.

\bibitem{guo2020entropy}
Xin Guo, Renyuan Xu, and Thaleia Zariphopoulou.
\newblock Entropy regularization for mean field games with learning.
\newblock {\em arXiv preprint arXiv:2010.00145}, 2020.

\bibitem{HW2014}
Dieter Hendricks and Diane Wilcox.
\newblock A reinforcement learning extension to the {A}lmgren-{C}hriss
  framework for optimal trade execution.
\newblock In {\em 2014 IEEE Conference on Computational Intelligence for
  Financial Engineering \& Economics (CIFEr)}, pages 457--464. IEEE, 2014.

\bibitem{ibrahimi2012efficient}
Morteza Ibrahimi, Adel Javanmard, and Benjamin~V Roy.
\newblock Efficient reinforcement learning for high dimensional linear
  quadratic systems.
\newblock In {\em Advances in Neural Information Processing Systems}, pages
  2636--2644, 2012.

\bibitem{jin2020analysis}
Zeyu Jin, Johann~Michael Schmitt, and Zaiwen Wen.
\newblock On the analysis of model-free methods for the linear quadratic
  regulator.
\newblock {\em arXiv preprint arXiv:2007.03861}, 2020.

\bibitem{leal2020learning}
Laura Leal, Mathieu Lauri{\`e}re, and Charles-Albert Lehalle.
\newblock Learning a functional control for high-frequency finance.
\newblock {\em arXiv preprint arXiv:2006.09611}, 2020.

\bibitem{li2004iterative}
Weiwei Li and Emanuel Todorov.
\newblock Iterative linear quadratic regulator design for nonlinear biological
  movement systems.
\newblock In {\em ICINCO}, pages 222--229, 2004.

\bibitem{malik2019derivative}
Dhruv Malik, Ashwin Pananjady, Kush Bhatia, Koulik Khamaru, Peter Bartlett, and
  Martin Wainwright.
\newblock Derivative-free methods for policy optimization: guarantees for
  linear quadratic systems.
\newblock In {\em The 22nd International Conference on Artificial Intelligence
  and Statistics}, pages 2916--2925. PMLR, 2019.

\bibitem{nesterov2003introductory}
Yurii Nesterov.
\newblock {\em Introductory lectures on convex optimization: A basic course},
  volume~87.
\newblock Springer Science \& Business Media, 2003.

\bibitem{NFK2006}
Yuriy Nevmyvaka, Yi~Feng, and Michael Kearns.
\newblock Reinforcement learning for optimized trade execution.
\newblock In {\em Proceedings of the 23rd International Conference on Machine
  Learning}, pages 673--680, 2006.

\bibitem{NLJ2018}
Brian Ning, Franco Ho~Ting Ling, and Sebastian Jaimungal.
\newblock Double deep {Q}-learning for optimal execution.
\newblock {\em arXiv preprint arXiv:1812.06600}, 2018.

\bibitem{ouyang2017control}
Yi~Ouyang, Mukul Gagrani, and Rahul Jain.
\newblock Control of unknown linear systems with {T}hompson sampling.
\newblock In {\em 2017 55th Annual Allerton Conference on Communication,
  Control, and Computing (Allerton)}, pages 1198--1205. IEEE, 2017.

\bibitem{patrinos2011}
Panagiotis Patrinos, Sergio Trimboli, and Alberto Bemporad.
\newblock Stochastic {MPC} for real-time market-based optimal power dispatch.
\newblock In {\em 2011 50th IEEE Conference on Decision and Control and
  European Control Conference}, pages 7111--7116. IEEE, 2011.

\bibitem{perold1988implementation}
Andre~F Perold.
\newblock The implementation shortfall: Paper versus reality.
\newblock {\em Journal of Portfolio Management}, 14(3):4, 1988.

\bibitem{ReBen2019}
Benjamin Recht.
\newblock A tour of reinforcement learning: The view from continuous control.
\newblock {\em Annual Review of Control, Robotics, and Autonomous Systems},
  2(1):253--279, 2019.

\bibitem{tu2018least}
Stephen Tu and Benjamin Recht.
\newblock Least-squares temporal difference learning for the linear quadratic
  regulator.
\newblock In {\em International Conference on Machine Learning}, pages
  5005--5014, 2018.

\bibitem{tu2019}
Stephen Tu and Benjamin Recht.
\newblock The gap between model-based and model-free methods on the linear
  quadratic regulator: An asymptotic viewpoint.
\newblock In {\em Conference on Learning Theory}, pages 3036--3083, 2019.

\bibitem{wasa2017}
Yasuaki Wasa, Kengo Sakata, Kenji Hirata, and Kenko Uchida.
\newblock Differential game-based load frequency control for power networks and
  its integration with electricity market mechanisms.
\newblock In {\em 2017 IEEE Conference on Control Technology and Applications
  (CCTA)}, pages 1044--1049. IEEE, 2017.

\bibitem{yang2019}
Zhuoran Yang, Yongxin Chen, Mingyi Hong, and Zhaoran Wang.
\newblock On the global convergence of actor-critic: a case for linear
  quadratic regulator with ergodic cost.
\newblock {\em arXiv preprint arXiv:1907.06246}, 2019.

\bibitem{zhang2019policy}
Kaiqing Zhang, Zhuoran Yang, and Tamer Basar.
\newblock Policy optimization provably converges to {N}ash equilibria in
  zero-sum linear quadratic games.
\newblock In {\em Advances in Neural Information Processing Systems}, pages
  11602--11614, 2019.

\bibitem{ZZR2020}
Zihao Zhang, Stefan Zohren, and Stephen Roberts.
\newblock Deep reinforcement learning for trading.
\newblock {\em The Journal of Financial Data Science}, 2(2):25--40, 2020.

\end{thebibliography}

\newpage
\appendix
\section{Market Simulator for Linear Price Dynamics}\label{sec:parameter_est}
 We estimate the parameters for the LQR model using NASDAQ ITCH data taken from Lobster\footnote{https://lobsterdata.com/}. 

\paragraph{Permanent Price Impact and Volatility.}
 The model in \eqref{eq:liquidaton_states} implies that prices changes are proportional to the {\it market-order flow imbalances} (MFI). We adopt the framework from \cite{CKS2014}, namely that the price change $\Delta S$ is given by
\begin{eqnarray}
\Delta S = \gamma \, \text{MFI}+\sigma\,\epsilon,
\end{eqnarray}
with $\text{MFI} = M^b-M^s$ where $M^s$ and $M^b$ are the volumes of market sell orders and market buy orders respectively during a time interval $\Delta T=5$mins and $\epsilon\sim \mathcal{N}(0,1)$. We then estimate $\gamma$ and $\sigma$ from the data.
	\begin{figure}[H]
\centering
\includegraphics[width=0.19\linewidth]{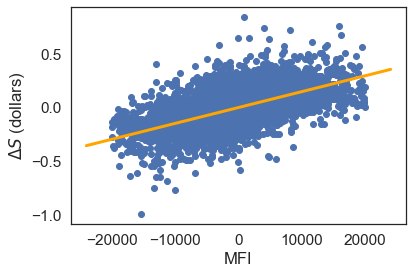}
\includegraphics[width=0.19\linewidth]{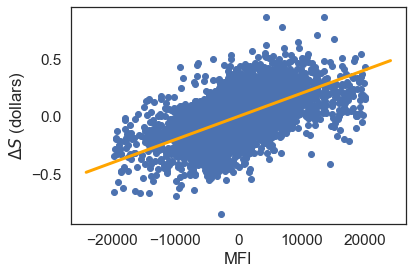}
\includegraphics[width=0.19\linewidth]{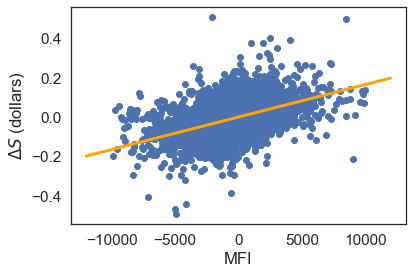}
\includegraphics[width=0.19\linewidth]{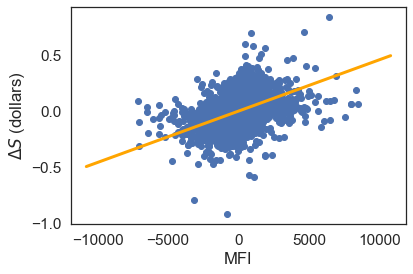}
\includegraphics[width=0.19\linewidth]{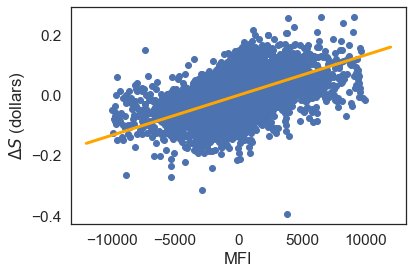}
\caption{Relationship between MFI and $\Delta S$. (Example (from left to right): AAP, FB, JPM, IBM and AAL, 10:00AM-11:00AM 01/01/2019-08/31/2019, $\Delta T = 1$min)}
\end{figure}

\paragraph{Temporary Price Impact.} We assume the LOB has a flat shape with constant queue length $l$ for the first few levels. Figure~\ref{fig:queue_length} shows the average queue lengths for the first 5 levels so that our assumption is not too unreasonable. Therefore the following equation, on the amount received when we liquidate $u$ shares with best bid price $S$, holds
\[
u(S-\beta u) = \int_{S-\frac{u\,\Delta }{l}}^{S} l v dv.
\]
Therefore we have $\beta =\frac{\Delta}{2l}$, where $\Delta$ is the tick size and $l$ is the average queue length.
	\begin{figure}[H]
\centering
\includegraphics[width=0.19\linewidth]{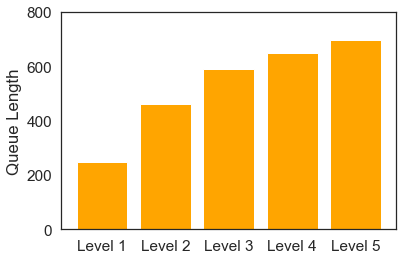}
\includegraphics[width=0.19\linewidth]{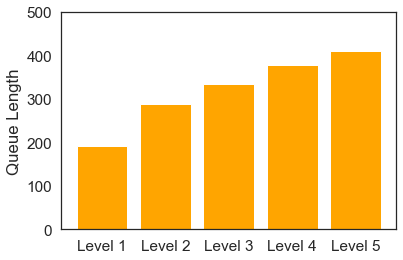}
\includegraphics[width=0.19\linewidth]{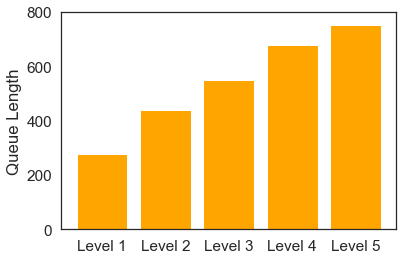}
\includegraphics[width=0.19\linewidth]{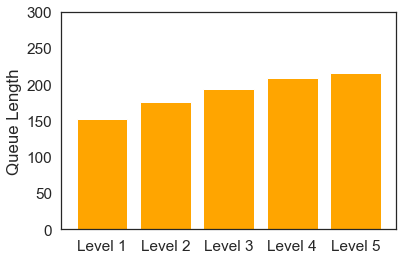}
\includegraphics[width=0.19\linewidth]{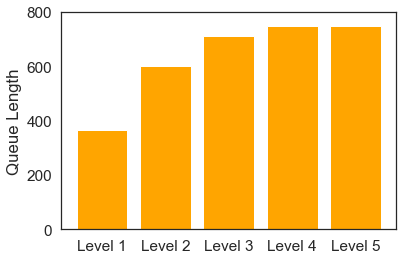}
\caption{\label{fig:queue_length}Average queue length (volume) of the first five levels on the limit buy side (Example (from left to right): AAP, FB, JPM, IBM and AAL, 10:00AM-11:00AM 01/01/2019-08/31/2019 with $5000$ samples uniformly sampled with natural time clock in each trading day.) }
\end{figure}

\paragraph{Parameter Estimation.}
	See the estimates for AAPL, FB, IBM, JPM, and AAL in Table \ref{tb:single-asset-estimation}.

   \begin{table}[H]
  \centering
\begin{tabular}{ |p{3cm}||p{2cm}p{2cm}p{2.2cm}p{2cm}p{2cm}|| }
  \hline
   Paramters/Stock& AAPL  & FB & IBM &JPM &AAL \\
   \hline\hline
  $\beta$ & $1.03\times 10^{-5}$ &  $1.30\times 10^{-5}$ &$2.65\times{10}**{-5}$&$9.28\times 10^{-6}$&$3.27\times 10^{-5}$\\
   $\gamma$  & $7.27\times 10^{-6}$  & $1.40\times 10^{-5}$ &$4.60\times 10^{-5}$&$1.65\times 10^{-5}$&$1.3310\times 10^{-5}$ \\
   ${ \sigma}$  & $0.107$ &$0.115$&$0.082$ &$0.059$&$0.042$ \\ 
  \hline\hline
  \end{tabular}
  \caption{\label{tb:single-asset-estimation}Parameter estimation from NASDAQ ITCH Data (10:00AM-11:00PM 01/01/2019-08/31/2019).}
  \end{table}

\section{Comparison between the Policy Gradient Method and Q-learning}\label{appendix:q_learning}

The computational complexity of Q-learning is highly dependent on the size of the set of the  (discrete) states and actions. Therefore Q-learning is typically less suited to problems with continuous and unbounded states and actions. In order to apply Q-learning for such problems, we need to discretize the continuous state and action space. Intuitively speaking, Q-learning suffers from low accuracy when the discretization scheme is less refined (see Figures \ref{fig:q_table1} and \ref{fig:q_learning}). On the other hand, the computational complexity grows quadratically when increasing the level of granularity of discretization (see Figures \ref{fig:q_table2} and \ref{fig:q_learning}).

To demonstrate this view point, we  compare the performance of the Q-learning algorithm with the policy gradient method on a one-dimensional LQR problem with finite horizon as suggested by the reviewer. (We would expect the deep Q-learning algorithm and the deep policy gradient method to have similar comparison results.)

\paragraph{Q learning update.}
We initialize the Q table $\{q^{(0)}_t(x,u)\}_{x,u,t}$ with all zeros.
In the $i$-th iteration, we update the Q table for  $t=0,1,\cdots,T-1$,
\begin{eqnarray}
q^{(i)}_t(x,u) = (1-\tilde{\eta}) \, q^{(i-1)}_t(x,u) + \tilde{\eta}\, \left[ c_t(x,u) +\min_{u^{\prime}}\,\,q_{t+1}^{(i)}(x^{\prime},u^{\prime}) \right],
\end{eqnarray}
with terminal condition $q^{(i)}_{T}(x,u) = x^2Q_T$. Here $c_t(x,u) = x^2Q_t +u^2R_t$ is the instantaneous cost at time $t$; $x'$ is the next state simulated from the system when the agent takes an action $u$ in state $x$ at time $t$; and $\tilde{\eta} \in (0,1)$ is the learning rate.

\paragraph{Model set-up.} We set $d=k=1$, $T=5$, $A=1.0$, $B=0.2$, $Q_t=0.2$ for $t=0,1,2,3,4$, $Q_T=0.4$, $R_t = 0.1(t+1)$ for $t=0,1,\cdots,4$, $w_t\sim\mathcal{N}(0,0.1)$, and $x_0\sim\mathcal{N}(0,0.1)$.

\paragraph{Parameter set-up.}
To perform Q-learning, we uniformly partition the states and actions in $[-1,1]$. We set the learning rate for Q-learning as $\tilde{\eta}=0.1$. For the policy gradient method, we set the learning rate $\eta = 0.2$ and the number of trajectories in the zero-th optimization as $m=50$.

\paragraph{Conclusion.} From Figure \ref{fig:q_learning}, we observe that
\begin{itemize}
    \item For LQR with finite horizon, the policy gradient method outperforms Q-learning algorithms (with the size of actions and states varying from 10 to 100) in terms of both sample efficiency and accuracy.
    \item When increasing the size of the states and actions from 10 to 100, the accuracy of the Q-learning algorithm improves, however, it requires many more samples to converge.
\end{itemize}
To conclude,  Q-learning is less suited to handling decision-making problems with continuous and unbounded states and actions. More advanced approximation  techniques may be needed in this case \cite{tu2018least}.

	\begin{figure}[H]
\centering
\includegraphics[width=0.32\linewidth]{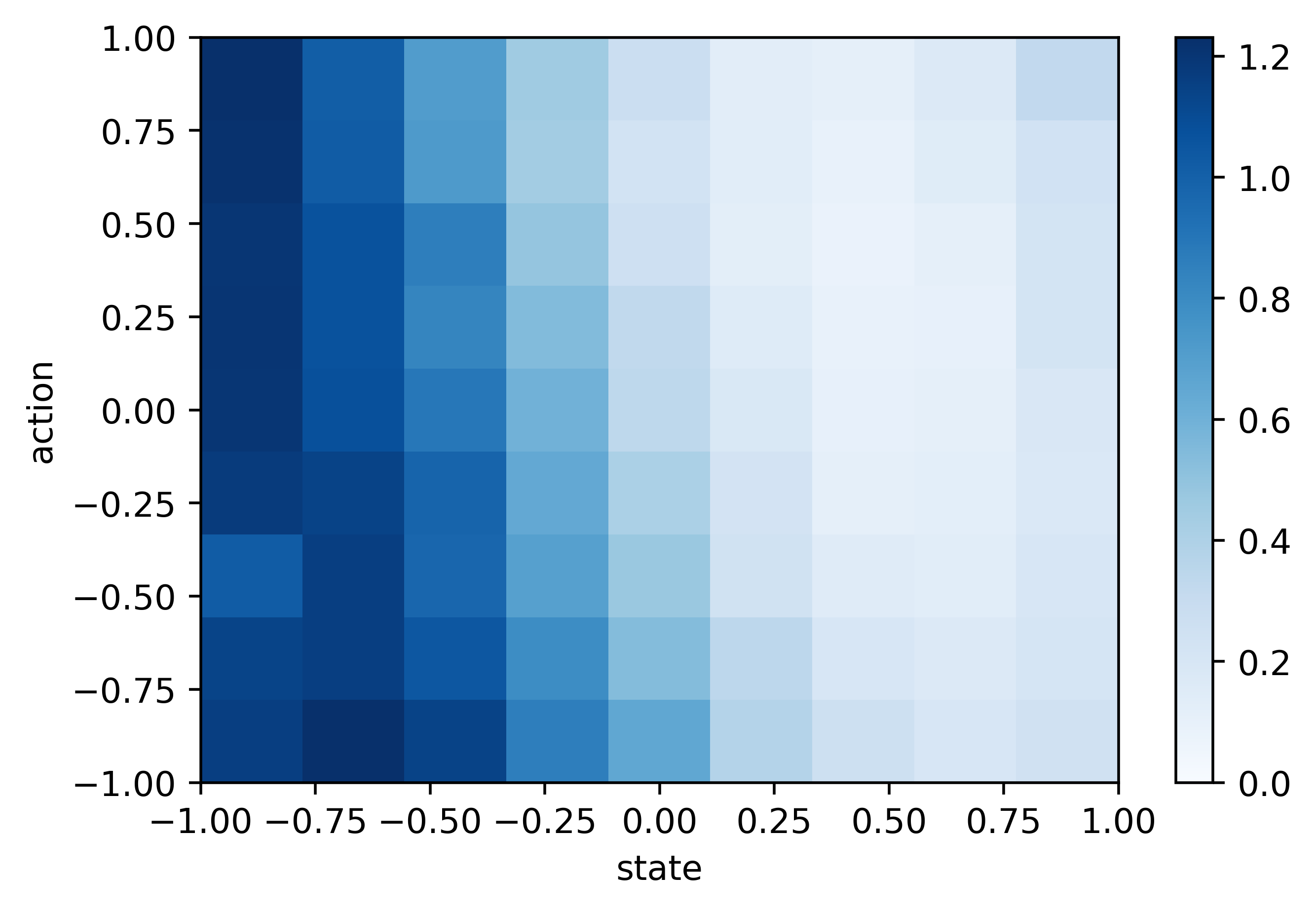}
\includegraphics[width=0.32\linewidth]{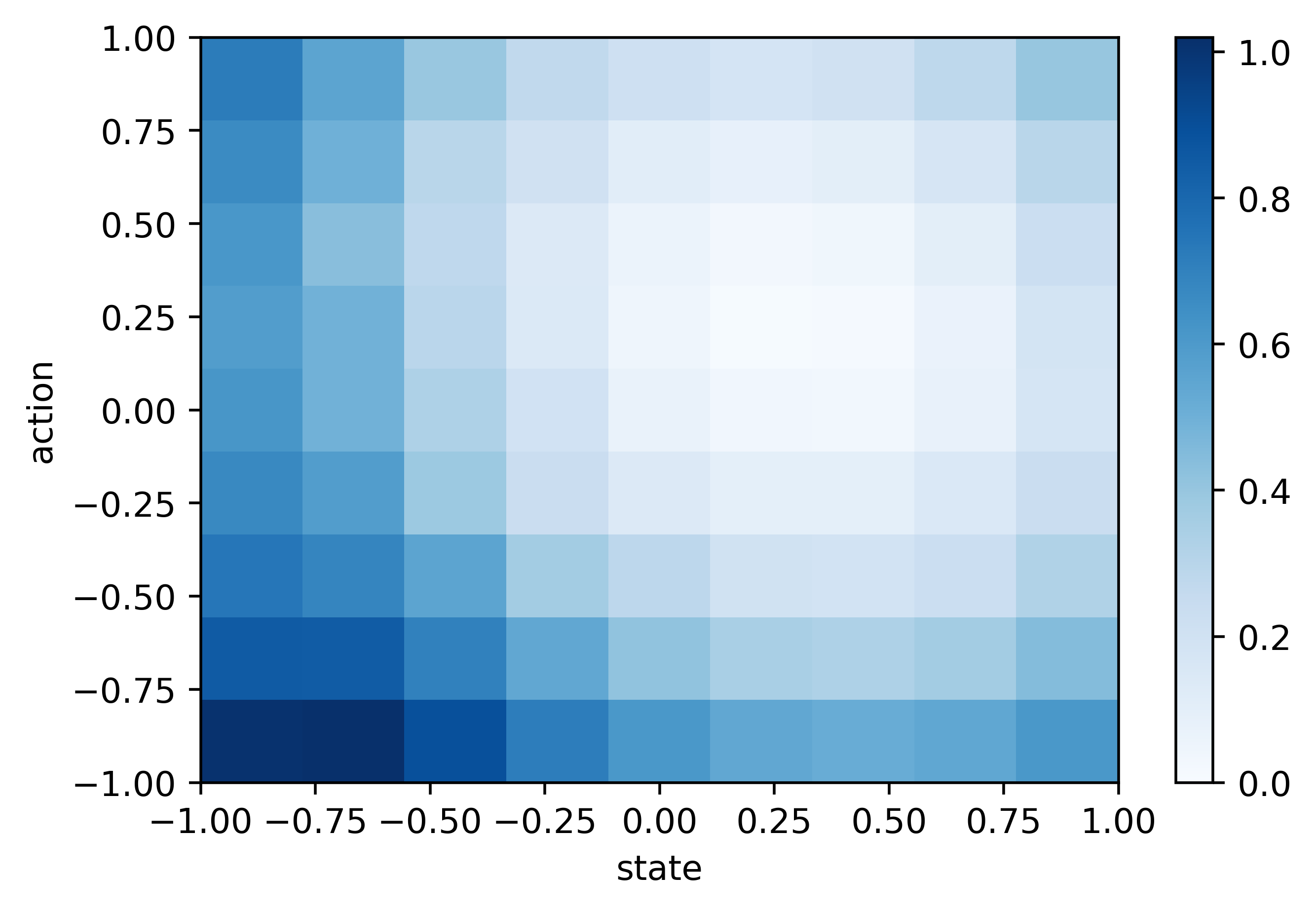}
\includegraphics[width=0.32\linewidth]{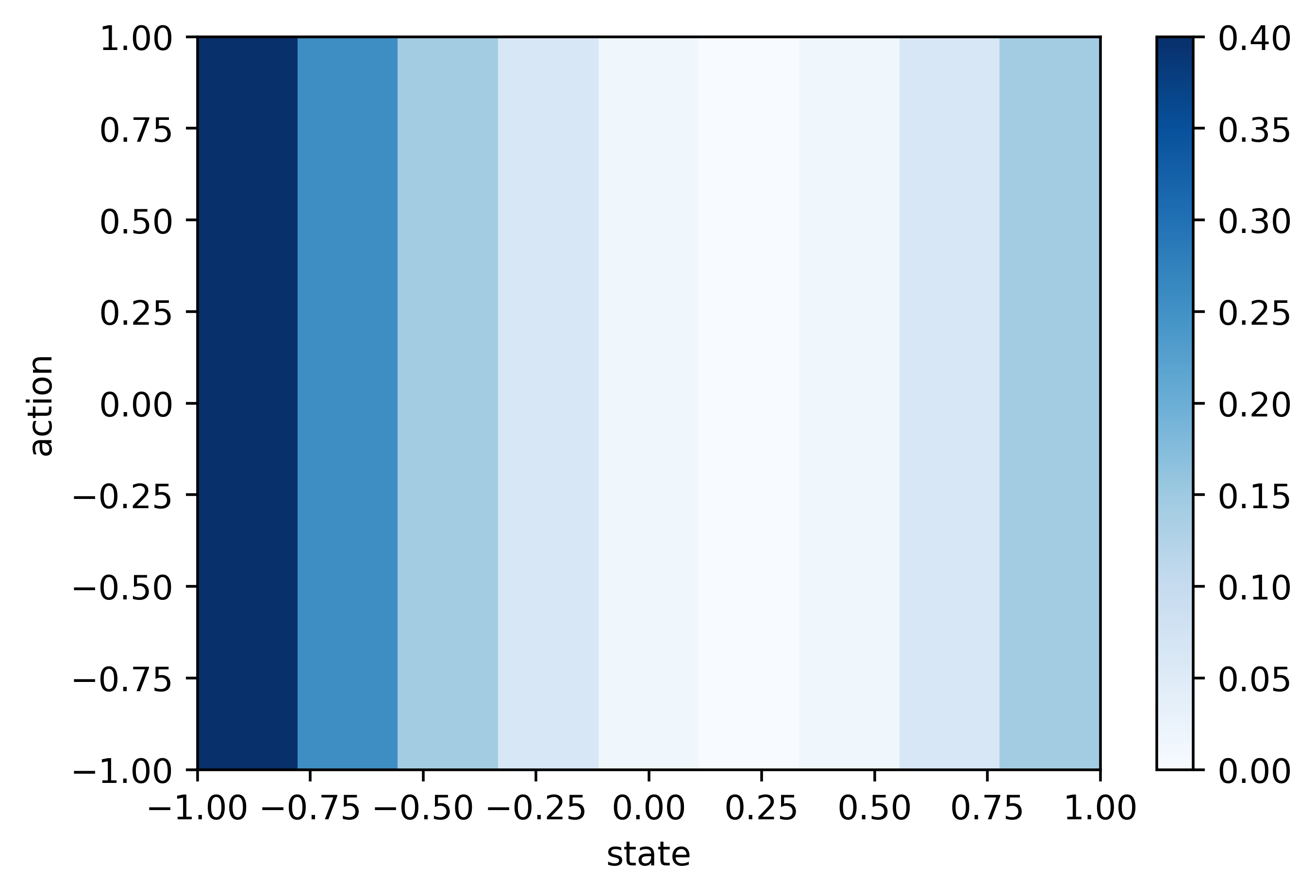}
\caption{\label{fig:q_table1}Q tables with $10$ states and $10$ actions: $q_0(s,a)$, $q_4(s,a)$ and $q_5(s,a)$ (from left to right).}
\end{figure}
	\begin{figure}[H]
\centering
\includegraphics[width=0.32\linewidth]{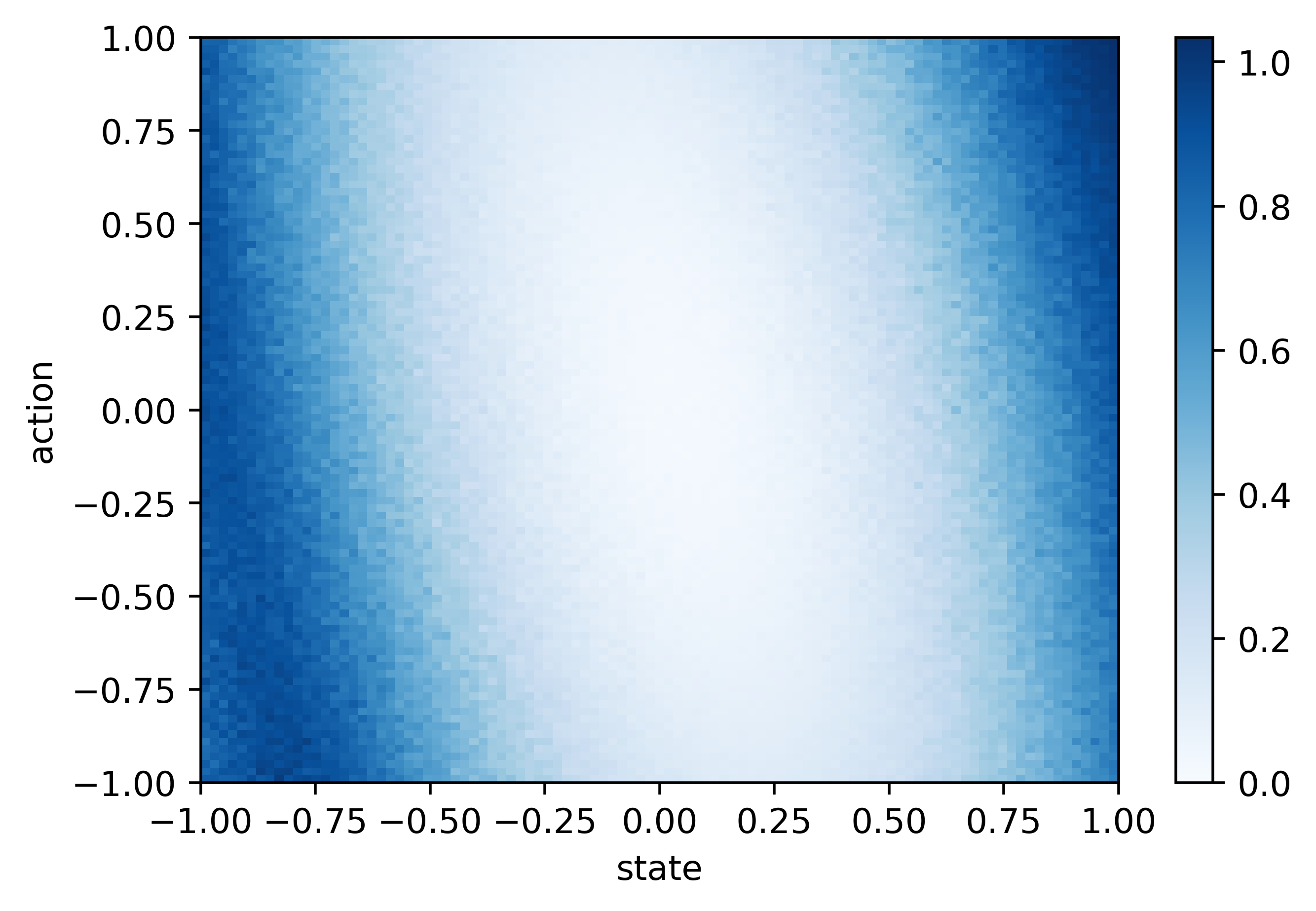}
\includegraphics[width=0.32\linewidth]{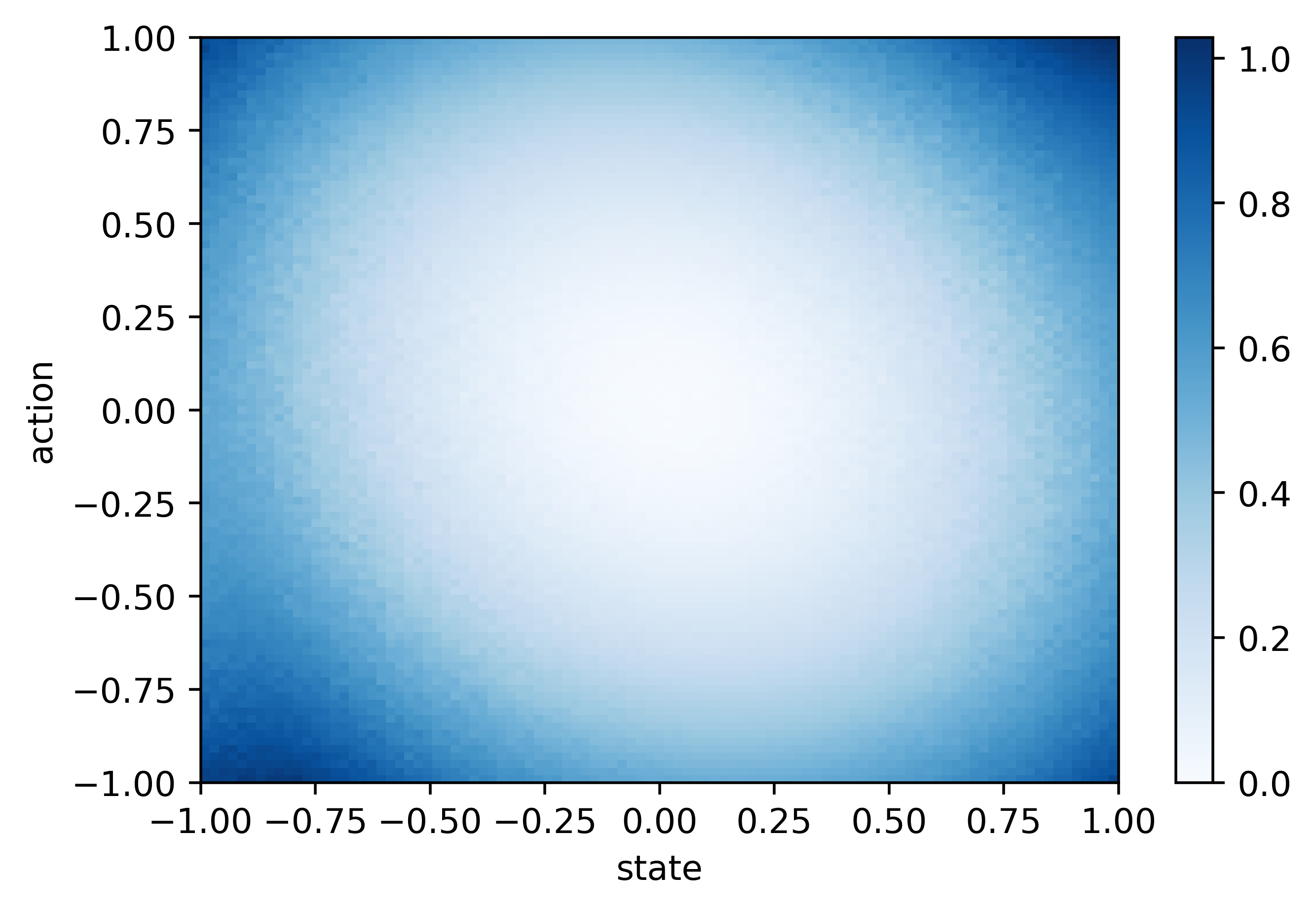}
\includegraphics[width=0.32\linewidth]{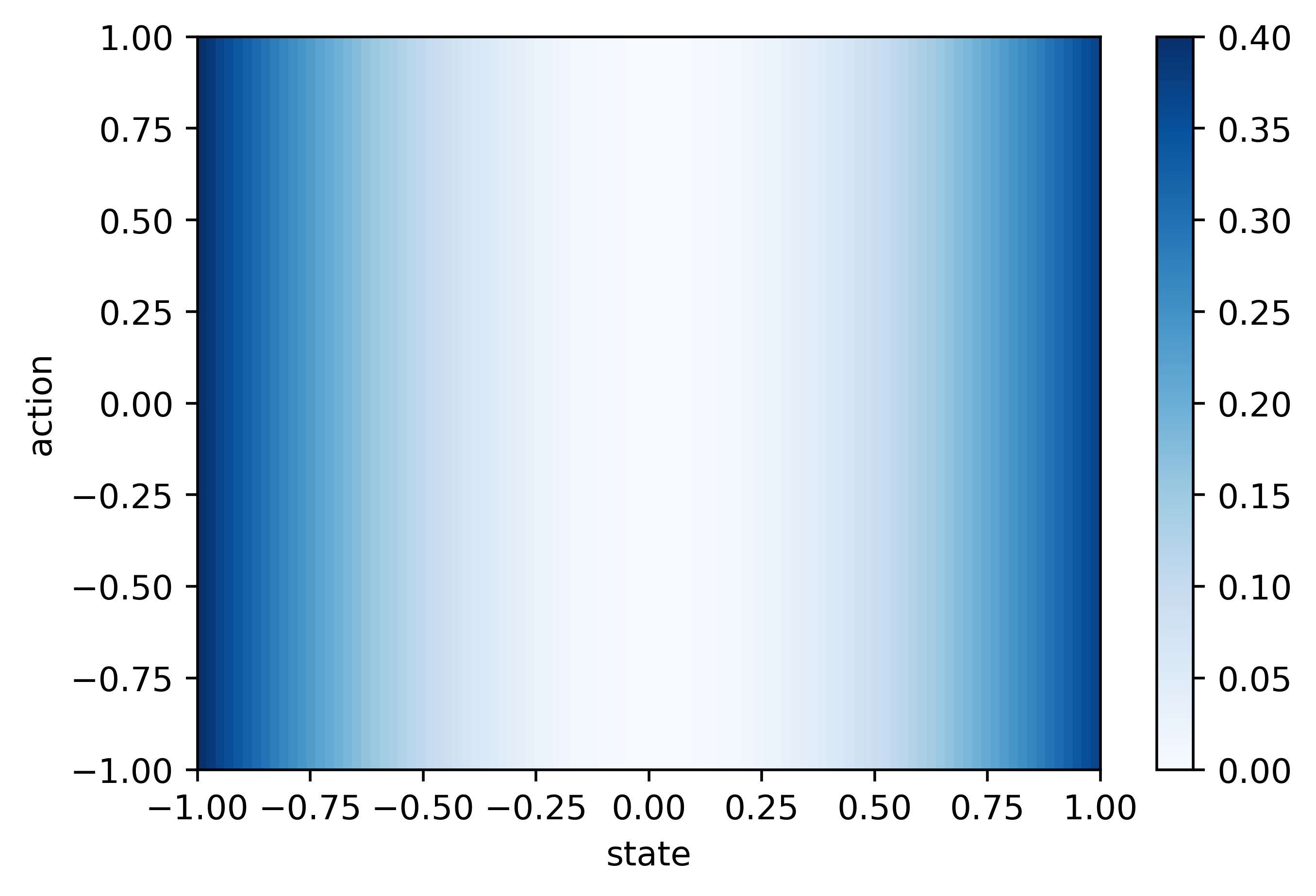}
\caption{\label{fig:q_table2} Q tables with $100$ states and $100$ actions: $q_0(s,a)$, $q_4(s,a)$ and $q_5(s,a)$ (from left to right).}
\end{figure}

	\begin{figure}[H]
\centering
\includegraphics[width=0.5\linewidth]{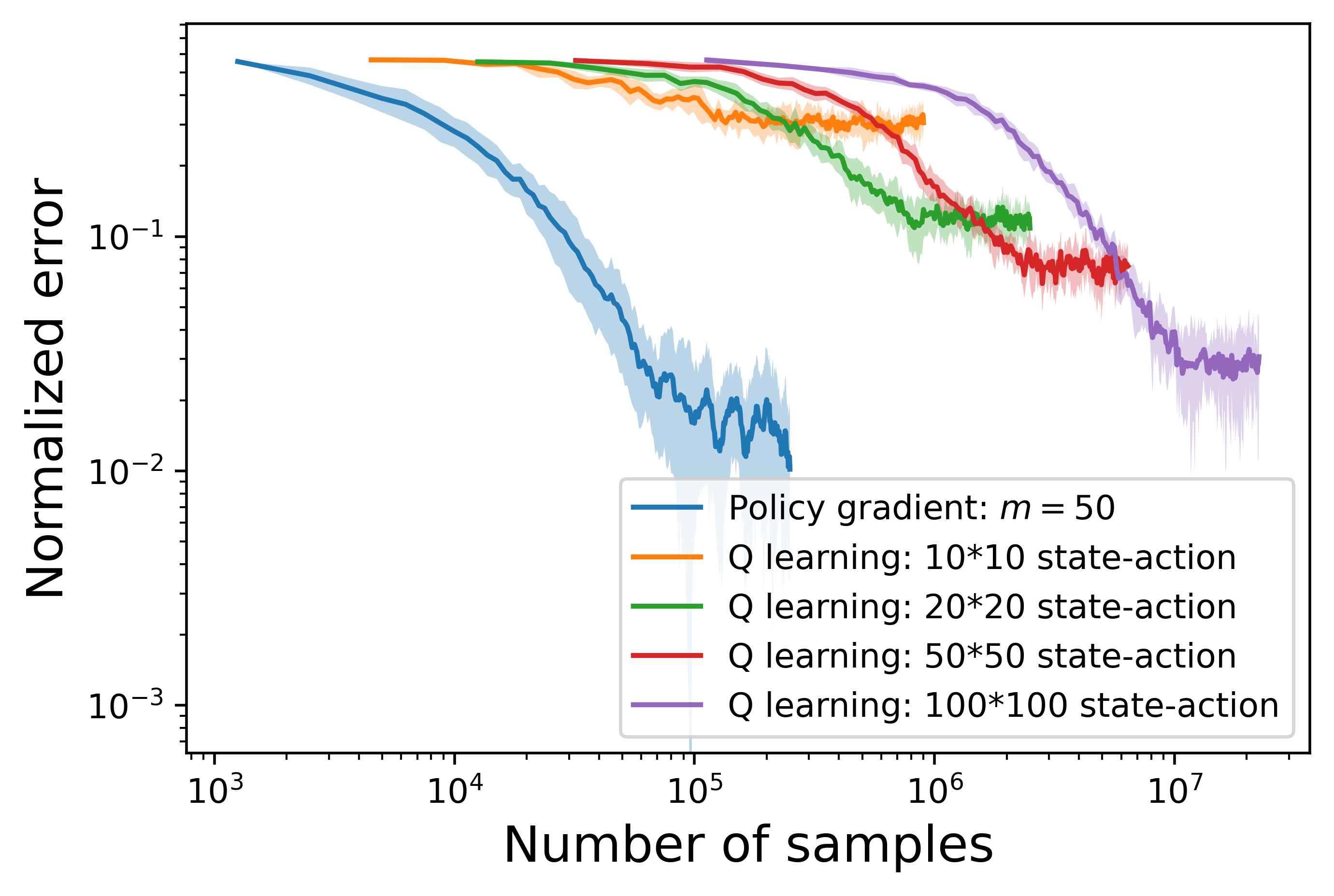}
\caption{\label{fig:q_learning}Comparison between Q-learning and the policy gradient method (log-log scale). (Average across 10 scenarios.)}
\end{figure}

\section{Proofs of Technical Results}\label{appendix:missing_proofs}

We now give the proofs that were omitted in the text.

\subsection{Proofs in Section \ref{sec:model_based_regularity}}\label{appendix:missing_proofs_sec3_1}

\begin{proof}[{\bf Proof of Lemma \ref{lemma:covariance}}] Denote by $\{x_t\}_{t=0}^T$ the state trajectory induced by an arbitrary control $\pmb{K}$. By Assumption \ref{ass:State} the matrix $\mathbb{E}[x_0 x_0^\top ]$ is positive definite. For $t \ge 1$, we have
$$\mathbb{E}[x_t x_t^\top ] =(A-BK_{t-1}) \mathbb{E}[x_{t-1}x_{t-1}^\top](A-BK_{t-1})^\top + \mathbb{E}[w_{t-1} w_{t-1}^\top].$$ 
Now $(A-BK_{t-1}) \mathbb{E}[x_{t-1}x_{t-1}^\top](A-BK_{t-1})^\top$ is positive semi-definite and $\mathbb{E}[w_{t-1} w_{t-1}^\top]$ is positive definite. Hence $\mathbb{E}[x_t x_t^\top ]$ is positive definite and as a result $\sx>0$. In this case, we can simply take $\sx =$\\ $\min(\mathbb{E}[x_0x_0^\top],\sigma_{\min}(W))$.
\end{proof}

\begin{proof}[{\bf Proof of Proposition \ref{prop:P_positive}}]
This can be proved by backward induction. For $t=T$,  $ P_{T}^{\pmb{K}} = Q_T$ is positive definite since $Q_T$ is positive definite. Assume $P_{t+1}^{\pmb{K}}$ is positive definite for some $t+1$, then take any $z \in \mathbb{R}^d$ such that $z \neq 0$,
\[
z^{\top} P_t^{\pmb{K}} z=z^{\top} Q_t\,z+z^{\top}K_t^{\top}R_tK_tz+z^{\top}\left(A-BK_t\right)^{\top}P_{t+1}^{\pmb{K}}\left(A-BK_t\right)z > 0.
\]
The last inequality holds since $z^{\top} Q_t\,z>0$, $z^{\top}K_t^{\top}R_tK_tz \ge 0$ and  $z^{\top}\left(A-BK_t\right)^{\top}P_{t+1}^{\pmb{K}}\left(A-BK_t\right)z \ge 0.$ By backward induction, we have $P_t^{\pmb{K}}$ positive definite, $\forall\,t=0,1,\cdots,T$.
\end{proof}

To prove Lemma \ref{lemma 11}, let us start with a useful result for the value function. Define the value function $V_{\pmb{K}}(x,\tau)$ for $\tau=0,1,\cdots,T-1$, as 
\begin{equation*}
    V_{\pmb{K}}(x,\tau) = \mathbb{E}_{\pmb{w}}\left.\left[\sum_{t=\tau}^{T-1}(x_t^{\top}Q_tx_t+u_t^{\top}R_tu_t)+x_T^{\top}Q_Tx_T\right|x_{\tau} = x\right]=x^{\top}P_\tau x +{ L_{\tau}},
\end{equation*}
 with terminal condition 
\begin{equation*}
    V_{\pmb{K}}(x,T)=x^{\top}Q_Tx ,
\end{equation*}
where ${ L_\tau}$ is defined in \eqref{qt}.
We then define the $Q$ function, $Q_{\pmb{K}}(x,u,\tau)$ for $\tau=0,1,\cdots,T-1$ as
\begin{equation*}
    Q_{\pmb{K}}(x,u,\tau)=x^{\top}Q_\tau x + u^{\top}R_\tau u + \mathbb{E}_{w_{\tau}}\left[ V_{\pmb{K}}(Ax+Bu+w_{\tau},\tau +1)\right],
\end{equation*}
and the advantage function
\begin{equation*}
    A_{\pmb{K}}(x,u,\tau) = Q_{\pmb{K}}(x,u,\tau)-V_{\pmb{K}}(x,\tau).
\end{equation*}
Note that $C(\pmb{K}) = \mathbb{E}_{x_0\sim \mathcal{D}}[V(x_0,0)]$. Then we can write the difference of value functions between $\pmb{K}$ and $\pmb{K}^{\prime}$ in terms of advantage functions.
\begin{Lemma}\label{lemma 10}
Assume $\pmb{K}$ and $\pmb{K}'$ have finite costs. Denote $\{x_t'\}_{t=0}^T$ and $\{u_t'\}_{t=0}^{T-1}$ as the state and control sequences of a single trajectory generated by $\pmb{K}'$ starting from $x_0^\prime=x_0=x$, then
\begin{equation}\label{lemma 10 equ 1}
    V_{\pmb{K}'}(x,0)-V_{\pmb{K}}(x,0) = \mathbb{E}_{\pmb{w}}\left[\sum_{t=0}^{T-1} A_{\pmb{K}}(x_t',u_t',t)\right],
\end{equation}
and
$A_{\pmb{K}}(x,-K_\tau'x,\tau)\quad= 2x^{\top}(K_\tau'-K_\tau)^{\top}E_\tau x +x^{\top}(K_\tau'-K_\tau)^{\top}(R_\tau+B^{\top}P_{\tau +1}B)(K_\tau'-K_\tau)x,$
where $E_{\tau}$ is defined in \eqref{eq:barEt}.
\end{Lemma}

\begin{proof} Denote by $c_t'(x)$ the cost generated by $\pmb{K}'$ with a single trajectory starting from $x'_0=x_0=x$. That is, $c_t^{\prime}(x) = (x_t^{\prime})^\top Q_tx_t^{\prime} + (u_t^{\prime})^\top R_tu_t^{\prime},\ t = 0,1,\cdots,T-1,$ and
$c_T^{\prime}(x) = (x_T^{\prime})^\top Q_Tx_T^{\prime},$
with $u_{t}' = - K'_t x'_t,\,\,\,\, x_{t+1}' = Ax_t' +Bu_t' +w_t,\,\,\,\, x_0^\prime=x.$

\noindent Therefore,
\begin{equation*}
\begin{split}
     V_{\pmb{K}'}(x,0)-V_{\pmb{K}}(x,0) & = \mathbb{E}_{\pmb{w}}\left[\sum_{t=0}^{T} c_t'(x)\right]- V_{\pmb{K}}(x,0) = \mathbb{E}_{\pmb{w}}\left[\sum_{t=0}^{T} \left(c_t'(x)+V_{\pmb{K}}(x_t',t)-V_{\pmb{K}}(x_t',t)\right)\right]- V_{\pmb{K}}(x,0)\\
     & = \mathbb{E}_{\pmb{w}}\left[\sum_{t=0}^{T-1}\left( c_t'(x)+V_{\pmb{K}}(x_{t+1}',t+1)-V_{\pmb{K}}(x_t',t)\right)\right]\\ &=\mathbb{E}_{\pmb{w}}\left.\left[ \sum_{t=0}^{T-1}\left( Q_{\pmb{K}}(x_t',u_t',t)-V_{\pmb{K}}(x_t',t)\right)\right|x_0=x\right]
      = \mathbb{E}_{\pmb{w}}\left.\left[\sum_{t=0}^{T-1} A_{\pmb{K}}(x_t',u_t',t)\right|x_0=x\right],
\end{split}
\end{equation*}
where the third equality holds since $c_T^{\prime}(x)=V_{\pmb{K}}(x_T^{\prime},T)$ with the same single trajectory.
For $u=-K_\tau'x$,
\begin{equation}\label{eq:At}
\begin{split}
    A_{\pmb{K}}(x,-K_\tau'x,\tau) & = Q_{\pmb{K}}(x,-K_\tau'x,\tau )-V_{\pmb{K}}(x,\tau) \\
    & = x^{\top}(Q_\tau+(K_\tau')^{\top}R_\tau K_\tau')x + \mathbb{E}_{w_\tau} \left[V_{\pmb{K}}((A-BK_\tau')x+w_{\tau},\tau +1)\right]- V_{\pmb{K}}(x,\tau) \\
    & = x^{\top}(Q_\tau+(K_\tau')^{\top}R_\tau K_\tau')x + \left(x^{\top}(A-BK_\tau')^{\top}P_{\tau+1}(A-BK_\tau')x +\Tr({WP_{\tau+1}})+{ L_{\tau+1}}\right) \\
    &\quad - \left( x^{\top}P_{\tau}x + { L_{\tau}}\right)\\
    & = x^{\top}(Q_\tau+(K_\tau'-K_\tau+K_\tau)^{\top}R_\tau(K_\tau'-K_\tau+K_\tau))x \\
    & \quad + x^{\top}(A-BK_\tau-B(K_\tau'-K_\tau))^{\top}P_{\tau +1}(A-BK_\tau-B(K_\tau'-K_\tau))x \\
    & \quad - x^{\top}(Q_\tau+K_\tau^{\top}R_\tau K_\tau+(A-BK_\tau)^{\top}P_{\tau+1}(A-BK_\tau))x \\
    & = 2x^{\top}(K_\tau'-K_\tau)^{\top}((R_\tau+B^{\top}P_{\tau +1}B)K_\tau-B^{\top}P_{\tau +1}A)x \\ 
    & \quad + x^{\top}(K_\tau'-K_\tau)^{\top}(R_\tau+B^{\top}P_{\tau +1}B)(K_\tau'-K_\tau)x.
\end{split}
\end{equation}
\end{proof}

\begin{proof}[{\bf Proof of Lemma \ref{lemma 11}}]
First for any $K_{\tau}^{\prime}$, from \eqref{eq:At},
\begin{equation}\label{Lemma 11 proof equ 1}
\begin{split}
    A_{\pmb{K}}(x,-K_\tau'x,\tau) & = Q_{\pmb{K}}(x,-K_\tau'x,\tau )-V_{\pmb{K}}(x,\tau) \\
    & = 2\Tr(xx^{\top}(K_\tau'-K_\tau)^{\top}E_\tau) + \Tr(xx^{\top}(K_\tau'-K_\tau)^{\top}(R_\tau+B^{\top}P_{\tau +1}B)(K_\tau'-K_\tau))\\
    & = \Tr\big(xx^{\top}(K_\tau'-K_\tau+(R_\tau+B^{\top}P_{\tau+1}B)^{-1}E_\tau)^{\top}(R_\tau+B^{\top}P_{\tau +1}B)\\
    & \quad (K_\tau'-K_\tau+(R_\tau+B^{\top}P_{\tau +1}B)^{-1}E_\tau)\big) -\Tr(xx^{\top}E_\tau^{\top}(R_\tau+B^{\top}P_{\tau +1}B)^{-1}E_\tau)\\
    & \geq -\Tr(xx^{\top}E_\tau^{\top}(R_\tau+B^{\top}P_{\tau +1}B)^{-1}E_\tau),
\end{split}
\end{equation}
with equality holds when $K_\tau'=K_\tau-(R_\tau + B^{\top}P_{\tau +1}B)^{-1}E_\tau$.
Then,
\begin{eqnarray*}
 C(\pmb{K})-C(\pmb{K}^*) & =& -\mathbb{E}\sum_{t=0}^{T-1}A_{\pmb{K}}(x_t^*,u_t^*,t)  \leq \mathbb{E}\sum_{t=0}^{T-1}\Tr \left(x_t^*(x_t^*)^{\top}E_t^{\top}(R_t+B^{\top}P_{t+1} B)^{-1}E_t\right)\\
    & \leq &\|\Sigma_{\pmb{K}^*}\|\sum_{t=0}^{T-1}\Tr(E_t^{\top}(R_t+B^{\top}P_{t+1}B)^{-1}E_t) \leq \frac{\|\Sigma_{\pmb{K}^*}\|}{\sr}\sum_{t=0}^{T-1}\Tr(E_t^{\top}E_t)\\
    & \leq& \frac{\|\Sigma_{\pmb{K}^*}\|}{4\sx^2\sr}\sum_{t=0}^{T-1}\Tr(\nabla_t C(\pmb{K})^{\top}\nabla_t C(\pmb{K})),
\end{eqnarray*}
where $\sx$ is defined in \eqref{Defn of barmu} and $\sr$ is defined in \eqref{Defn of barSigmaR}.
For the lower bound, consider $K_t'=K_t-(R_t + B^{\top}P_{t +1}B)^{-1}E_t$ where the equality holds in \eqref{Lemma 11 proof equ 1}. Using  $C(\pmb{K}^*)
\leq C(\pmb{K}^\prime)$
\begin{equation}\label{Lemma 11 proof equ 3}
\begin{split}
    C(\pmb{K})-C(\pmb{K}^*) & \geq C(\pmb{K})-C(\pmb{K'}) = -\mathbb{E}\sum_{t=0}^{T-1}A_{\pmb{K}}(x_t',u_t',t) = \mathbb{E}\sum_{t=0}^{T-1}\Tr(x_t'(x_t')^{\top}E_t^{\top}(R_t+B^{\top}P_{t+1} B)^{-1}E_t)\\
    & \geq \sx \sum_{t=0}^{T-1} \frac{1}{\|R_t+B^{\top}P_{t+1} B\|}\Tr(E_t^{\top}E_t)
\end{split}
\end{equation}
\end{proof}

\begin{proof}[{\bf Proof of Lemma \ref{lemma 12}}]
By lemma \ref{lemma 10} we have
\begin{equation*}
\begin{split}
    C(\pmb{K}')-C(\pmb{K}) & = \mathbb{E}\left[\sum_{t=0}^{T-1} A_{\pmb{K}}(x_t',-K_t^{\prime}x_t',t)\right]\\
    & = \sum_{t=0}^{T-1} \left(2\Tr(\Sigma_t'(K_t'-K_t)^{\top}E_t) + \Tr(\Sigma_t'(K_t'-K_t)^{\top}(R_t+B^{\top}P_{t +1}B)(K_t'-K_t))\right).
\end{split}
\end{equation*}
\end{proof}

\begin{proof}[{\bf Proof of Lemma \ref{lemma 13}}]
For $t=0,1,\cdots,T$,
\begin{equation*}
    C(\pmb{K})\geq\mathbb{E}[x_t^{\top}P_tx_t] \geq \|P_t\|\sigma_{\min}(\mathbb{E}[x_tx_t^{\top}])\geq \sx\|P_t\|,
\end{equation*}
\begin{equation*}
    C(\pmb{K})=\sum_{t=0}^{T-1}\Tr(\mathbb{E}[x_tx_t^{\top}](Q_t+K_t^{\top}R_tK_t))+\Tr(\mathbb{E}[x_Tx_T^{\top}]Q_T)\geq \sq\Tr(\Sigma_{\pmb{K}})\geq \sq\|\Sigma_{\pmb{K}}\| .
\end{equation*}
Therefore the statement in Lemma \ref{lemma 13} follows provided that $\sx>0$ and Assumption \ref{ass:parameters} holds.
\end{proof}

\begin{proof} [{\bf Proof of Proposition \ref{prop:Sigma_Gamma_relation}}]
Recall that $ \Sigma_{t} =  \mathbb{E}\left[x_t x_t^{\top}\right]$.
\noindent Note that 
\begin{eqnarray*}
\Sigma_1 &=& \mathbb{E}\left[x_{1}x_{1}^\top\right] = \mathbb{E}\left[\left((A-B\,K_0) x_0+w_0\right)\left((A-B\,K_0)x_0+w_0\right)^\top\right]\\ \nonumber
&=& (A-B\,K_0)\Sigma_0 \left(A-B\,K_0\right)^\top +W = \mathcal{G}_0(\Sigma_0) +W.
\end{eqnarray*}

\noindent Now we first prove that
\begin{eqnarray}\label{eqn:intermedidate1}
\Sigma_t = \mathcal{G}_{t-1}(\Sigma_0) + \sum_{s=1}^{t-1} D_{t-1,s} W D_{t-1,s}^{\top} +W,\ \forall\, t=2,3,\cdots,T.
\end{eqnarray}
\noindent When $t=2$,
\begin{eqnarray*}
\Sigma_2 &=& \mathbb{E}\left[x_{2}x_{2}^\top\right] = \mathbb{E}\left[\left((A-B\,K_1) x_1+w_1\right)\left((A-B\,K_1)x_1+w_1\right)^\top\right]\\
&=& (A-B\,K_1)\Sigma_1 \left(A-B\,K_1\right)^\top +W = \mathcal{G}_1(\Sigma_0) + (A-BK_1)W(A-BK_1)^\top + W,
\end{eqnarray*}
which satisfies \eqref{eqn:intermedidate1}. Assume \eqref{eqn:intermedidate1} holds for $t\leq k$. Then for $t =  k+1$,
\begin{eqnarray*}
\mathbb{E}\left[x_{t+1}x_{t+1}^\top\right] &= &\mathbb{E}\left[\left((A-B\,K_t) x_t+w_t\right)\left((A-B\,K_t)x_t+w_t\right)^\top\right]\\
&=& (A-B\,K_t)\Sigma_t \left(A-B\,K_t\right)^\top +W
=  \mathcal{G}_{t}(\Sigma_0) + \sum_{s=1}^{t} D_{t,s} W D_{t,s}^{\top} +W.
\end{eqnarray*}
Therefore \eqref{eqn:intermedidate1} holds,  $\forall\,t=1,2,\cdots,T$. Finally,
\[
\Sigma_{\pmb{K}} = \sum_{t=0}^T \Sigma_t = \Sigma_0
+\sum_{t=0}^{T-1}\mathcal{G}_t(\Sigma_0) +\sum_{t=1}^{T-1}\sum_{s=1}^t D_{t,s}WD_{t,s}^{\top} + TW = \mathcal{T}_{\pmb{K}}(\Sigma_0)+\Delta(\pmb{K},W).
\]

\end{proof}

\subsection{Proofs in Section \ref{sec:model_based_perturbation}}\label{appendix:missing_proofs_sec3_2}

\begin{proof}[{\bf Proof of Lemma \ref{Lemma 20}}]
By direct calculation,
\begin{eqnarray}\label{eqn:G_t_bound}
\|\mathcal{G}_t\| \leq \rho^{2(t+1)}, \quad \mbox{and} \quad\|\mathcal{G}^{\prime}_t\| \leq \rho^{2(t+1)}.
\end{eqnarray}

\noindent Denote $\mathcal{F}_t=\mathcal{F}_{K_t}$ and $\mathcal{F}_t^{\prime} = \mathcal{F}_{K^{\prime}_t}$ to ease the exposition. Then for any symmetric matrix $\Sigma\in\mathbb{R}^{d\times d}$ and $t \ge 0$,
\begin{eqnarray*}
\|(\mathcal{G}_{t+1}^{\prime}-\mathcal{G}_{t+1})(\Sigma)\| &=& \|\mathcal{F}_{t+1}^{\prime}\circ\mathcal{G}_t^{\prime}(\Sigma)-\mathcal{F}_{t+1}\circ\mathcal{G}_t(\Sigma)\|\\
&=& \|\mathcal{F}_{t+1}^{\prime}\circ\mathcal{G}_{t}^{\prime}(\Sigma)-\mathcal{F}_{t+1}^{\prime}\circ\mathcal{G}_{t}(\Sigma)+\mathcal{F}_{t+1}^{\prime}\circ\mathcal{G}_{t}(\Sigma)-\mathcal{F}_{t+1}\circ\mathcal{G}_{t}(\Sigma)\| \\
&\leq &\|\mathcal{F}_{t+1}^{\prime}\circ\mathcal{G}_{t}^{\prime}(\Sigma)-\mathcal{F}_{t+1}^{\prime}\circ\mathcal{G}_{t}(\Sigma)\|+\|\mathcal{F}_{t+1}^{\prime}\circ\mathcal{G}_{t}(\Sigma)-\mathcal{F}_{t+1}\circ\mathcal{G}_{t}(\Sigma)\| \\
& =  & \|\mathcal{F}_{t+1}^{\prime}\circ(\mathcal{G}_{t}^{\prime}-\mathcal{G}_{t})(\Sigma)\|+\|(\mathcal{F}_{t+1}^{\prime}-\mathcal{F}_{t+1})\circ\mathcal{G}_{t}(\Sigma)\|\\
&\leq & \|\mathcal{F}_{t+1}^{\prime}\| \,\|(\mathcal{G}_{t}^{\prime}-\mathcal{G}_{t})(\Sigma)\|+\|\mathcal{G}_{t}\| \,\|\mathcal{F}_{t+1}^{\prime}-\mathcal{F}_{t+1}\|\,\|\Sigma\|\\
&\leq& \rho^2\|(\mathcal{G}_{t}^{\prime}-\mathcal{G}_{t})(\Sigma)\| +\rho^{2(t+1)}\|\mathcal{F}_{t+1}^{\prime}-\mathcal{F}_{t+1}\| \|\Sigma\|.
\end{eqnarray*}
Therefore,
\begin{eqnarray}\label{tn1}
\|(\mathcal{G}_{t+1}^{\prime}-\mathcal{G}_{t+1})(\Sigma)\|
\leq \rho^2\|(\mathcal{G}_{t}^{\prime}-\mathcal{G}_{t})(\Sigma)\| +\rho^{2(t+1)}\|\mathcal{F}_{t+1}^{\prime}-\mathcal{F}_{t+1}\| \|\Sigma\|.
\end{eqnarray}
Summing \eqref{tn1} up for $t \in \{1,2,\cdots,T-2\}$ with $\|\mathcal{G}^{\prime}_0-\mathcal{G}_0\|=\|\mathcal{F}_0^{\prime}-\mathcal{F}_0\|$, we have
\[
\sum_{t=0}^{T-1}\Big\|(\mathcal{G}_{t}-\mathcal{G}_{t}^{\prime})(\Sigma)\Big\| \leq \frac{\rho^{2T}-1}{\rho^2-1} \Big( \sum_{t=0}^{T-1}\|\mathcal{F}_{t}-\mathcal{F}^{\prime}_{t}\|\Big)\|\Sigma\|.
\]
\end{proof}

\subsection{Proofs in Section \ref{sec:model_based_convergence}}
\label{appendix:model_based_convergence}

\begin{proof}[{\bf Proof of Lemma \ref{lemma 24}}]
Given \eqref{eq:gd} and condition \eqref{stepsize condition}, we have
$\|K_t'-K_t\| = \eta \|\nabla_{t} C(\pmb{K})\| \leq \frac{\sq\sx}{2C(\pmb{K})\|B\|}.$
Therefore, 
\begin{equation}\label{eqn:BKdiff_bd}
  \|B\|\|K_t'-K_t\| \leq \frac{\sq\sx}{2C(\pmb{K})}
\leq \frac{1}{2}.  
\end{equation}
The last inequality holds since $\sx\leq \frac{C(\pmb{K})}{\sq}$ given by Lemma \ref{lemma 13}.
Therefore, by Lemma \ref{lemma 19},
\begin{equation}\label{F_Kt bound}
    \begin{split}
        \sum_{t=0}^{T-1} \|\mathcal{F}_{K_t}-\mathcal{F}_{K^{\prime}_t}\| 
        & \leq  (2\rho+1)\|B\|\left(\sum_{t=0}^{T-1}\|K_t-K_t^{\prime}\|\right).
    \end{split}
\end{equation}
By Lemmas \ref{lemma 1} and \ref{lemma 12},
\begin{equation}\label{Lemma 24 proof eqn 1}
\begin{split}
     C(\pmb{K}')-C(\pmb{K}) & = \sum_{t=0}^{T-1} \Big[2\Tr\Big(\Sigma_t'(K_t'-K_t)^{\top}E_t\Big) + \Tr\Big(\Sigma_t'(K_t'-K_t)^{\top}(R_t+B^{\top}P_{t +1}B)(K_t'-K_t)\Big)\Big]\\
     & = \sum_{t=0}^{T-1} \Big[-4\eta\Tr\Big(\Sigma_t'\Sigma_tE_t^{\top}E_t\Big) + 4\eta^2\Tr\Big(\Sigma_t'\Sigma_tE_t^{\top}(R_t+B^{\top}P_{t +1}B)E_t\Sigma_t\Big)\Big]\\
     & = \sum_{t=0}^{T-1} \Big[-4\eta\Tr\Big((\Sigma_t'-\Sigma_t+\Sigma_t)\Sigma_tE_t^{\top}E_t\Big) + 4\eta^2\Tr\Big(\Sigma_t'\Sigma_tE_t^{\top}(R_t+B^{\top}P_{t +1}B)E_t\Sigma_t\Big)\Big]\\
     & \leq \sum_{t=0}^{T-1} \Big[-4\eta\Tr\Big(\Sigma_tE_t^{\top}E_t\Sigma_t\Big) + 4\eta\Tr((\Sigma_t'-\Sigma_t)\Sigma_tE_t^{\top}E_t\Sigma_t\Sigma_t^{-1}) \\
     &\qquad + 4\eta^2\Tr\Big(\Sigma_t'\Sigma_tE_t^{\top}(R_t+B^{\top}P_{t +1}B)E_t\Sigma_t\Big)\Big]\\
     & \leq \sum_{t=0}^{T-1} \Big[-4\eta\Tr\Big(\Sigma_tE_t^{\top}E_t\Sigma_t\Big) + 4\eta\frac{\|\Sigma_t'-\Sigma_t\|}{\sigma_{\min}(\Sigma_t)}\Tr\Big(\Sigma_tE_t^{\top}E_t\Sigma_t\Big)\\
     & \qquad \qquad + 4\eta^2\|\Sigma_t'(R_t+B^{\top}P_{t +1}B) \|\Tr\Big(\Sigma_tE_t^{\top}E_t\Sigma_t\Big)\Big]\\
     &  \leq -\eta\Big(1 - \frac{\sum_{t=0}^{T-1} \|\Sigma_t^{\prime}-\Sigma_t\|}{\sx}
     - \eta\|\Sigma_{\pmb{K}^{\prime}}\|\sum_{t=0}^{T-1}\|R_t+B^{\top}P_{t +1}B \|\Big) \sum_{t=0}^{T-1} \Big[\Tr(\nabla_tC(\pmb{K})^{\top} \nabla_tC(\pmb{K}))\Big].\\
\end{split}
\end{equation}

By Lemma \ref{lemma 11}, we have
\begin{equation}\label{Lemma 24 proof eqn 2}
\begin{split}
     C(\pmb{K}')-C(\pmb{K}) 
     & \leq -\eta\Big(1 - \frac{\sum_{t=0}^{T-1} \|\Sigma_t^{\prime}-\Sigma_t\|}{\sx}
     - \eta\|\Sigma_{\pmb{K}^{\prime}}\|\sum_{t=0}^{T-1}\|R_t+B^{\top}P_{t +1}B \|\Big)
     \Big(\frac{4\sx^2\sr}{\|\Sigma_{\pmb{K}^*}\|}\Big)\Big(C(\pmb{K})-C(\pmb{K}^*)\Big)
\end{split}
\end{equation}
provided that
\begin{equation}\label{Lemma 24 proof eqn 3}
\begin{split}
     1 - \frac{\sum_{t=0}^{T-1} \|\Sigma_t^{\prime}-\Sigma_t\|}{\sx}
     - \eta\|\Sigma_{\pmb{K}^{\prime}}\|\sum_{t=0}^{T-1}\|R_t+B^{\top}P_{t +1}B \| > 0.
\end{split}
\end{equation}

\noindent By \eqref{eq:Colla_intem}, \eqref{eq:gd}, and \eqref{eqn:BKdiff_bd}, 
\begin{eqnarray*}
\sum_{t=0}^{T-1} \|\Sigma_t^{\prime}-\Sigma_t\| \leq \frac{ \rho^{2T}-1}{\rho^2-1} \left(\frac{C(\pmb{K})}{\sq}+T\|W\|\right)\left(\eta(2\rho+1)\|B\|\sum_{t=0}^{T-1}\|\nabla_t C(\pmb{K})\|\right).
\end{eqnarray*}
Given the step size condition in \eqref{stepsize condition}, we have
\begin{equation}\label{eq:stepsize int}
    \eta(2\rho+1)\|B\|\sum_{t=0}^{T-1}\|\nabla_tC(\pmb{K})\| \leq \eta(2\rho+1)\|B\|\Big(T\cdot\max_t\{\|\nabla_tC(\pmb{K})\|\}\Big) \leq \frac{(\rho^2-1)\sq\sx}{2(\rho^{2T}-1)(C(\pmb{K})+\sq T\|W\|)}.
\end{equation}
Then, by Corollary \ref{Corr 20} and \eqref{F_Kt bound},
\begin{eqnarray*}
\frac{\|\Sigma_{\pmb{K}^{\prime}}-\Sigma_{\pmb{K}}\|}{\sx} 
&\leq&  \frac{\rho^{2T}-1}{\rho^2-1} \Big( \sum_{t=0}^{T-1}\|\mathcal{F}_{K_t}-\mathcal{F}_{K^{\prime}_t}\|\Big)\frac{\|\Sigma_0\|+T\|W\|}{\sx} \\ 
&\leq& \frac{\rho^{2T}-1}{\rho^2-1}(2\rho+1)\|B\|\left(\sum_{t=0}^{T-1}\eta\|\nabla_tC(\pmb{K}\|\right)\frac{C(\pmb{K})+\sq T \|W\|}{\sq\sx} \leq \frac{1}{2},
\end{eqnarray*}
where the last step holds by \eqref{eq:stepsize int}.
Therefore, the bound of $\|\Sigma_{\pmb{K}^\prime}\|$ in \eqref{Lemma 24 proof eqn 3} is given by
\begin{equation}\label{eq:SigmaKp_bound}
    \|\Sigma_{\pmb{K}^\prime}\|\leq\|\Sigma_{\pmb{K}^{\prime}}-\Sigma_{\pmb{K}}\| + \|\Sigma_{\pmb{K}}\|\leq \frac{1}{2}\sx + \frac{C(\pmb{K})}{\sq} \leq \frac{1}{2}\|\Sigma_{\pmb{K}^\prime}\| + \frac{C(\pmb{K})}{\sq},
\end{equation}
which indicates that $\|\Sigma_{\pmb{K}^\prime}\|\leq \frac{2C(\pmb{K})}{\sq}$.
Therefore, \eqref{Lemma 24 proof eqn 3} gives
\begin{equation*}
    \begin{split}
     & 1 - \frac{\sum_{t=0}^{T-1} \|\Sigma_t^{\prime}-\Sigma_t\|}{\sx}
     - \eta\|\Sigma_{\pmb{K}^{\prime}}\|\sum_{t=0}^{T-1}\|R_t+B^{\top}P_{t +1}B \|  \\
     & \geq 1- \frac{ (\rho^{2T}-1)}{(\rho^2-1)\sx} \left(\frac{C(\pmb{K})}{\sq}+T\|W\|\right)\left({\eta(2\rho+1)\|B\|\sum_{t=0}^{T-1}\|\nabla_t C(\pmb{K})\|}\right)  - \eta\frac{2C(\pmb{K})}{\sq}\sum_{t=0}^{T-1}\|R_t+B^{\top}P_{t +1}B \|  \\
    & = 1 - C_1\eta,
    \end{split}
\end{equation*}
where $C_1$ is defined in \eqref{eq:C1}.
So if $\eta \leq \frac{1}{2C_1}$, then,
\begin{equation*}
   1 - \frac{\sum_{t=0}^{T-1} \|\Sigma_t^{\prime}-\Sigma_t\|}{\sx}
     - \eta\|\Sigma_{\pmb{K}^{\prime}}\|\sum_{t=0}^{T-1}\|R_t+B^{\top}P_{t +1}B \| \geq 1-C_1\eta\geq \frac{1}{2} >0.
\end{equation*}
Hence,$
     C(\pmb{K}^{\prime})-C(\pmb{K}) \leq -\frac{\eta}{2}
     \Big(\frac{4\sx^2\sr}{\|\Sigma_{\pmb{K}^*}\|}\Big)\Big(C(\pmb{K})-C(\pmb{K}^*)\Big),$
and 
\begin{equation*}
     C(\pmb{K}^{\prime})-C(\pmb{K}^*)  =  \left(C(\pmb{K}^{\prime})-C(\pmb{K})\right)+\left(C(\pmb{K})-C(\pmb{K}^*)\right) \leq 
     \Big(1-2\eta\frac{\sx^2\sr}{\|\Sigma_{\pmb{K}^*}\|}\Big)\Big(C(\pmb{K})-C(\pmb{K}^*)\Big).
\end{equation*}
\end{proof}

\subsection{Proofs in Section \ref{sc:single_agent_model_free}}
\label{appendix:missing_proofs_sec4}
{Before proceeding to the proof of Theorem \ref{thm:projected_GD}, we show two important Lemmas which provide the intermediate steps. We first show the optimality condition for the projection operator in Lemma \ref{lemma:optimality_projection}. We then show the one-step convergence result in Lemma \ref{lemma 24 projection}.

\begin{Lemma}[Optimality Condition]\label{lemma:optimality_projection} Fix a policy matrix $\pmb{L}^1$ and write $\pmb{L}^* = \Pi_{\mathcal{S}}(\pmb{L}^1)$. Then for any $\pmb{L}^0\in \mathcal{S}$, we have
\begin{eqnarray}\label{eq:optimality_condition}
\sum_{t=0}^{T-1}\Tr \left( (L^0_t - L^*_t)(L^*_t - L^1_t)^{\top}\right) \ge 0.
\end{eqnarray}
\end{Lemma}

\begin{proof}[Proof of Lemma \ref{lemma:optimality_projection}]
We show condition \eqref{eq:optimality_condition} by contradiction. Assume condition \eqref{eq:optimality_condition} does not hold, then there exist some $\pmb{L}^3 \in \mathcal{S}$ and some constant $b>0$ such that 
\begin{eqnarray}\label{eq:optimality_condition_contra}
\sum_{t=0}^{T-1}\Tr \left( (L^3_t - L^*_t)(L^*_t - L^1_t)^{\top}\right)  = -b <0.
\end{eqnarray}
Let $$M =1 + \frac{\sum_{t=0}^{T-1} \Tr \left((L^3_t-L_t^*)(L^3_t-L_t^*)^{\top}\right)}{-2\sum_{t=0}^{T-1} \Tr \left( (L^3_t-L_t^*)(L_t^*-L_t^1)^{\top} \right)} =1 + \frac{\sum_{t=0}^{T-1} \Tr \left((L^3_t-L_t^*)(L^3_t-L_t^*)^{\top}\right)}{2b} >1,$$ and take 
\[
\overline{\pmb{L}} = \frac{1}{M} \pmb{L}^3 + \left(1-\frac{1}{M}\right) \pmb{L}^*.
\]
By the convexity of $\mathcal{S}$, we have $\overline{\pmb{L}} \in \mathcal{S}$. Hence from definition {\eqref{eqn:projection}}, we have
\begin{eqnarray}\label{eq:mediate3}
\sum_{t=0}^{T-1} \Tr \left((L_t^*-L_t^1)(L_t^*-L_t^1)^{\top}\right) \leq \sum_{t=0}^{T-1} \Tr \left((\overline{L}_t-L_t^1)(\overline{L}_t-L_t^1)^{\top}\right).
\end{eqnarray}
On the other hand,
\begin{eqnarray}
&&\sum_{t=0}^{T-1} \Tr \left((\overline{L}_t-L_t^1)(\overline{L}_t-L_t^1)^{\top}\right)\nonumber\\
&=&\sum_{t=0}^{T-1} \Tr \left(\left(\frac{1}{M}(L^3_t-L^*_t)+(L_t^*-L_t^1)\right)\left(\frac{1}{M}(L^3_t-L^*_t)+(L_t^*-L_t^1)\right)^{\top}\right)\nonumber\\
&=&\sum_{t=0}^{T-1} \frac{1}{M^2}\Tr \left((L^3_t-L_t^*)(L^3_t-L_t^*)^{\top}\right)+\sum_{t=0}^{T-1} \Tr \left((L_t^*-L_t^1)(L_t^*-L_t^1)^{\top}\right) \nonumber\\
&&\qquad \qquad + \sum_{t=0}^{T-1} \frac{2}{M}\Tr \left( (L^3_t-L_t^*)(L_t^*-L_t^1)^{\top} \right).
\label{eq:mediate1}
\end{eqnarray}
By the definition of $M$ we have, 
\begin{eqnarray*}
&&\sum_{t=0}^{T-1} \Tr \left((L^3_t-L_t^*)(L^3_t-L_t^*)^{\top}\right)+ 2M\sum_{t=0}^{T-1} \Tr \left( (L^3_t-L_t^*)(L_t^*-L_t^1)^{\top} \right)\\
&=& \sum_{t=0}^{T-1} \Tr \left((L^3_t-L_t^*)(L^3_t-L_t^*)^{\top}\right)- 2Mb \\
&=& \sum_{t=0}^{T-1} \Tr \left((L^3_t-L_t^*)(L^3_t-L_t^*)^{\top}\right)- 2b - \sum_{t=0}^{T-1} \Tr \left((L^3_t-L_t^*)(L^3_t-L_t^*)^{\top}\right)= -2b <0.
\end{eqnarray*}
Thus substituting this in \eqref{eq:mediate1} contradicts \eqref{eq:mediate3} which completes the proof.


\end{proof}

\begin{Lemma}\label{lemma 24 projection}
Assume Assumption 2.1 holds,  { $\sx>0$}, {$\pmb{K}\in \mathcal{S}$} and that 
\begin{equation}\label{eq:gd projection}
    K_t^{\prime}=K_t-\eta\nabla_{t} C(\pmb{K}), \qquad {\rm where }
\end{equation}
\begin{equation}\label{stepsize condition projection}
    \eta \leq \min\left\{C_1,C_2\right\}, \qquad {\rm with}
\end{equation}
\begin{equation}\label{eq:C1 projection}
    C_1=\frac{(\rho^2-1)\sq\sx}{4d\,T^2\,{\sqrt{d+k}}(\rho^{2T}-1)(2\rho+1)(C(\pmb{K})+\sq T\|W\|)\|B\|\max_{t}\{\|\nabla_tC(\pmb{K})\|\}}
\end{equation}
\begin{equation}\label{eq:C2}
    C_2 = \frac{\sq}{8C(\pmb{K})\sum_{t=0}^{T-1}\|R_t+B^{\top}P_{t +1}B \|}.
\end{equation}
Take
\begin{equation}\label{eq:projection}
    \widetilde{\pmb{K}} = \Pi_{\mathcal{S}}\left( \pmb{K}^{\prime}\right),
\end{equation}
with $\pmb{K}^{\prime} = (K_0^{\prime},\cdots,K_{T-1}^{\prime})$ and $K_t^{\prime}$ defined in \eqref{eq:gd projection} $(t=0,1,\cdots,T-1)$.
Then we have
\begin{equation*}
   2\eta\sum_{t=0}^T \|G_t(\pmb{K})\|_F^2 =  2\eta\sum_{t=0}^T \Tr \left(G_t(\pmb{K})^{\top}G_t(\pmb{K})\right) \leq C(\pmb{K})-C(\widetilde{\pmb{K}}).
\end{equation*}
\end{Lemma}
\begin{proof}
By definition of $ \widetilde{\pmb{K}}$, and as $\pmb{K}\in \mathcal{S}$, we have
\begin{eqnarray}\label{eq:projection_condition}
\sum_{t=0}^{T-1}\Tr \left((\widetilde{{K}}_t -{K}_t^{\prime} ) (\widetilde{{K}}_t - {K}^{\prime}_t )^\top\right)\leq \sum_{t=0}^{T-1}\Tr \left(({{K}}_t -{K}_t^{\prime} ) ({{K}_t} -{K}_t^{\prime} )^\top\right).
\end{eqnarray}
Take $\pmb{L}^1=\pmb{K}^{\prime}$ and $\pmb{L}^0 = \pmb{K}$ in Lemma
\ref{lemma:optimality_projection}, we have
\begin{eqnarray}\label{eq:optimality_condition3}
\sum_{t=0}^{T-1}\Tr \left( (K_t - \widetilde{K}_t)(\widetilde{K}_t - K^{\prime}_t)^{\top}\right) \ge 0.
\end{eqnarray}
Combining \eqref{eq:projection_condition} and \eqref{eq:optimality_condition3} leads to
\begin{eqnarray}\label{eq:projection_condition2}
\sum_{t=0}^{T-1}\Tr \left(({{K}}_t -{K}_t^{\prime} ) ( {K}_t-\widetilde{{K}}_t  )^\top\right)\geq \sum_{t=0}^{T-1}\Tr \left(({{K}}_t -\widetilde{{K}}_t  ) ({{K_t}} -\widetilde{{K}}_t  )^\top\right).
\end{eqnarray}
Given the definition \eqref{eq:projection}, we have $G(\pmb{K}) = \frac{\widetilde{\pmb{K}}-\pmb{K}}{2\eta}$ and $G_t(\pmb{K}) = \frac{\widetilde{K}_t - K_t}{2\eta}$. By Lemmas \ref{lemma 1} and \ref{lemma 12},
\begin{equation}\label{Lemma 24 projection proof eqn 1}
\begin{split}
     C(\widetilde{\pmb{K}})-C(\pmb{K}) & = \sum_{t=0}^{T-1} \Big[2\Tr\Big(\widetilde{\Sigma}_t(\widetilde{K}_t-K_t)^{\top}E_t\Big) + \Tr\Big(\widetilde{\Sigma}_t(\widetilde{K}_t-K_t)^{\top}(R_t+B^{\top}P_{t +1}B)(\widetilde{K}_t-K_t)\Big)\Big]\\
      & = \sum_{t=0}^{T-1} \Big[4\eta\Tr\Big((\widetilde{\Sigma}_t-\Sigma_t)(G_t(\pmb{K}))^{\top}E_t\Big) +4\eta\Tr\Big(\Sigma_t(G_t(\pmb{K}))^{\top}E_t\Big)\\
      &\qquad + 4\eta^2\Tr\Big(\widetilde{\Sigma}_t (G_t(\pmb{K}))^{\top}(R_t+B^{\top}P_{t +1}B)G_t(\pmb{K})\Big)\Big],\\
\end{split}
\end{equation}
with $\widetilde{\Sigma}_t :=\mathbb{E}[\widetilde{x}_t\widetilde{x}_t^{\top}]$ and $\{\widetilde{x}_t\}_{t=0}^{T-1}$ is the trajectory under policy $\widetilde{\pmb{K}}$.

First, we have
\begin{eqnarray*}
&&\sum_{t=0}^{T-1}\Tr\Big(\Sigma_t(G_t(\pmb{K}))^{\top}E_t\Big) =\sum_{t=0}^{T-1} \Tr\Big((G_t(\pmb{K}))^{\top}E_t\Sigma_t\Big) = \frac{1}{4\eta^2}\sum_{t=0}^{T-1}\Tr ((\widetilde{K}_t-{K}_t)^{\top}(K_t-K_t^{\prime})) \\
&\leq& -\frac{1}{4\eta^2}\sum_{t=0}^{T-1} \Tr((\widetilde{K}_t-{K}_t)(\widetilde{K}_t-{K}_t)^{\top}) = -\Tr ((G_t(\pmb{K}))^{\top}(G_t(\pmb{K}))) ,
\end{eqnarray*}
in which the last inequality holds by \eqref{eq:projection_condition2}.

Second given \eqref{eq:gd projection} and condition \eqref{eq:C1 projection}, we have
\[
\|K_t'-K_t\| = \eta \|\nabla_{t} C(\pmb{K})\| \leq \frac{\sq\sx}{2\,T{\sqrt{d+k}}C(\pmb{K})\|B\|}.
\]
Therefore, 
\begin{eqnarray*}
\sum_{t=0}^{T-1}\|B\|\|\tilde{K}_t-K_t\| &\leq& \sum_{t=0}^{T-1} \|B\|\|\tilde{K}_t-K_t\|_F \leq  \sum_{t=0}^{T-1}\|B\|\|{K}^{\prime}_t-K_t\|_F\\
&\leq&\sqrt{d+k} \|B\|\sum_{t=0}^{T-1}\|K_t'-K_t\| \leq \frac{\sq\sx}{2C(\pmb{K})}
\leq \frac{1}{2}.
\end{eqnarray*}
The second inequality holds by \eqref{eq:projection_condition2} and the last inequality holds since $\sx\leq \frac{C(\pmb{K})}{\sq}$ given by Lemma  \ref{lemma 13}. 
By \eqref{eq:Colla_intem},
\begin{eqnarray}\label{eqn:diff_Sig_bd}
&&\sum_{t=0}^{T-1} \|\widetilde{\Sigma}_t-\Sigma_t\| 
\leq   {\frac{ \rho^{2T}-1}{\rho^2-1} \left(\frac{C(\pmb{K})}{\sq}+T\|W\|\right)\left(2\rho\,\|B\|\,|||\pmb{K}-\widetilde{\pmb{K}}|||+\|B\|^2\,|||\pmb{K}-\widetilde{\pmb{K}}|||^2\right)}\nonumber
\\
&&\leq \frac{ \rho^{2T}-1}{\rho^2-1} \left(\frac{C(\pmb{K})}{\sq}+T\|W\|\right)\left(2(2\rho+1)\|B\|\sum_{t=0}^{T-1}\eta\|G_t(\pmb{K})\|\right)\nonumber\\
&&\leq \frac{\sx}{2d\,T^2\,\sqrt{d+k}\cdot\max_t\|\nabla_tC(\pmb{K})\|}\sum_{t=0}^{T-1}\|G_t(\pmb{K})\|,
\end{eqnarray}
where the last inequality holds by step size condition $\eta\leq C_1$.
Hence
\begin{eqnarray*}
\eqref{Lemma 24 projection proof eqn 1} &\leq &\sum_{t=0}^{T-1}\left[2\eta \frac{d\|\widetilde{\Sigma}_t-\Sigma_t\|}{\sigma_{\min}(\Sigma_t)} \|G_t(\pmb{K})\|\,\|\nabla_tC(\pmb{K})\|-4\eta \Tr ((G_t(\pmb{K}))^{\top}(G_t(\pmb{K})))\right.\\
&&\left.  + 4\eta^2\|\Sigma_{\widetilde{\pmb{K}}}\|\,\|R_t+B^{\top}P_{t +1}B\|\Tr\Big( (G_t(\pmb{K}))^{\top}G_t(\pmb{K})\Big)\right]\\
&\leq&2\eta\frac{d}{\sx}\Big(\sum_{t=0}^{T-1}\|\widetilde{\Sigma}_t-\Sigma_t\|\Big) \Big(\sum_{t=0}^{T-1}\|G_t(\pmb{K})\|\Big)\Big(\sum_{t=0}^{T-1}\|\nabla_tC(\pmb{K})\|\Big)-4\eta\sum_{t=0}^{T-1} \Tr ((G_t(\pmb{K}))^{\top}(G_t(\pmb{K})))\\
&& + 4\eta^2\sum_{t=0}^{T-1}\|\Sigma_{\widetilde{\pmb{K}}}\|\,\|R_t+B^{\top}P_{t +1}B\|\Tr\Big( (G_t(\pmb{K}))^{\top}G_t(\pmb{K})\Big)\\
&\leq& \frac{\eta}{T}\Big(\sum_{t=0}^{T-1}\|G_t(\pmb{K})\|\Big)^2-4\eta\sum_{t=0}^{T-1} \Tr ((G_t(\pmb{K}))^{\top}(G_t(\pmb{K})))\\
&& + 4\eta^2\sum_{t=0}^{T-1}\|\Sigma_{\widetilde{\pmb{K}}}\|\,\|R_t+B^{\top}P_{t +1}B\|\Tr\Big( (G_t(\pmb{K}))^{\top}G_t(\pmb{K})\Big)\\
&\leq &\sum_{t=0}^{T-1}\left[\eta \|G_t(\pmb{K})\|^2-4\eta \Tr ((G_t(\pmb{K}))^{\top}(G_t(\pmb{K})))\right.\\
&&\left.  + 4\eta^2\|\Sigma_{\widetilde{\pmb{K}}}\|\,\|R_t+B^{\top}P_{t +1}B\|\Tr\Big( (G_t(\pmb{K}))^{\top}G_t(\pmb{K})\Big)\right]\\
&\leq &\sum_{t=0}^{T-1}\eta\left[-3 + 4\eta\|\Sigma_{\widetilde{\pmb{K}}}\|\,\|R_t+B^{\top}P_{t +1}B\|\right]\Tr\Big( (G_t(\pmb{K}))^{\top}G_t(\pmb{K})\Big),
\end{eqnarray*}
where the third inequality holds by \eqref{eqn:diff_Sig_bd} and the fourth inequality holds by Cauchy-Schwarz inequality. By \eqref{eq:projection_condition2} we have $\sqrt{d+k}\sum_{t=0}^{T-1}\|\nabla_tC(\pmb{K})\|\geq \sum_{t=0}^{T-1}\|G_t(\pmb{K})\|$ and thus $\eqref{eqn:diff_Sig_bd}\leq\frac{\sx}{2}$ and
\[
 \|\Sigma_{\widetilde{\pmb{K}}}\| \leq \left\|\Sigma_{\widetilde{\pmb{K}}}-\Sigma_{\pmb{K}}+\Sigma_{\pmb{K}}\right\|\leq \frac{\sx}{2}+\|\Sigma_{\pmb{K}}\|\leq \frac{\|\Sigma_{\widetilde{\pmb{K}}}\|}{2}+\frac{C(\pmb{K})}{\sq}.   
\]
Thus $\|\Sigma_{\widetilde{\pmb{K}}}\|\leq\frac{2C(\pmb{K})}{\sq}$.  Therefore when $\eta \leq C_2$, we have
\[
\eqref{Lemma 24 projection proof eqn 1} \leq  -2\eta\sum_{t=0}^T \Tr ((G_t(\pmb{K}))^{\top}(G_t(\pmb{K}))).
\]
\end{proof}

\begin{proof}[{\bf Proof of Theorem \ref{thm:projected_GD}}]
The key step in this proof is Lemma \ref{lemma 24 projection} and it suffices to show that the projected policy gradient method enjoys sublinear convergence rate in the setting of known paramaters. This is because moving from the analysis for the case of known parameters to that for the case of unknown parameters follows the same procedure of policy gradient descent (without projection). In particular, the zeroth order estimation of the gradient term $\eta\nabla_{t} C(\pmb{K})$ in \eqref{eq:gd projection} is the same for policy gradient method and projected policy gradient method. 

We now show that the projected policy gradient method with known parameters enjoys a sublinear convergence rate. Since the step size conditions \eqref{stepsize condition projection}-\eqref{eq:C2} are independent of the term $G_t(\pmb{K})$, the existence of $\eta$ follows the analysis in Theorem \ref{thm:convergence_egd}. Hence when $\eta\in\HH(\frac{1}{C(\pmb{K}^0)+1})$ is an appropriate polynomial in $\frac{1}{C(\pmb{K}^0)+1}$ and model parameters, by Lemma \ref{lemma 24 projection}, we have for any $N \in \mathbb{N}^+$,
\begin{eqnarray}
\sum_{n=1}^N \left(\sum_{t=0}^{T-1} \Tr ((G_t^{\pmb{K}^n})^{\top}(G_t^{\pmb{K}^n}))\right) \leq \frac{\sum_{n=1}^N C(\pmb{K}^{n-1})-C(\pmb{K}^{n})}{2\eta} \leq \frac{C(\pmb{K}^0)-C(\pmb{K}^{*})}{2\eta},
\end{eqnarray}

Therefore $\left\{\frac{1}{N} \sum_{n=0}^{N-1}\left(\sum_{t=0}^{T-1} \|G_t({\pmb{K}^n})\|^2_F\right)\right\}_{N \ge 1}$ converges at rate $\mathcal{O}\left(\frac{1}{N}\right)$, which thus completes the proof for the case of known parameters. 
\end{proof}
}

\begin{proof}[{\bf Proof of Lemma \ref{Model-free C_K perturbation}}]
\noindent 
{
Under Assumption \ref{ass:State-2}, we have
$\mathbb{E}\left[x_0x_0^\top\right] = \widetilde{W}_0\mathbb{E}\left[z_0z_0^\top\right]\widetilde{W}_0^\top,
$ and $ \left\|\mathbb{E}\left[x_0x_0^\top\right]\right\| \leq \sigma_0^2\|\widetilde{W}_0\|^2.$ With the sub-Gaussian distributed noise, 
$W=\mathbb{E}\left[w_tw_t^\top\right] = \widetilde{W}\mathbb{E}\left[v_tv_t^\top\right]\widetilde{W}^\top,$
then we have $\left\|W\right\| \leq\sigma_w^2\left\|\widetilde{W}^2\right\|$.}

\noindent Denote $S_{t} = Q_{t}+K_t^{T}R_{t}K_t$,  $\forall\,t=1,\cdots,T-1$. Thus, for $t = 0, 1,\cdots,T-2$,
\begin{equation*}
    \begin{split}
        \mathbb{E}[x_{t+1}^\top Q_{t+1} x_{t+1} + u_{t+1}^\top R_{t+1}u_{t+1}] &= \mathbb{E}[x_{t+1}^\top S_{t+1}x_{t+1}]= \Tr(\mathbb{E}[x_{t+1}^\top S_{t+1}x_{t+1}]) = \Tr(\mathbb{E}[x_{t+1}x_{t+1}^{\top}] S_{t+1})\\
        & = \Tr\left(\mathcal{G}_{t}(\Sigma_0)S_{t+1} + \sum_{s=1}^{t} D_{t,s} W D_{t,s}^{\top}S_{t+1} +WS_{t+1}\right).
    \end{split}
\end{equation*}
The last equality holds by \eqref{eqn:intermedidate1}.
\noindent Therefore,
\begin{equation*}
    \begin{split}
        C(\pmb{K}^\prime)-C(\pmb{K})
        & = \underbrace{\mathbb{E}[x_0^\top(K_0^\prime)^\top R_0K_0^\prime x_0 - x_0^\top K_0^\top R_0K_0x_0] }_{(\rom{1})}+ 
        \underbrace{\sum_{t=0}^{T-2}\Tr\Big(\mathcal{G}_{t}^\prime(\Sigma_0)S_{t+1}^\prime-\mathcal{G}_{t}(\Sigma_0)S_{t+1}\Big)}_{(\rom{2})}\\
        & \quad + \underbrace{\sum_{t=0}^{T-2}\Tr\Big(\sum_{s=1}^{t}\left( D_{t,s}^\prime W (D_{t,s}^\prime)^{\top}S_{t+1}^\prime-D_{t,s} W D_{t,s}^{\top}S_{t+1}\right)
        + W(S_{t+1}^\prime-S_{t+1})\Big)}_{(\rom{3})}\\
        & \quad + \underbrace{\Tr\left(\mathcal{G}_{T-1}(\Sigma_0)Q_T - \mathcal{G}_{T-1}^\prime(\Sigma_0)Q_T + \sum_{s=1}^{T-1}\left( D_{T-1,s}^\prime W (D_{T-1,s}^\prime)^{\top}Q_T - D_{T-1,s} W D_{T-1,s}^{\top}Q_T\right)\right)}_{(\rom{4})}.
    \end{split}
\end{equation*}
For the first term, $    (\rom{1}) \leq \Tr(\mathbb{E}[x_0x_0^\top])\|(K_0^\prime)^\top R_0 K_0^\prime - K_0^\top R_0K_0\|.$
 For the second term $(\rom{2})$, since
 \begin{eqnarray*}
 \sum_{t=0}^{T-2}\left(\Tr\left(\mathcal{G}_{t}(\Sigma_0)S_{t+1}\right)\right)  & = \mathbb{E}\left[ \sum_{t=0}^{T-2} \left(\Tr\left(\Pi_{i=0}^t (A-BK_i)x_0x_0^\top \Pi_{i=0}^t (A-BK_{t-i})^{\top} S_{t+1}\right)\right)\right] \leq \Tr\left(\mathbb{E}\left[x_0x_0^\top\right]\right) \left\|\sum_{t=0}^{T-2}\mathcal{G}_{t}(S_{t+1})\right\|,
 \end{eqnarray*}

we have,
$(\rom{2}) \leq  \Tr\left(\mathbb{E}\left[x_0x_0^\top\right]\right) \left\|\sum_{t=0}^{T-2}\left(\mathcal{G}_{t}^\prime\left(S_{t+1}^\prime\right)-\mathcal{G}_{t}\left(S_{t+1}\right)\right)\right\|.$
    
We denote $\mathcal{G}_{d}:=\sum_{t=0}^{T-2}\left(\mathcal{G}_t^{\prime}\left(S^{\prime}_{t+1}\right)-\mathcal{G}_{t}\left(S_{t+1}\right)\right)$, then
\begin{equation}\label{eq:G_d_bound}
    \begin{split}
        \|\mathcal{G}_d\| 
        &\leq \sum_{t=0}^{T-2} \Big\| \mathcal{G}_t^{\prime}\left(Q_{t+1}+(K_{t+1}^{\prime})^{\top}R_{t+1}K_{t+1}^{\prime}\right)-\mathcal{G}_t\left(Q_{t+1}+(K_{t+1}^{\prime})^{\top}R_{t+1}K_{t+1}^{\prime}\right)-\\
        & \quad \mathcal{G}_t\circ\left(K_{t+1}^{\top}R_{t+1}K_{t+1}-(K_{t+1}^{\prime})^{\top}R_{t+1}K_{t+1}^{\prime}\right)\Big\| \\
        & \leq \frac{\rho^{2T}-1}{\rho^2-1} \left((2\rho+1)\|B\|\sum_{t=0}^{T-2}\|K_t-K_t^{\prime}\|\right)\left(\sum_{t=1}^{T-1} \|Q_t+(K_t^{\prime})^{\top}R_tK_t^{\prime}\|\right)\\
        & \quad +\sum_{t=0}^{T-2}\left\|\mathcal{G}_t\right\|\left\|(K_{t+1}^{\prime})^{\top}R_{t+1}K_{t+1}^{\prime}-K_{t+1}^{\top}R_{t+1}K_{t+1}\right\| \\
        & \leq \frac{\rho^{2T}-1}{\rho^2-1} \left((2\rho+1)\|B\|\sum_{t=0}^{T-2}\|K_t-K_t^{\prime}\|\right)\left(\sum_{t=1}^{T-1} \|Q_t+(K_t^{\prime})^{\top}R_tK_t^{\prime}-K_t^\top R_tK_t+K_t^\top R_tK_t\|\right)\\
        & \quad +\frac{\rho^2(\rho^{2(T-1)}-1)}{\rho^2-1}\sum_{t=1}^{T-1}\left\|(K_t^{\prime})^{\top}R_tK_t^{\prime}-K_t^{\top}R_tK_t\right\| \\
        & \leq \frac{\rho^{2T}-1}{\rho^2-1} (2\rho+1)\|B\|\,\vertiii{\pmb{K}^\prime-\pmb{K}}\left(\vertiii{\pmb{Q}}+\vertiii{\pmb{K}}^2\,\vertiii{\pmb{R}}\right)\\
        & \quad +\left(\frac{\rho^{2T}-1}{\rho^2-1} (2\rho+1)\|B\|\,\vertiii{\pmb{K}^\prime-\pmb{K}}+\frac{\rho^2(\rho^{2(T-1)}-1)}{\rho^2-1}\right)\sum_{t=1}^{T-1}\left\|(K_t^{\prime})^{\top}R_tK_t^{\prime}-K_t^{\top}R_tK_t\right\|. \\
    \end{split}
\end{equation}
where the second inequality holds by Lemma \ref{Lemma 20} and \eqref{F_Kt bound}, and the third inequality holds by \eqref{eqn:G_t_bound}.
For the first term in $(\rom{3})$, we have
\begin{eqnarray*}
&& \sum_{t=0}^{T-2}\Tr\left(\sum_{s=1}^{t}D_{t,s}^\prime W (D_{t,s}^\prime)^{\top}S_{t+1}^\prime- D_{t,s} W D_{t,s}^{\top}S_{t+1}\right) \\
& = & \sum_{t=0}^{T-2}\Tr\left(\sum_{s=1}^{t}D_{t,s}^\prime W (D_{t,s}^\prime)^{\top}(S_{t+1}^\prime-S_{t+1})+(D_{t,s}^\prime W (D_{t,s}^\prime)^{\top} - D_{t,s} W D_{t,s}^{\top})S_{t+1}\right)\\
& \leq & \Big(\sum_{t=0}^{T-2}\sum_{s=1}^{t}\Tr(W)\|D_{t,s}^\prime\|^2\Big)\left\|\sum_{t=1}^{T-1}(K_{t}^\prime)^\top R_{t}K_{t}^\prime-K_{t}^\top R_{t}K_{t}\right\|\\
& & +\sum_{t=0}^{T-2}\left\|\sum_{s=1}^{t}D_{t,s}^\prime W (D_{t,s}^\prime)^{\top} - D_{t,s} W D_{t,s}^{\top}\right\|\Big(\sum_{t=1}^{T-1}\Tr(Q_t)+\|K_t\|^2\Tr(R_t)\Big)\\
& \leq & {\Tr(W)\frac{(T-1)(\rho^{2(T-1)}-1)}{\rho^2-1}}\left\|\sum_{t=1}^{T-1}(K_{t}^\prime)^\top R_{t}K_{t}^\prime-K_{t}^\top R_{t}K_{t}\right\|\\
& & + T \frac{ (\rho^{2T}-1)}{\rho^2-1} (2\rho+1)\|B\|\,\|W\|\,\vertiii{\pmb{K}^\prime-\pmb{K}}\left(\Tr\left(\sum_{t=1}^{T-1}Q_t\right)+\vertiii{\pmb{K}}^2\Tr\left(\sum_{t=1}^{T-1}R_t\right)\right),\\
\end{eqnarray*}
where the last step holds by \eqref{eq:sum_D_bound}.
The second term in $(\rom{3})$ is bounded by
\[
\sum_{t=0}^{T-2}\Tr\Big( W(S_{t+1}^\prime-S_{t+1})\Big) \leq \Tr(W)\sum_{t=1}^{T-1}\left\|(K_t^{\prime})^{\top}R_tK_t^{\prime}-K_t^{\top}R_tK_t\right\|.
\]
{Similarly, by \eqref{eq:sum_D_bound} and \eqref{F_Kt bound}, $(\rom{4})$ is bounded by
\begin{eqnarray*}
(\rom{4}) &\leq&
\Tr(\mathbb{E}[x_0x_0^\top])\sum_{t=0}^{T-1}\Big\|(\mathcal{G}_{t}^{\prime}-\mathcal{G}_{t})(Q_T)\Big\|+\Tr\left(\sum_{s=1}^{T-1} D_{T-1,s}^\prime W (D_{T-1,s}^\prime)^{\top}Q_T - D_{T-1,s} W D_{T-1,s}^{\top}Q_T\right) \\
& \leq& \Tr(\mathbb{E}[x_0x_0^\top])\frac{\rho^{2T}-1}{\rho^2-1}(2\rho+1)\|B\|\|Q_T\|\,\vertiii{\pmb{K}^\prime-\pmb{K}} + \Tr(Q_T)\frac{\rho^{2T}-1}{\rho^2-1} (2\rho+1)\|B\|\,\|W\|\,\vertiii{\pmb{K}^\prime-\pmb{K}}.
\end{eqnarray*}}
Now we bound the term $ \sum_{t=1}^{T-1}\left\|(K_t^{\prime})^{\top}R_tK_t^{\prime}-K_t^{\top}R_tK_t\right\|$, which appears several times in previous inequalities:
\begin{equation*}
\begin{split}
    \sum_{t=1}^{T-1}\left\|(K_t^{\prime})^{\top}R_tK_t^{\prime}-K_t^{\top}R_tK_t\right\| &= \sum_{t=1}^{T-1}\left\|(K_t^{\prime}-K_t+K_t)^{\top}R_t(K_t^{\prime}-K_t+K_t)-K_t^{\top}R_tK_t\right\|\\
    & \leq  \sum_{t=1}^{T-1}\|K_t^{\prime}-K_t\|^2\|R_t\|+2\|K_t\|\|R_t\|\|K_t^{\prime}-K_t\| \leq 3\vertiii{\pmb{K}}\,\vertiii{\pmb{R}}\,\vertiii{\pmb{K}^\prime-\pmb{K}}.\\
\end{split}
\end{equation*}
The last step holds since $\|K_t^{\prime}-K_t\|\leq \|K_t\|$ by assumption.

Therefore,
\begin{equation*}
    \begin{split}
        |C(\pmb{K}^\prime)-C(\pmb{K})|
        & \leq \Tr(\mathbb{E}[x_0x_0^\top])\Big\{3\vertiii{\pmb{K}}\|R_0\|\vertiii{\pmb{K}^\prime-\pmb{K}} + \frac{\rho^{2T}-1}{\rho^2-1}(2\rho+1)\|B\|\|Q_T\|\,\vertiii{\pmb{K}^\prime-\pmb{K}} \\
        & \quad + \frac{\rho^{2T}-1}{\rho^2-1} (2\rho+1)\|B\|\,\vertiii{\pmb{K}^\prime-\pmb{K}}\left(\vertiii{\pmb{Q}}+\vertiii{\pmb{K}}^2\,\vertiii{\pmb{R}}\right)\\
        & \quad +\left(\frac{\rho^{2T}-1}{\rho^2-1} (2\rho+1)\|B\|\,\vertiii{\pmb{K}^\prime-\pmb{K}}+\frac{\rho^2(1-\rho^{2(T-1)})}{\rho^2-1}\right)3\vertiii{\pmb{K}}\,\vertiii{\pmb{R}}\,\vertiii{\pmb{K}^\prime-\pmb{K}}\Big\} \\
        & \quad + 3{\Tr(W)\Big(\frac{(T-1)(\rho^{2(T-1)}-1)}{\rho^2-1}}+1\Big)\vertiii{\pmb{K}}\,\vertiii{\pmb{R}}\,\vertiii{\pmb{K}^\prime-\pmb{K}}\\
        & \quad + \left(T \frac{ (\rho^{2T}-1)}{\rho^2-1} (2\rho+1)\|B\|\,\|W\|\,\vertiii{\pmb{K}^\prime-\pmb{K}}\right)\left(\Tr\left(\sum_{t=1}^{T-1}Q_t\right)+\vertiii{\pmb{K}}^2\Tr\left(\sum_{t=1}^{T-1}R_t\right)\right)\\
        & \quad  +  \Tr(Q_T)\frac{\rho^{2T}-1}{\rho^2-1} (2\rho+1)\|B\|\,\|W\|\,\vertiii{\pmb{K}^\prime-\pmb{K}}.
    \end{split}
\end{equation*}
{By \eqref{rho eqn}, Lemma \ref{lemma 13}, and Lemma \ref{lemma 25}, $\rho$ is bounded above by polynomials in $\|A\|$, $\|B\|$, $\vertiii{\pmb{R}}$, $\frac{1}{\sx}$, $\frac{1}{\sr}$ and $C(\pmb{K})$, or a constant $1+\xi$.} Therefore, we rewrite the above inequality by
\begin{equation}\label{lemma 12 inte}
    |C(\pmb{K}^\prime)-C(\pmb{K})| \leq h_{CK}\vertiii{\pmb{K}^\prime-\pmb{K}} + h_{CK}^\prime \vertiii{\pmb{K}^\prime-\pmb{K}}^2,
\end{equation}
where $h_{CK}\in\HCK$ and $h_{CK}^\prime\in\HCK$ are polynomials in $C(\pmb{K})$ and model parameters. 
Given assumption  \eqref{Lemma 27 assumption}, we have $\vertiii{\pmb{K}^\prime-\pmb{K}} \le 1$ and hence
\[
\vertiii{\pmb{K}^\prime-\pmb{K}} \ge  \vertiii{\pmb{K}^\prime-\pmb{K}}^2.
\]
Define $h_{cost} = h_{CK}+h_{CK}^\prime$, then \eqref{lemma 12 inte} gives
\begin{equation*}
    |C(\pmb{K}^\prime)-C(\pmb{K})| \leq h_{cost}\vertiii{\pmb{K}^\prime-\pmb{K}},
\end{equation*}
with $h_{cost}\in\HCK$.
\end{proof}

\begin{proof}[{\bf Proof of Lemma \ref{lemma 28}}]
Recall $\nabla_t C(\pmb{K})=2E_t\Sigma_t$ {and $W=\mathbb{E}\left[w_tw_t^\top\right]=\widetilde{W}\mathbb{E}\left[v_tv_t^\top\right]\widetilde{W}^\top$}. We have,
\begin{equation}\label{eq:gra_CK_proof_eq1}
    \|\nabla_t C(\pmb{K}^{\prime}) - \nabla_t C(\pmb{K})\| =  \|2E_t^{\prime}\Sigma_t^{\prime}-2E_t\Sigma_t\|\leq  2\|E_t^{\prime}-E_t\|\|\Sigma_t^{\prime}\|+2\|E_t\|\|\Sigma_t^{\prime}-\Sigma_t\|,
\end{equation}
For the second term, by Lemma \ref{lemma 11} and Cauchy-Schwarz inequality,
\begin{equation}\label{bound E_t}
\begin{split}
     \left\|E_t\right\|\leq \sum_{t=0}^{T-1}\left\|E_t\right\| &\leq \sum_{t=0}^{T-1} \sqrt{\Tr(E_t^{\top}E_t)}
     \leq \sqrt{T\cdot\frac{\max_t \|R_t+B^{\top}P_{t+1} B\|}{\sx} \left(C(\pmb{K})-C(\pmb{K}^*)\right)}.
\end{split}
\end{equation}
By \eqref{tn1} and direct calculation, we have
\begin{eqnarray*}
\|(\mathcal{G}_{t+1}^{\prime}-\mathcal{G}_{t+1})(\Sigma_0)\|
\leq \rho^{2(t+1)}\left(\sum_{s=0}^{t+1}\|\mathcal{F}_{K_s^{\prime}}-\mathcal{F}_{K_s}\|\,\|\Sigma_0\|\right).
\end{eqnarray*}

\noindent By \eqref{F_Kt bound} and \eqref{eq:sum_D_bound}, for $t=1,2,\cdots,T-1$,
\begin{equation}\label{bound Sigma_t}
\begin{split}
       \|\Sigma_t^{\prime}-\Sigma_t\| &\leq \|(\mathcal{G}_t^{\prime}-\mathcal{G}_t)(\Sigma_0)\|+\left\|\sum_{s=0}^{t-1}D_{t-1,s}{W}D_{t-1,s}^{\top}-D^{\prime}_{t-1,s}{W}(D^{\prime}_{t-1,s})^{\top}\right\|\\
       & \leq \rho^{2t}(2\rho+1)\|B\|\,\|\Sigma_0\|\,\vertiii{\pmb{K}^\prime-\pmb{K}}\ + \frac{(\rho^{2T}-1)}{\rho^2-1} (2\rho+1)\|B\|\,\|{W}\|\,\vertiii{\pmb{K}^\prime-\pmb{K}}.
\end{split}
\end{equation}
Therefore the second term in \eqref{eq:gra_CK_proof_eq1} is bounded by the product of \eqref{bound E_t} and \eqref{bound Sigma_t}.\\

Next we bound the first term in \eqref{eq:gra_CK_proof_eq1}. {Similar to \eqref{eq:SigmaKp_bound},  $\|\Sigma_t^{\prime}\|\leq$ $\|\sum_{t=0}^T \Sigma_t^{\prime}\|=\|\Sigma_{\pmb{K}^{\prime}}\| \leq \| \Sigma_{\pmb{K}}^\prime-\Sigma_{\pmb{K}}\| + \|\Sigma_{\pmb{K}}\|\leq \frac{C(\pmb{K})}{\sq} + \|\Sigma_{\pmb{K}}\|$.} For $\|E_t^{\prime}-E_t\|$, we first need a bound on $\left\|P_t^{\prime}-P_t\right\|$. Since
$P_0 = S_0+\sum_{t=0}^{T-2}\mathcal{G}_t(S_{t+1})+\mathcal{G}_{T-1}(Q_T)$, by \eqref{eq:G_d_bound}, we have
\begin{equation}\label{bound for diff P_0}
    \begin{split}
        \|P_t^\prime-P_t\|& \leq\|P_0^\prime-P_0\|\leq 3\|K_0\|\|R_0\|\|K_0^\prime-K_0\| + \|\mathcal{G}_d\| + \frac{\rho^{2T}-1}{\rho^2-1}(2\rho+1)\|B\|\|Q_T\|\left(\sum_{t=0}^{T-1}\|K_t-K_t^{\prime}\|\right) \\
        & \leq \frac{\rho^{2T}-1}{\rho^2-1} (2\rho+1)\|B\|\,\vertiii{\pmb{K}^\prime-\pmb{K}}\left(\vertiii{\pmb{Q}}+\vertiii{\pmb{K}}^2\vertiii{\pmb{R}}\right) \\
        & \quad +3\left(1+\frac{\rho^{2T}-1}{\rho^2-1} (2\rho+1)\|B\|\,\vertiii{\pmb{K}^\prime-\pmb{K}}+\frac{\rho^2(1-\rho^{2(T-1)})}{\rho^2-1}\right)\cdot\vertiii{\pmb{K}}\,\vertiii{\pmb{R}}\,\vertiii{\pmb{K}^\prime-\pmb{K}} \\
        & \quad +\frac{\rho^{2T}-1}{\rho^2-1}(2\rho+1)\|B\|\|Q_T\|\vertiii{\pmb{K}^\prime-\pmb{K}}.
    \end{split} 
\end{equation}
Thus,
\begin{equation*}
\begin{split}
     \left\|E_t^{\prime}-E_t\right\| &= \left\| R_t(K_t^{\prime}-K_t)-B^{\top}(P_{t+1}^{\prime}-P_{t+1})A+B^{\top}(P_{t+1}^{\prime}-P_{t+1})BK_t^{\prime}+B^{\top}P_{t+1}B(K_t^{\prime}-K_t)\right\|\\
     & \leq \left(\|R_t\|+\|B\|^2\|P_{0}\| \right)\vertiii{\pmb{K}^\prime-\pmb{K}}+\|B\|\,\|P_{0}^{\prime}-P_{0}\|\,\|A\|+2\|B\|^2\|P_{0}^{\prime}-P_{0}\|\vertiii{\pmb{K}}.
\end{split}
\end{equation*}
Given the bound on $\vertiii{\pmb{K}}=\sum_{t=0}^{T-1}\|K_t\|$ in Lemma \ref{lemma 25} and the bound on $\|P_t\|$ in Lemma \ref{lemma 13}, all the terms in \eqref{eq:gra_CK_proof_eq1} can be bounded by polynomials of related parameters multiplied by $\vertiii{\pmb{K}^\prime-\pmb{K}}$ and $\vertiii{\pmb{K}^\prime-\pmb{K}}^2$. Similarly to the proof of Lemma \ref{Model-free C_K perturbation}, we have $\vertiii{\pmb{K}^\prime-\pmb{K}}\leq 1$ and 
\[
 \|\nabla_t C(\pmb{K}^{\prime}) - \nabla_t C(\pmb{K})\| \leq h_{grad}\vertiii{\pmb{K}^\prime-\pmb{K}},
\]
{for some polynomial $h_{grad}\in\HCK$}.

\end{proof}

\subsection{Proofs in Section \ref{sc:experiment}}
\label{appendix:missing_proofs_sec5}
\begin{proof}[Proof of Proposition \ref{prop:liquidation}]
Denote $H_t:=
\begin{pmatrix}
 1+\gamma k_t^1 &\gamma k_t^2\\ 
 k_t^1& 1+ k_t^2 
\end{pmatrix}
$. Since  $H_t$ has two eigenvalues $1$ and $\gamma k_t^1 +k_t^2+1$,  $H_t$ is positive definite when $\gamma k_t^1 +k_t^2>-1$ ($ 0\leq t \leq T-1$).

Then let us show the first claim by induction. Assume $\mathbb{E}[x_s x_s^{\top}]$ is positive definite for all $s\leq t$, then
\begin{eqnarray*}
\mathbb{E}[x_{t+1} x_{t+1}^{\top}] &=& \mathbb{E}[\left((A -BK_t)x_t +w_t\right)\left((A -BK_t)x_t +w_t\right)^{\top}]= \mathbb{E}[\left(H_tx_t +w_t\right)\left(H_tx_t +w_t\right)^{\top}]\\
& = & \mathbb{E}[H_t x_t x_t^{\top} H_t^{\top} +w_t w_t^{\top} +w_tw_t^{\top} + 2 H_t x_t w^{\top}_t]= H_t \mathbb{E}[x_t x_t^{\top}] H_t^\top +  \begin{pmatrix}
 \sigma & 0\\ 
 0&0
\end{pmatrix}.
\end{eqnarray*}
Hence $\mathbb{E}[x_{t+1} x_{t+1}^{\top}]$ is positive definite since  $\mathbb{E}[x_t x_t^{\top}]$  is positive definite and $H_t$ is positive definite. Therefore $\sx>0$.

The second claim can be proved by backward induction. For $t=T$,  $ P_{T}^{\pmb{K}} = Q_T$ is positive definite since $Q_T$ is positive definite. Assume $P_{t+1}^{\pmb{K}}$ is positive definite for some $t+1$, then take any $z \in \mathbb{R}^d$ such that $z \neq 0$,
\[
z^{\top} P_t^K z=z^{\top} Q_t\,z+z^{\top}K_t^{\top}R_tK_tz+z^{\top}H_t^{\top}P_{t+1}^{\pmb{K}}H_tz > 0.
\]
Note that $H_t$ is positive definite when  $\gamma k_t^1 +k_t^2>-1$ and $1+\gamma k_t^1>0$. The last inequality holds since $Q_t$ and   $H_t^{\top}P_{t+1}^{\pmb{K}}H_t$ are  positive definite, and  $K_t^{\top}R_tK_t$ is positive semi-definite.  Hence we have $P_t^{\pmb{K}}$ positive definite for all $t=0,1,2,\cdots,T$.
\end{proof}

\end{document}